\documentclass[journal]{IEEEtran}

\makeatletter
\let\NAT@parse\undefined
\makeatother

\usepackage[T1]{fontenc}
\usepackage[utf8]{inputenc}
\usepackage{tabularx,ragged2e,booktabs,caption}
\newcolumntype{C}[1]{>{\Centering}m{#1}}

\usepackage{graphicx} 
\graphicspath{{figures/}}
\usepackage{epsfig} 
\usepackage{ragged2e}
\usepackage{amsmath}
\usepackage{amssymb}
\usepackage{epstopdf}
\usepackage{algpseudocode}
\usepackage[noend,ruled,linesnumbered]{algorithm2e}
\SetKwComment{Comment}{$\triangleright$\ } {}
\usepackage{multirow}
\usepackage{rotating}
\usepackage{subfigure}
\usepackage{color}
\usepackage{mysymbol}
\usepackage{url}
\usepackage{ulem}
\usepackage{enumerate}
\usepackage{stmaryrd}
\usepackage{array}
\usepackage{pifont}
\usepackage{multicol}
\usepackage[bookmarks=true]{hyperref}
\usepackage{colortbl}

\makeatletter
\renewcommand*{\@opargbegintheorem}[3]{\trivlist
  \item[\hskip \labelsep{\it\quad  #1\ #2:}] {\it(#3)}\ }
\makeatother
\newtheorem{theorem}{Theorem}[section]
\newtheorem{cor}[theorem]{Corollary}

\newtheorem{prop}[theorem]{Proposition}
\newtheorem{problem}{Problem}
\newtheorem{asmp}[theorem]{Assumption}
\newtheorem{exmp}{Example}

\newtheorem{defn}[theorem]{Definition}
\newtheorem{rem}[theorem]{Remark}

\newcommand{\level}[2]{\phi_{#2}^{#1}}

  \newenvironment{cexmp}[2]
  {\renewcommand{\theexmp}{\ref{#1}}\begin{exmp}{\textit{continued} (#2)}}
  {\end{exmp} \addtocounter{exmp}{-1}}

\newcommand{\state}[3]{#1_{#2}^{#3}}
\newcommand{\switch}[2]{\zeta_{\text{#1}}^{#2}}
\newcommand{\symtext}[2]{#1_{\text{#2}}}
\newcommand{\modifycolor}{black}  
 
\newcommand{\hltl}{{\color{\modifycolor}H-LTL$_f$}}  
\newcommand{\ltl}{{\color{\modifycolor}LTL$_f$}}  

\usepackage{tikz}

\newcommand{\bhline}{\noalign{\hrule height 1.2pt}}

\newenvironment{sizeddisplay}[1]
 {\par\nopagebreak#1\noindent\ignorespaces}
 {\nopagebreak\ignorespacesafterend}

 \usepackage{fancyhdr}  

\usepackage{relsize} 

\graphicspath{ {figs/} }
\usepackage[numbers,compress]{natbib}

\definecolor{somegray}{rgb}{0.5, 0.5, 0.5}
\newcommand{\darkgrayed}[1]{\textcolor{somegray}{#1}}
\makeatletter
\newcommand*\titleheader[1]{\gdef\@titleheader{#1}}
\AtBeginDocument{%
  \let\st@red@title\@title
  \def\@title{%
    \vskip-2em
    \bgroup\normalfont\large\centering\@titleheader\par\egroup
    \vskip0.5em\st@red@title}
}
\makeatother

\titleheader{\darkgrayed{This paper has been accepted for publication in the
 IEEE Transactions on Robotics (T-RO),\,2025. \textcopyright IEEE}}

\title{\LARGE \bf Simultaneous Task Allocation and Planning for Multi-Robots \\ under Hierarchical Temporal Logic Specifications
}%
\author{Xusheng Luo$^{1}$ and Changliu Liu$^{1}$
\thanks{This work was supported in part by Siemens, in part by Manufacturing Futures Institute, Carnegie Mellon University, through a grant from the Richard King Mellon Foundation, and in part by the National Science Foundation under Grant No. 2144489.}
\thanks{$^{1}$Xusheng Luo and Changliu Liu are with Robotics Institute, Carnegie Mellon University, Pittsburgh, PA 15213, USA (e-mail: {\tt\small\{xushengl, cliu6\}@andrew.cmu.edu})}%
}

\begin{document}


\maketitle

\begin{abstract}
Research in robotic planning with temporal logic specifications, such as Linear Temporal Logic (LTL), has relied on single formulas. However, as task complexity increases, LTL formulas become lengthy, making them difficult to interpret and generate, and straining the computational capacities of planners. {\color{\modifycolor}To address this, we introduce a hierarchical structure for a widely used specification type—LTL on finite traces (\ltl). The resulting language, termed H-LTL$_f$, is defined with both its syntax and semantics. We further prove that H-LTL$_f$ is more expressive than its standard “flat” counterparts.} {\color{\modifycolor}Moreover}, we conducted a user study that compared the {\color{\modifycolor}standard} \ltl\ with our hierarchical version and found that users could more easily comprehend complex tasks using the hierarchical structure. We develop a search-based approach to synthesize plans for multi-robot systems, achieving simultaneous task allocation and planning. This method approximates the search space by loosely interconnected sub-spaces, each corresponding to an \ltl\ specification. The search primarily focuses on a single sub-space, transitioning to another under conditions determined by the decomposition of {\color{\modifycolor}automata}. We develop multiple heuristics to significantly expedite the search. Our theoretical analysis, conducted under mild assumptions, addresses completeness and optimality. Compared to existing methods used in various simulators for service tasks, our approach improves planning times while maintaining comparable solution quality.
\end{abstract}
\IEEEpeerreviewmaketitle

\section{Introduction}
In the field of multi-robot systems, two challenges are consistently at the forefront of research interest: (1) task allocation~\cite{korsah2013comprehensive}, determining which robot should perform which task, and (2) planning~\cite{lavalle2006planning}, the strategy for executing these tasks. Traditionally, these problems have been tackled separately by researchers from various fields. To simplify these issues, certain assumptions are often employed, such as the existence of  low-level controllers for task allocation or the use of pre-defined tasks in motion planning. These approaches, however, address only portions of the whole problem. Over the past two decades, methods have emerged that address these issues by combining task allocation and planning. Our work aligns with this integrative approach, but distinguishes itself by focusing on tasks defined by temporal logic specifications.

Formal methods, characterized by their mathematical rigor, are essential for specifying, developing, analyzing, and verifying software and hardware systems~\cite{woodcock2009formal}. The recent trend of utilizing formal specifications, such as Linear Temporal Logic (LTL)~\cite{pnueli1977temporal}, as high-level task specifications for multi-robots has garnered significant attention. This is because temporal logic specifications can encapsulate not just conventional point-to-point Boolean goals but also complex temporal requirements. 
While offering an expressive framework for task descriptions, temporal logic specifications also introduce computational challenges for solvers due to their complexity. Take LTL for an example. A typical method involves translating the formula into an automaton, a graphic representation. As shown in~\cite{kurtz2023temporal}, which involves a task of collecting 5 keys and then opening 5 doors, it took approximately half an hour on a laptop with an Intel i7 CPU and 32 GB of RAM to convert the corresponding LTL into an automaton consisting of 65 nodes and 792 edges. Similarly, in our simulations, an automaton could not be generated for a complex task within an hour on a laptop with Apple M2 Pro and 16G RAM. This issue stems from the usage of the {\color{\modifycolor}standard ``flat'' form of LTL}, that is, putting all requirements on robot behaviors into a single LTL formula, which tends to become difficult to solve for complex tasks. However, one key observation is that robot tasks often have a loose connection and can be broken down into smaller components. Studies suggest that humans prefer hierarchical task specification, which improves interpretability of planning and execution, making it easier to identify ongoing work and conveniently adjust unfeasible parts without affecting other components~\cite{tenenbaum2011grow, kemp2007learning}. {\color{\modifycolor}Inspired by this, we propose a hierarchical form of a widely used specification type—LTL on finite traces (LTL$_f$)~\cite{de2013linear}, which can be satisfied by finite sequence of system states. This characteristic makes \ltl\ apt for modeling and reasoning about systems with finite durations, such as those found in the robotics field.}

In this work, we first establish syntax and semantics for \hltl. To do this, considering the context of multiple specifications, we introduce the concept of {\it state-specification sequence}, which pairs the state of a system with a  specification it seeks to fulfill at a given time. Theoretically, we prove that the hierarchical form is more expressive compared to the standard form. Practically, we conducted a user study to show that using a hierarchical structure makes it easier to express complex tasks with temporal constraints.  Next, we propose a bottom-up planning algorithm that involves searching within a product graph that combines the environment and task space. This method's advantage is it can simultaneously achieve  task allocation and planning (STAP).
Particularly, inspired by~\cite{schillinger2018simultaneous}, we approximate the entire search space as a set of loosely interconnected sub-spaces, each corresponding to an individual \ltl\ specification in the hierarchical structure. The search process primarily occurs within a single sub-space, with transitions to adjacent sub-spaces happening under conditions determined by the decomposition of the {\color{\modifycolor}automata}. Under certain assumptions regarding the state-specification sequence, we demonstrate that our planning algorithm is sound, complete, and optimal. Additionally, we introduce several heuristics to expedite the search process, including minimizing the loose connections between sub-spaces and leveraging task-level progress to guide the search. We evaluate on extensive simulations on robot service tasks. Comparisons between our approaches with and without heuristics highlight a trade-off between computational time and solution quality. Our method shows significant improvements in computational time than work~\cite{schillinger2018simultaneous}. Notably, for complex tasks,~\cite{schillinger2018simultaneous} failed to produce solutions within a one-hour timeout. In scalability tests, our method demonstrated the ability to generate sub-optimal solutions in approximately three minutes for complex tasks involving up to 30 robots. The standard version of these tasks could not be solved using graph-theoretic methods due to the inability to generate an automaton within a one-hour timeout. A glossary of key notations is provided in Tab.~\ref{tab:glossary}.

{\bf Contributions} The contributions are listed as follows:
\begin{enumerate}
\item We introduce a hierarchical form of \ltl\ and prove it is more expressive than the standard form;
\item We conduct a user study to show that the \hltl\  is easier to interpret compared to the standard \ltl;
\item We develop a search-based planning algorithm, achieving simultaneous task allocation and planning. This planner is, to the best of our knowledge, the first to offer both completeness and optimality for \hltl;
\item We devise multiple heuristics to expedite the search;
\item Under mild assumptions, we theoretically analyze the completeness and optimality of our approach;
\item We conduct extensive comparative simulations focusing on service tasks to showcase the efficiency and scalability of our proposed method.
\end{enumerate}
\section{Related Work}

\subsection{Multi-robots Planning under Temporal Logic Specifications}\label{sec:lit_mr}
In the existing body of research on optimal control synthesis from LTL specifications, two primary approaches are present in handling LTL tasks for multi-robot systems. One approach, as seen in works~\cite{guo2015multi,tumova2016multi,yu2021distributed}, involves assigning LTL tasks {\it locally} to individual robots within the team. The other approach assigns a {\it global} LTL specification to the entire team. In scenarios where global LTL specifications are used, these specifications can either explicitly assign tasks to individual robots~\cite{loizou2004automatic,smith2011optimal,saha2014automated,kantaros2017sampling,kantaros2018distributedOpt,kantaros2018sampling,kantaros2020stylus,kantaros2022perception,luo2019transfer,luo2021abstraction}, or the tasks may not be explicitly designated to specific robots~\cite{kloetzer2011multi,shoukry2017linear,moarref2017decentralized,lacerda2019petri},  similar to our problem in this work.

 As temporal logic formulas are used to tackle tasks involving multiple robots and complex environments, they inevitably become lengthy. Some attempts have been made to simplify this, such as merging multiple atomic propositions into one using logical operators~\cite{kantaros2020stylus,luo2021abstraction}, or combining multiple sub-formulas into one formula using logical operators~\cite{saha2014automated,smith2011optimal,sahin2019multirobot}. Despite these efforts, we still categorize these as the standard form. Notably, temporal operators, the feature distinguishing LTL from propositional logic, only appear on one side, either for the integration of merged propositions or inside the sub-formulas.  Works~\cite{shoukry2017linear,sahin2019multirobot,leahy2021scalable,luo2022temporal,liu2022time,li2023fast} have defined propositions involving more than one robot to encapsulate collaborative tasks. In our study, we introduce {\it composite propositions} that encompass more than one sub-formula and can be combined using temporal operators. Besides, various extensions of temporal logic have been introduced:~\cite{sahin2019multirobot} developed {\color{\modifycolor}Counting Linear Temporal Logic (cLTL)} to express the collective behavior of multiple robots.~\cite{djeumou2020probabilistic,yan2019swarm} utilized Graph Temporal Logic (GTL) and Swarm Signal Temporal Logic (SwarmSTL) respectively to specify swarm characteristics like centroid positioning and density distribution.~\cite{leahy2021scalable} introduced Capability Temporal Logic (CaTL) to encapsulate the diverse capabilities of robots. However, all works mentioned focus on using standard LTL specifications to describe the behaviors in multi-robot systems.

Global specifications that do not explicitly assign tasks to robots generally require decomposition to derive the task allocation, which can be accomplished in three ways: The most common method, utilized in works such as~\cite{schillinger2018simultaneous,luo2019transfer,camacho2017non,camacho2019ltl,schillinger2019hierarchical,luo2022temporal,liu2024time}, involves decomposing a global specification into multiple tasks, which leverages the transition relations within the automaton, which is the graphical representation of an LTL formula. As demonstrated in~\cite{shoukry2017linear,sahin2019multirobot}, the second approach builds on Bounded Model Checking (BMC) methods~\cite{biere2006linear} to create a Boolean Satisfaction or Integer Linear Programming (ILP) model, which simultaneously addresses task allocation and implicit task decomposition in a unified formulation. Another method, proposed by~\cite{leahy2022fast}, directly interacts with the syntax tree of LTL formulas, which segments the global specification into smaller, more manageable sub-specifications. 

The works most closely related to ours are~\cite{schillinger2018simultaneous,schillinger2018decomposition,faruq2018simultaneous,robinson2021multiagent,luo2024decomposition}. Works~\cite{schillinger2018simultaneous,schillinger2018decomposition} approached STAP by breaking down temporal logic tasks into  independent tasks, each of which can be completed by a robot.~\cite{faruq2018simultaneous} built upon this idea, incorporating environmental uncertainties and  failures into the planning.~\cite{robinson2021multiagent} expanded further into a multi-objective setting, integrating conflicting objectives such as cost and success probability into a Markov Decision Process (MDP). Our work presents key differences: (i) While the aforementioned studies focus on standard LTL specifications, we consider hierarchical specifications and can solve tasks which are unsolvable by these earlier methods. (ii) In the mentioned studies, tasks between robots are independent thus robot can execute their tasks in parallel. In contrast, our approach also assigns tasks with temporal dependencies to different robots, thus robots may depend on others. Note that~\cite{luo2024decomposition}, developed around the same time, is the only existing work that investigated decision-making under \hltl. This work differs from ours in that, algorithmically,~\cite{luo2024decomposition} proposed a hierarchical approach that first focused on task allocation and then planning, unlike STAP here. 
Theoretically, the planning algorithm introduced in~\cite{luo2024decomposition} lacks formal guarantees of either completeness or optimality.

\subsection{Task Allocation and Planning}

The domain of Multi-Robot Task Allocation (MRTA) is thoroughly studied, and various classifications of MRTA have been developed~\cite{korsah2013comprehensive,gerkey2004formal}, organizing the extensive research in this field. In terms of MRTA, the problem is this work can be  is characterized by four aspects: (i) single-task robots (each robot is limited to executing one task at a time), (ii) single-robot tasks (only one robot is required for each task), (iii) instantaneous assignment (tasks are allocated under all available information at this moment), and (iv) cross-schedule interdependence (the assignment of a robot to a task is influenced not solely by its own schedule but also by the schedules of other robots within the system).

Methods for MRTA primarily fall into two categories: market-based~\cite{quinton2023market} and optimization-based approaches~\cite{chakraa2023optimization}. Market-based methods use economic principles, employing auction and bidding techniques to distribute tasks according to cost and resource. Of these,  auction-based methods see robots acting selfishly, each calculating bids for tasks based on their self-interest~\cite{zavlanos2008distributed}, while consensus-based methods~\cite{choi2009consensus} combine an auction stage with a consensus stage, using this as a conflict resolution mechanism to establish agreement on the final bids. Moreover, MRTA problems can also be addressed using various optimization techniques, such as Integer Linear Programming (ILP) or Mixed Integer Linear Programming (MILP)~\cite{koes2005heterogeneous}. Additionally, meta-heuristic algorithms like genetic algorithms~\cite{patel2020decentralized} or particle swarm optimization~\cite{wei2020particle} are employed for more intricate MRTA challenges where exact methods are unfeasible. 
Generally, research on MRTA assumes that the cost of a robot executing a task is known or known with some uncertainty, focusing less on the planning.


 The domain of integrated task assignment and path planning in multi-robot systems involves both task-level reasoning and motion-level planning. The primary objective is to create collision-free paths for robots, enabling them to accomplish a variety of reach-avoid tasks pending assignment. This field is often associated with the unlabeled version of multi-agent path planning (MAPF), as discussed by~\cite{okumura2023solving}. Various formulations have been proposed, including those by~\cite{edison2011integrated,chen2021integrated,aggarwal2022extended,ma2016optimal}, to name a few.~\cite{chen2021integrated} focused on simultaneous task assignment and path planning for Multi-agent Pickup and Delivery (MAPD) in warehouse settings, where the goal is to efficiently manage agents transporting packages.~\cite{aggarwal2022extended} tackled the combined challenge of task allocation and path planning, where an operator must assign multiple tasks to each vehicle in a fleet, ensuring collision-free travel and minimizing total travel costs.~\cite{ma2016optimal} explored this issue for teams of agents, each assigned the same number of targets as there are agents in the team. Our work differs from these studies in that: (i) We do not assume a predefined set of point-to-point navigation tasks. Instead, tasks, which may include navigation and manipulation, are implicitly defined within the temporal logic specifications. (ii) The presence of logical and temporal constraints between tasks adds complexity to our problem. (iii) In our work, a single robot might be assigned multiple tasks, while others might remain stationary. This contrasts with scenarios where each robot is assigned exactly one task in works mentioned above.

\subsection{Hierarchical Task Models}
{\color{\modifycolor}Studies~\cite{tenenbaum2011grow, kemp2007learning} suggest that hierarchical reasoning enhances humans’ ability to understand the world more effectively.} In classical AI planning, researchers crafted task models reflecting hierarchical structures by employing  domain control knowledge~\cite{wilkins2014practical}. These models have proven to be superior to flat models in terms of interpretability and efficiency, due to the significant reduction in the search space. Hierarchical Task Network (HTN)~\cite{georgievski2015htn}, a commonly used task model, exemplifies this. It presents a hierarchy of tasks, each of which can be executed if it's primitive, or broken down into finer sub-tasks if it's complex. Its plan comprises a set of primitive tasks applicable to the initial world state. Owing to its expressiveness, HTN has been implemented in the planning~\cite{weser2010htn}. There are other hierarchical models such as \texttt{AND/OR} graphs~\cite{de1990and} and sequential/parallel graphs~\cite{cheng2021human}. 

Research combining hierarchical task models with LTL includes studies that use LTL to express temporally extended preferences over tasks and sub-tasks in HTN~\cite{baier2009htn}, as well as research into the expressive power of HTN in combination with LTL~\cite{lin2022expressive}.  However, despite the widespread use of hierarchical task models in classical AI planning, it's intriguing to note the lack of specification hierarchy in temporal logic robotic planning. Our work differs from these studies as we follow an inverse direction; instead of integrating \ltl\ into HTN to express the goal of the planning problem, we incorporate HTN into \ltl, allowing for hierarchical structures within multiple \ltl\ formulas, making them  more capable of expressing complex tasks than a single \ltl\ formula.

\newcolumntype{L}{>{\raggedright\arraybackslash}X}
\newcommand{\thickhline}{\noalign{\hrule height 1.2pt}}
\begin{table*}[!t]
\centering
\begin{tabularx}{\linewidth}{lL|lL}
\thickhline
\textbf{Notation} & \textbf{Description} & \textbf{Notation} & \textbf{Description} \\
\hline
$\sigma$ & finite word over the alphabet & $\ccalA (\phi)$ & NBA of specification $\phi$  \\
 $\ccalQ_\ccalA^0$ & set of initial automaton states & $\ccalQ_\ccalA$ & set of automaton states \\
$\Sigma = 2^{\ccalA\ccalP}$ & alphabet &   $\to_\ccalA$  & transition relation in $\ccalA$ \\
$\ccalA\ccalP$ & set of atomic propositions & $\ccalQ_\ccalA^F$ & set of final automaton states \\
$\ccalT (r)$ & transition system of robot $r$ & $\rho_\ccalA$ & finite run \\
$\ccalS_r$ & states of robot $r$ & $\ccalL_\ccalA$ & accepted language of $\ccalA$ \\
$s_r^0$ & initial state of robot $r$ & $\ccalP(r, \phi)$ & PBA of robot $r$ and specification $\phi$ \\
$\to_r$ & transition relation of robot $r$ & $\ccalQ_\ccalP$ & set of product states \\
$\ccalA\ccalP_r$ & atomic proposition of robot $r$ & $\ccalQ_\ccalP^0$ & set of initial product states \\
$\ccalL_r$ & observation function of robot $r$ & $\to_\ccalP$ & transition relation in $\ccalP$ \\
$\level{i}{k}$ & $i$-th specification at the $k$-th level & $\ccalQ_\ccalP^F$ & set of final states \\
$ \ccalG_h = (\ccalV_h, \ccalE_h)$ & {\color{\modifycolor}specification hierarchy tree} & $c_\ccalP$ & cost function in $\ccalP$ \\
$\Phi_k$ & set of specifications at the $k$-th level & $\tau = \tau_0 \ldots \tau_h$ & {\color{\modifycolor}state-specification sequence} \\
$\Phi_K$ & set of specifications at the lowest level & $\ccalD_\ccalA(\phi)$ & decomposition set of specification $\phi$ \\
$\Phi_{\text{leaf}}$ & set of leaf specifications & $\ccalP(\phi)$ & product team model associated with specification $\phi$ \\
$\zeta_{\text{in}}$ & in-spec switch transition & $\ccalP$ & hierarchical team models \\
$\zeta_{\text{inter}}^1$ & type I inter-spec switch transition & $\zeta_{\text{inter}}^2$ & type II inter-spec switch transition \\
$v = (r, \phi, \mathbf{s}, \mathbf{q})$ & search state & $\Lambda = v_0 \ldots v_n$ & search path \\
$\Lambda' = \Lambda_1\ldots \Lambda_l$ & search path after removing all switch states & $\ccalR_k$ & active robots of the $k$-th path segment \\
\thickhline
\end{tabularx}
\caption{Summary of notations.}
\label{tab:glossary}
 \vspace{-10pt}
\end{table*}

\section{Preliminaries}\label{sec:preliminaries}
{\bf Notation:} Let $\mathbb{N}$ denote the set of all integers, $[K] = \{0, \ldots, K\}$ and $[K]_+ = \{1, \ldots, K\}$ represent the sets of integers from 0 to $K$ and from 1 to $K$, respectively, and $|\cdot|$ denote the cardinality of a set.

In this section, we introduce \ltl\ by presenting its syntax and semantics, and automata-based \ltl\ model checking.

LTL~\cite{baier2008principles} is a type of formal logic whose basic ingredients are a set of atomic propositions $\mathcal{AP}$, the {\color{\modifycolor}Boolean} operators, conjunction $\wedge$ and negation $\neg$, and  temporal operators, next $\bigcirc$ and until $\mathcal{U}$. LTL formulas over $\mathcal{AP}$ abide by the grammar 
\begin{align}\label{eq:grammar}
\phi::=\text{true}~|~\pi~|~\phi_1\wedge\phi_2~|~\neg\phi~|~\bigcirc\phi~|~\phi_1~\mathcal{U}~\phi_2.    
\end{align}
{For brevity, we abstain from deriving  other Boolean and temporal operators, e.g., \textit{disjunction} $\vee$, \textit{implication} $\Rightarrow$, \textit{always} $\square$, \textit{eventually} $\lozenge$, which can be found in \cite{baier2008principles}.} 

{\color{\modifycolor}\ltl~\cite{de2013linear} is a variant of LTL that is interpreted over finite sequences of states while maintaining the same syntax as standard LTL. A finite \textit{word} $\sigma$ over the alphabet $2^{\mathcal{AP}}$ is defined as a finite sequence  $\sigma=\sigma_0\sigma_1\ldots\sigma_h$ with $\sigma_i \in 2^{\ccalA\ccalP}$ for $i \in [h]$. The satisfaction of an \ltl\ formula $\phi$ over a sequence $\sigma$ at an instant $i$, for $i\in [h]$, is inductively defined as follows:
\begin{itemize}
  \item  $\sigma, i  \models \pi\;\; \text{iff}\; \pi \in \sigma_i$.
  \item  $\sigma, i  \models \neg \phi\;\; \text{iff}\; \sigma, i \not\models \phi$.
  \item  $\sigma, i \models \phi_1 \wedge \phi_2 \;\; \text{iff}\; \sigma, i \models \phi_1 \;\text{and}\;  \sigma, i \models \phi_2$.
  \item  $\sigma, i  \models \bigcirc \phi \;\; \text{iff}\; i < h \;\text{and}\;  \sigma, i+1 \models \phi$.
  \item  $\sigma, i  \models \phi_1\,\ccalU\, \phi_2$ \text{iff for some} $j$ such that $i \leq j \leq h$, we have that $\sigma, j \models \phi_2$, and for all $k, i \leq k < j$, we have that $\sigma, k \models \phi_1$.
\end{itemize}

A formula $\phi$ is satisfied by $\sigma$, denoted by $\sigma \models \phi$, if $\sigma, 0 \models \phi$.  The language $\texttt{Words}(\phi)=\left\{\sigma|\sigma\models\phi\right\}$ is defined as the set of words that satisfy the formula $\phi$. An \ltl\ formula $\phi$ can be translated into a Nondeterministic Finite Automaton}:
\begin{defn}[NFA]
{\color{\modifycolor} A \textit{Nondeterministic Finite Automaton} (NFA) $\ccalA$ of an \ltl\ formula $\phi$ over $2^{\mathcal{AP}}$ is defined as a tuple $\ccalA(\phi)=\left(\ccalQ_{\ccalA}, \ccalQ_{\ccalA}^0,\Sigma,\rightarrow_{\ccalA},\mathcal{Q}_\ccalA^F\right)$, where 
\begin{itemize}
    \item $\ccalQ_{\ccalA}$ is the set of states;
    \item $\ccalQ_{\ccalA}^0\subseteq\ccalQ_{\ccalA}$ is a set of initial states;
    \item $\Sigma=2^{\mathcal{AP}}$ is an alphabet;
    \item $\rightarrow_{\ccalA}\, \subseteq\, \ccalQ_{\ccalA}\times \Sigma\times\ccalQ_{\ccalA}$ is the transition relation;
    \item $\ccalQ_{\ccalA}^F\subseteq\ccalQ_{\ccalA}$ is a set of accepting/final states.
\end{itemize} 
}
\end{defn}
{\color{\modifycolor}A \textit{finite run} $\rho_{\ccalA}$ of $\ccalA$ over a finite word $\sigma=\sigma_0\sigma_1\dots\sigma_h$, $\sigma_i =  2^{\mathcal{AP}}$, $\forall i\in [h]$, is a sequence $\rho_{\ccalA}=q_{\ccalA}^0q_{\ccalA}^1 \dots q_{\ccalA}^{h+1}$ such that $q_{\ccalA}^0\in \ccalQ_{\ccalA}^0$ and $(q_{\ccalA}^{i},\sigma_i,q_{\ccalA}^{i+1})\in\rightarrow_{\ccalA}$, $\forall i\in [h]$. A run $\rho_{\ccalA}$ is called \textit{accepting} if $ q_{\ccalA}^{h+1}  \in  \ccalQ_{\ccalA}^F$. The words $\sigma$ that produce an accepting run of $\ccalA$ constitute the accepted language of $\ccalA$, denoted by $\ccalL_{\ccalA}$. Then~\cite{baier2008principles} proves that the accepted language of $\ccalA$ is equivalent to the words of $\phi$, i.e., $\ccalL_{\ccalA}=\texttt{Words}(\phi)$.}

  {The dynamics of robot $r$ is captured by a Transition System:

\begin{defn}[TS]\label{def:ts}
  A {\it Transition System} for robot $r$ is a tuple $\ccalT(r) = \{\ccalS_r, \state{s}{r}{0}, \to_r, \ccalA\ccalP_r, \ccalL_r\}$ where: 
  \begin{itemize}
      \item {\color{\modifycolor}$\ccalS_r$ is the set of discrete states of robot $r$, and $s_r \in \ccalS_r$ denotes a specific state;}
      \item {\color{\modifycolor} $\state{s}{r}{0} \in \ccalS_r$ is the initial state of robot $r$;} 
      \item  $\to_r \,\subseteq\, \ccalS_r \times \ccalS_r$ is the transition relation;
      \item $\mathcal{AP}_r$ is the set of atomic propositions related to robot $r$;
      \item $\ccalL_r:\ccalS_r\rightarrow 2^{\mathcal{AP}_r}$  is the observation (labeling) function that returns a subset of atomic propositions that are satisfied, i.e., $\ccalL_r(s_{r}) \subseteq \mathcal{AP}_r$.
  \end{itemize}
\end{defn}
{\color{\modifycolor}A finite sequence of states $\state{s}{r}{0}, \state{s}{r}{1}, \ldots, \state{s}{r}{h}$ is referred to as a {\it trace}.} Given the transition system $\ccalT(r)$ of robot $r$ and the NFA $\ccalA$ of an \ltl\ formula $\phi$, we can define the \textit{Product Automaton} (PA) $\ccalP(r, \phi) = \ccalT(r) \times \ccalA(\phi)$ as follows~\cite{baier2008principles}:

\begin{defn}[PA]\label{defn:PA}
For a robot $r$ and an \ltl\ formula $\phi$, the \textit{Product Automaton} is defined by the tuple $\ccalP(r, \phi)=(\mathcal{Q}_\ccalP, \mathcal{Q}_\ccalP^0,\longrightarrow_{\ccalP},\mathcal{Q}_\ccalP^F)$, where
\begin{itemize}
    \item $\mathcal{Q}_\ccalP=\ccalS_r \times\mathcal{Q}_{\ccalA}$ is the set of product states;
    \item $\mathcal{Q}_\ccalP^0=\{ \state{s}{r}{0}\}\times\mathcal{Q}_\ccalA^0$ is a set of initial states;
    \item $\longrightarrow_{\ccalP}\subseteq\mathcal{Q}_\ccalP \times \mathcal{Q}_\ccalP$ is the transition relation defined by the rule: $\frac{(s_r \rightarrow_{r} s'_r)\wedge(q\xrightarrow{\ccalL_r\left(s_r \right)}_{\ccalA}q')}{q_{\ccalP}=\left(s_r,q\right)\longrightarrow_\ccalP q_{\ccalP}'=\left(s'_r,q'\right)}$. The transition from the state $q_\ccalP\in\mathcal{Q}_\ccalP$ to $q_\ccalP'\in\mathcal{Q}_\ccalP$, is denoted by $(q_\ccalP,q_\ccalP')\in\longrightarrow_\ccalP$, or $q_\ccalP\longrightarrow_\ccalP q_\ccalP'$;
    \item $\mathcal{Q}_\ccalP^F=\ccalS_r\times\mathcal{Q}_\ccalA^F$ is a set of accepting/final states;
\end{itemize}
\end{defn}

\section{Hierarchical \ltl}\label{sec:hltl}
We first present the syntax and semantics of \hltl, then we analyze its expressiveness.

\subsection{Syntax of \hltl}

{\color{\modifycolor}
\begin{defn}[Hierarchical \ltl]\label{def:hltl}
\hltl\ is structured into $K$ levels, labeled $L_1, \ldots, L_K$, arranged from the highest to the lowest. Each level $L_k$ with $k \in [K]_+$ contains $n_k$ \ltl\ formulas. The \hltl\ specification is represented as $\Phi = \left\{\level{i}{k} \,|\, k \in [K]_+, i \in [n_k]_+\right\}$, where $\level{i}{k}$ denotes the $i$-th \ltl\ formula at level $L_k$. Let $\Phi_k$ denote the set of formulas at level $L_k$, and let $\texttt{Prop}(\level{i}{k})$ represent the set of propositions appearing in formula $\level{i}{k}$. The \hltl\ follows these rules:
\begin{enumerate}
    \item \label{cond:highest} There is exactly one formula at the highest level: $n_1 = 1$.
    \item \label{cond:inclusion} Each formula at level $L_k$ consists either entirely of atomic propositions, i.e., $\texttt{Prop}(\level{i}{k}) \subseteq \ccalA\ccalP$, or entirely of formulas from the next lower level, i.e., $\texttt{Prop}(\level{i}{k}) \subseteq \Phi_{k+1}$.
    \item \label{cond:union} Each formula at level $L_{k+1}$ appears in exactly one formula at the next higher level: $\level{i}{k+1} \in \bigcup_{j \in [n_k]_+} \texttt{Prop}(\level{j}{k})$ and $\texttt{Prop}(\level{j_1}{k}) \cap \texttt{Prop}(\level{j_2}{k}) = \varnothing$, for $j_1, j_2 \in [n_k]_+$ and $j_1 \not= j_2$.
\end{enumerate}
\end{defn}

}
 
For a formula $\level{i}{k}$ at a non-highest level, we slightly bend the notation to use $\level{i}{k}$ to represent the same symbol at the higher level $L_{k-1}$, which we refer to as {\it composite proposition}. $\level{i}{k}$ represents not only the $i$-th formula at level $L_{k}$, but also the corresponding composite proposition at level $L_{k-1}$. When $\level{i}{k}$ appears at the right side in a certain formula, we consider it as a composite proposition. When it appears at the left side as a standalone formula, we refer to it as a specification.  

{\color{\modifycolor}
\begin{exmp}[\hltl]\label{exmp:hltl}
    Consider a task where robots are tasked with {\it first} picking and placing item $a$ and then fetching items $b$ and $c$ {\it in any order}. The \ltl\ specification is 
\begin{align*}
\phi  \,= \, & \Diamond (s_a \wedge \Diamond (t_a \wedge \Diamond (s_b \wedge \Diamond t_b) \wedge \Diamond (s_c \wedge \Diamond t_c))),
\end{align*}
where sub-formula $\Diamond (s_x \wedge \Diamond t_x)$ denotes the event of {\it first} picking item $x$ from its source location $s_x$ and {\it then} placing it at the target location $t_x$.  This formula is less interpretable by arranging all sub-formulas side by side. Furthermore, its automaton has 17 states and 102 transitions, which is unexpectedly large for a task of this complexity. The \hltl\ specifications of the same task are 
\begin{align}\label{eq:example_hltl}
\begin{aligned}
     L_1: \quad & \level{1}{1} = \Diamond (\level{1}{2} \wedge \Diamond \level{2}{2} \wedge \Diamond \level{3}{2}) \\
    L_2: \quad &  \level{1}{2} = \Diamond (s_a \wedge \Diamond t_a) \\
                & \level{2}{2} = \Diamond (s_b \wedge \Diamond t_b)\\
                & \level{3}{2} = \Diamond (s_c \wedge \Diamond t_c).
\end{aligned}
\end{align}
There are two levels, $L_1$ and $ L_2$, with $L_1$ having one formula and $L_2$ having three formulas.  The symbol $\level{1}{2}$, appearing on the left side of the equal sign, is a specification at level $L_2$ but is a composite proposition at level $L_1$ since it is on the right side of the equal sign. $\Phi_1 = \{\level{1}{1}\}$ and   $\Phi_2 = \{\level{1}{2}, \level{2}{2}, \level{3}{2}\}$, $\texttt{Prop}(\level{1}{1}) = \{\level{1}{2}, \level{2}{2}, \level{3}{2}\} \subseteq \Phi_2$ and $\texttt{Prop}(\level{1}{2})=\{s_a, t_a\} \subseteq \ccalA\ccalP$. This representation  reduces the length of each formula, leading to smaller {\color{\modifycolor}automata} and improving interpretability. Their {\color{\modifycolor}automata} have a total of 8 states and 18 transitions. Note that the hierarchical representation is not unique. Depending on the level of granularity in the hierarchy, the same task can be represented as:
\begin{align}\label{eq:fine_example_hltl}
     L_1: \quad & \level{1}{1} = \Diamond (\level{1}{2} \wedge \Diamond \level{2}{2}) \nonumber\\
    L_2: \quad &  \level{1}{2} = \Diamond (s_a \wedge \Diamond t_a), \quad
                  \level{2}{2} = \Diamond \level{1}{3} \wedge \Diamond \level{2}{3} \\
    \hspace*{1.cm}  L_3: \quad & \level{1}{3} = \Diamond (s_b \wedge \Diamond t_b) , \quad \,
                 \level{2}{3} = \Diamond (s_c \wedge \Diamond t_c). \hspace*{0.8cm} \square \nonumber
\end{align}
\end{exmp}}


{\color{\modifycolor}
\begin{defn}[Specification hierarchy tree]  The specification hierarchy tree, denoted as \(\ccalG_h = (\ccalV_h, \ccalE_h)\), is a tree where each node represents a specification within the \hltl, and an edge \((u, v)\) indicates that specification \(u\) contains specification \(v\) as a composite proposition. Any \hltl\ specifications can be turned into a specification hierarchy tree.
\end{defn}
}

\begin{figure}[!t]
    \centering
     \subfigure[Specification hierarchy tree]{
      \label{fig:hierarchy}
      \includegraphics[width=0.45\linewidth]{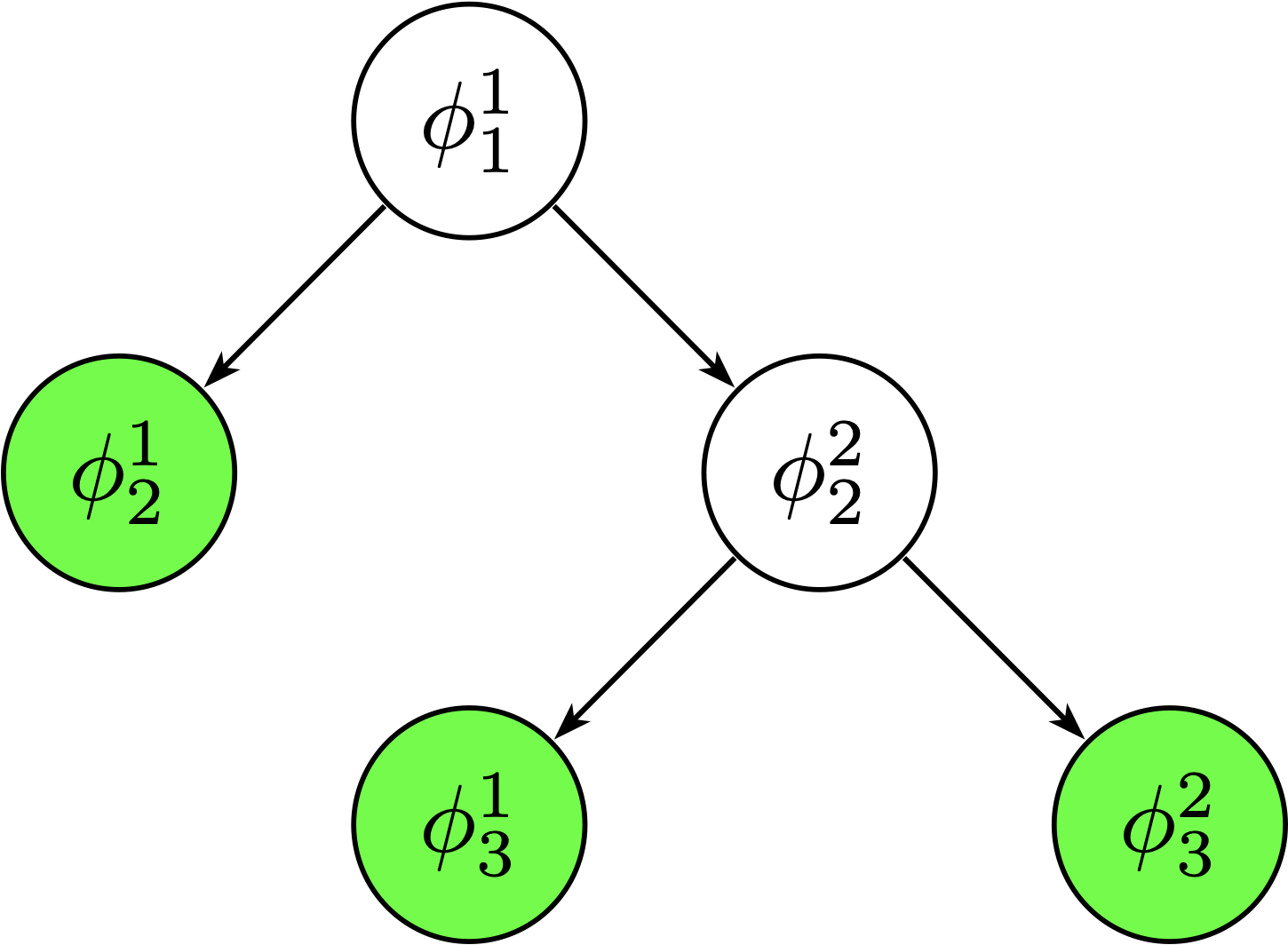}}
     \subfigure[NFA of \ltl\ $\Diamond a \wedge \Diamond b$]{
      \label{fig:nba}
      \includegraphics[width=0.49\linewidth]{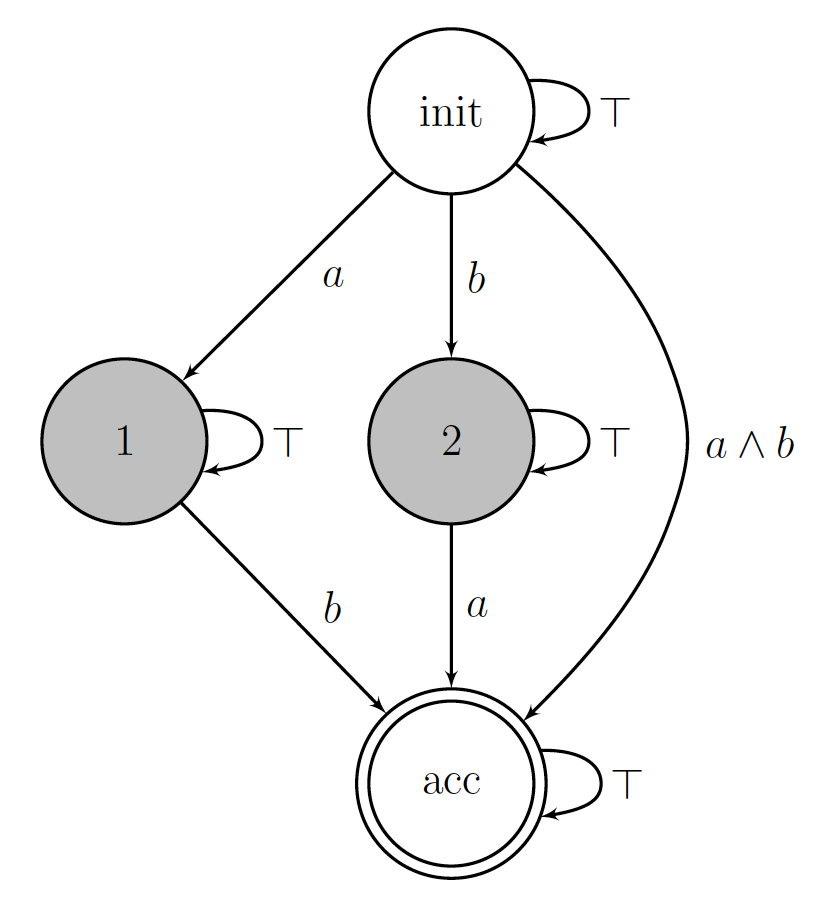}}
    \caption{The specification hierarchy tree of \hltl\ specifications in~\eqref{eq:fine_example_hltl} 
 {\color{\modifycolor} where green nodes represent leaf specifications}.  The NFA corresponds to the \ltl\ specification in Example~\ref{exmp:decomp}, where non-trivial decomposition states are colored in gray.}
    \label{fig:d}
 \end{figure}

    

{\color{\modifycolor}From a tree perspective, the level $k$ of a specification node \( \level{i}{k} \) corresponds to its depth, defined as the number of nodes along the longest path from the root to that node.}

\begin{defn}[Leaf and Non-leaf Specifications]
A specification is termed as a leaf specification if the associated node in the graph \(\ccalG_h\) does not have any children; otherwise, it is referred to as a non-leaf specification.
\end{defn}

{\color{\modifycolor}  Let $\Phi_{\text{leaf}}$ denote the set of leaf specifications. Based on Def.~\ref{def:hltl}, leaf specifications consist exclusively of atomic propositions, while non-leaf specifications consist solely of composite propositions. }
\begin{cexmp}{exmp:hltl}{Specification hierarchy tree}~\label{exmp:syntax}
The specification hierarchy tree is illustrated in Fig.~\ref{fig:hierarchy}.\hfill $\square$
\end{cexmp}



\subsection{Semantics of \hltl}
\hltl\ cannot be interpreted by simply substituting composite propositions with their corresponding specifications. For example, replacing composite propositions in Eq.~\eqref{eq:example_hltl} at level $L_1$ generates a standard specification:
$$
\level{1}{1} = \Diamond ( \underbrace{\Diamond (s_a \wedge \Diamond t_a)}_{\level{1}{2}} 
\wedge
\Diamond ( \underbrace{\Diamond (s_b \wedge \Diamond t_b)}_{\level{2}{2}} )
\wedge
\Diamond ( \underbrace{\Diamond (s_c \wedge \Diamond t_c)}_{\level{3}{2}} ) 
),
$$
which is challenging to interpret. 
\begin{cexmp}{exmp:hltl}{Interpretation}~\label{exmp:interpretation}
 The way of interpretation is to start from the bottom level and progress upwards, verifying if the satisfaction of lower specifications allows the upper ones to move forward. We provide two separate cases for specifications in~\eqref{eq:example_hltl}. Case (i): If $s_a$ now holds true, then $\level{1}{2}$ transitions to $\Diamond t_a$ based on the rules of progression~\cite{bacchus2000using} that generate sub-formulas to be satisfied, while $\level{1}{1}$ remains unaffected as none of its composite propositions are fulfilled. If $t_a$ subsequently holds true, $\level{1}{2}$ is satisfied. Consequently, $\level{1}{1}$ then simplifies to $\Diamond \level{2}{2} \wedge \Diamond \level{3}{2}$. Following this logic,  $\level{1}{1}$ will be fulfilled when both $\level{2}{2}$ and $\level{3}{2}$ are satisfied, regardless of the order. Case (ii): Assuming the robot is capable of carrying multiple items. If $s_b$ is satisfied, then $\level{2}{2}$ progresses to $\Diamond t_b$. Subsequently, as $s_a$ and $t_a$ sequentially become true, $\level{1}{1}$ transitions to $\Diamond \level{2}{2} \wedge \Diamond \level{3}{2}$. At this stage, $\level{2}{2}$ only requires $t_b$ to be satisfied. We allow the partial satisfaction of $\level{2}{2}$ provided that it achieves full satisfaction no earlier than $\level{1}{2}$. \hfill $\square$
\end{cexmp}

{\color{\modifycolor}Given that \hltl\ includes multiple specifications, we consider multiple systems, particularly transition systems. At any given state, a system may be addressing a particular specification. We associate the state  $s_r$  of a system  $\ccalT(r)$  with a leaf specification  $\psi_r \in \Phi_{\text{leaf}}$. This association establishes a state-specification pair  $(s_r, \psi_r)$, indicating that at state  $s_r$, the system  $\ccalT(r)$  is engaged in satisfying the leaf specification  $\psi_r$.}
\begin{defn}[State-Specification Sequence]\label{defn:traj}
{\color{\modifycolor}A state-specification sequence with a horizon \( h \), represented as \( \tau \), is a timed sequence \( \tau = \tau_0\tau_1\tau_2 \ldots\tau_h \). Here, \( \tau_i = ((\state{s}{1}{i}, \state{\psi}{1}{i}), (\state{s}{2}{i}, \state{\psi}{2}{i}), \ldots, (\state{s}{N}{i}, \state{\psi}{N}{i})) \) is the collective state-specification pairs of \( N \) systems at the \( i \)-th timestep, where \( \state{s}{r}{i} \in \ccalS_r \), and \( \state{\psi}{r}{i} \in \Phi_{\text{leaf}} \cup \{\epsilon\} \), with \( \epsilon \) indicating the system's non-involvement in any leaf specification at that time.}
\end{defn}

{\color{\modifycolor}

Given the state-specification sequence $\tau$ and a leaf specification $\phi$, we generate a word as $\sigma = \sigma_0,\sigma_1,\ldots,\sigma_n$, where $\sigma_i = \left\{\ccalL_r(s_r^i)\ |\  (\state{s}{r}{i}, \state{\psi}{r}{i}) \in \tau_i\; \text{and}\; \state{\psi}{r}{i} = \phi\right\}$, represents the collective observations generated by the systems engaging in the leaf specification $\phi$ at instant $i$. Furthermore, given a leaf (non-leaf) specification $\phi$ and an {\it input} word $\sigma = \sigma_0, \sigma_1, \ldots, \sigma_n$ with $\sigma_i \subseteq \texttt{Prop}(\phi)$, consisting of its atomic (composite) propositions, we construct an {\it output} word of $\phi$, denoted by $\sigma' =  \sigma'_0, \sigma'_1,\ldots, \sigma'_n$, where $\sigma'_i$ is either $\varnothing$ or $\{\phi\}$, tracking the satisfaction of $\phi$. Let $\sigma_{i, j} = \sigma_{i},\ldots,\sigma_{j}$ denote the segment of the input word between indices $i$ and $j$. 
\begin{defn}[Output word]
  Given a specification $\phi$ and an input word $\sigma = \sigma_0, \sigma_1, \ldots, \sigma_n$ with $\sigma_i \subseteq \texttt{Prop}(\phi)$, the output word $\sigma' = \sigma'_0, \sigma_1', \ldots, \sigma'_n$ of $\phi$ is defined as:
\begin{itemize}
    \item By default, $\sigma'_{-1} = \{\phi\} $.
    \item $\sigma'_j = \{\phi\}$ if $\sigma_{i+1, j}, i+1 \models \phi$, where $i$ (with $i < j$) is the most recent instant when $\phi$ was satisfied, i.e., $i = max \{k \,|\, \sigma'_k = \{\phi\}, k = -1, 0, \ldots, j-1 \}$.
\end{itemize}
\end{defn}

The satisfaction relation $\models$ follows the same definition as introduced in Section~\ref{sec:preliminaries}, treating composite propositions as atomic propositions for non-leaf specifications.
The progress of $\phi$ resets after it is satisfied at instant $i$, and it is considered satisfied again at instant $j$ if the input word segment $\sigma_{i+1, j}$ satisfies $\phi$. In the context of robots, this implies that the duration of task $\phi$ extends from $i+1$ to $j$, with the completion time at instant $j$. The task’s completion status is not persistent, meaning that once completed, it must start from scratch if repeated.  The output words of specifications at level $L_{k+1}$ serve as the input words and produce output words of specifications at level $L_k$. Thus, starting from the words of leaf specifications generated by the state-specification sequence $\tau$ and proceeding in a bottom-up manner, we ultimately obtain an output word of the root specification $\level{1}{1}$. The whole \hltl\ is satisfied by a state-specification sequence $\tau$ if the output word of $\level{1}{1}$ contains $\level{1}{1}$. Algorithmically, a post-order traversal can be used to traverse the specification hierarchy tree to compute the output word of  $\level{1}{1}$,
as outlined in Alg.~\ref{alg:outputword}.

\begin{defn}[Semantics]\label{def:semantics}
Given \hltl\ specifications $\Phi$ and a state-specification sequence $\tau$, the sequence $\tau$ satisfies $\Phi$ if the root specification $\level{1}{1}$ appears in the output word of $\level{1}{1}$ generated by Alg.~\ref{alg:outputword}, which is returned by the function $\texttt{GenerateOutputWord}(\tau, \level{1}{1}, \Phi)$.
\end{defn}
\begin{algorithm}[t]
{
\color{\modifycolor}\caption{\texttt{GenerateOutputWord}($\tau, \Phi, \phi$)}
\LinesNumbered
\label{alg:outputword}
\KwIn {state-specification sequence $\tau$, \ltl\  specification $\phi$, \hltl\ specifications $\Phi$}
\KwOut {output word $\sigma'$ of specification $\phi$}
\Comment*[r]{{\color{black}Get input word $\sigma$ [lines~\ref{semantics:leaf}-\ref{semantics:non-leaf}]}}
\Comment*[r]{{\color{black}Leaf specification}}
\If{$\phi \in \Phi_{\text{leaf}}$\label{semantics:leaf}}{
    $\sigma (\phi) = \sigma_0, \sigma_1, \ldots, \sigma_n$ where $\sigma_i = \left\{\ccalL_r(s_r^i)\ |\  (\state{s}{r}{i}, \state{\psi}{r}{i}) \in \tau_i\; \text{and}\; \state{\psi}{r}{i} = \phi\right\}$\;
}
\Comment*[r]{{\color{black}Non-leaf specification}}
\Else{ 
\For{$\phi' \in \texttt{Prop}(\phi)$}{
    $\sigma(\phi') = \texttt{GenerateOutputWord}(\tau, \Phi, \phi')$ \label{semantics:merge}\;
}

\Comment*[r]{{\color{black}Merge output words at the lower level as the input word of $\phi$}}
$\sigma(\phi) = \sigma_0, \sigma_1, \ldots, \sigma_n$ where 
    $\sigma_i = \left\{ \sigma_i(\phi') \;|\; \forall \phi' \in \texttt{Prop}(\phi)   \right\}$\label{semantics:non-leaf} \;
}
\Comment*[r]{{\color{black}Get output word $\sigma'$ [lines~\ref{semantics:i}-\ref{semantics:empty}] }}
$i = -1$ \label{semantics:i} \Comment*[r]{{\color{black}Latest time when $\phi$ is true}}
\For{$j \in \{0, 1, \ldots, n\}$}{
    \If{$\sigma_{i+1, j}, i+1 \models \phi$}{
        $\sigma'_j = \{\phi\}$\;
        $i = j$\;
    }\Else{
        $\sigma'_j = \varnothing$ \label{semantics:empty} \;
    }
}
\Return $\sigma'$\;
}
\end{algorithm}

\begin{figure}[t]  
    \centering
\resizebox{\linewidth}{!}{ 
\begin{tikzpicture}[node distance=1cm, >=stealth, auto]

    \node[font=\LARGE] (phi1) at (-1.5, 1.5) {\(\displaystyle t:\)};
    \node[font=\LARGE][font=\LARGE] (A) at (0, 1.5) {\(0\)};
    \node[font=\LARGE] (B) at (2, 1.5) {\(1\)};
    \node[font=\LARGE] (C) at (4, 1.5) {\(2\)};
    \node[font=\LARGE] (D) at (6, 1.5) {\(3\)};
    \node[font=\LARGE] (E) at (8, 1.5) {\(4\)};
    \node[font=\LARGE] (F) at (10, 1.5) {\(5\)};
    \node[font=\LARGE] (G) at (12, 1.5) {\(6\)};
    \node[font=\LARGE] (G1) at (14, 1.5) {\(7\)};
    \node[font=\LARGE] (G2) at (16, 1.5) {\(8\)};
    \node[font=\LARGE] (G3) at (18, 1.5) {\(9\)};

    \node[font=\LARGE] (phi1) at (-1.5, 0) {obsv:};
    \node[circle, draw, minimum size=1cm] (A) at (0, 0) {};
    \node[circle, draw, minimum size=1cm, fill=black!70, text=white, font=\large] (B) at (2, 0) {$\{s_a\}$};
    \node[circle, draw, minimum size=1cm, fill=black!70, text=white, font=\large] (C) at (4, 0) {$\{t_a\}$};
    \node[circle, draw, minimum size=1cm, fill=black!70, text=white, font=\large] (D) at (6, 0) {$\{s_b\}$};
    \node[circle, draw, minimum size=1cm] (E) at (8, 0) {};
    \node[circle, draw, minimum size=1cm, fill=black!70, text=white, font=\large] (F1) at (10, 0) {$\{t_b\}$};
    \node[circle, draw, minimum size=1cm] (G) at (12, 0) {};
    \node[circle, draw, minimum size=1cm, fill=black!70, text=white, font=\large] (G1) at (14, 0) {$\{s_c\}$};
    \node[circle, draw, minimum size=1cm, fill=black!70, text=white, font=\large] (G2) at (16, 0) {$\{t_c\}$};
    \node[circle, draw, minimum size=1cm] (G3) at (18, 0) {};

    \draw[->] (A) -- (B);
    \draw[->] (B) -- (C);
    \draw[->] (C) -- (D);
    \draw[->] (D) -- (E);
    \draw[->] (E) -- (F1);
    \draw[->] (F1) -- (G);
    \draw[->] (G) -- (G1);
    \draw[->] (G1) -- (G2);
    \draw[->] (G2) -- (G3);

    \node[font=\LARGE] (phi2) at (-1.5, -1.5) {\(\phi_2^1:\)};
    \node[circle, draw, minimum size=1cm] (E) at (0, -1.5) {};
    \node[circle, draw, minimum size=1cm] (F) at (2, -1.5) {};
    \node[circle, draw, minimum size=1cm,  fill=black!70, text=white, font=\large] (G) at (4, -1.5) {$\{\level{1}{2}\}$};
    \node[circle, draw, minimum size=1cm]  (H) at (6, -1.5) {};
    \node[circle, draw, minimum size=1cm] (I) at (8, -1.5) {};
    \node[circle, draw, minimum size=1cm]  (J1) at (10, -1.5) {};
    \node[circle, draw, minimum size=1cm] (K) at (12, -1.5) {};
    \node[circle, draw, minimum size=1cm] (K1) at (14, -1.5) {};
    \node[circle, draw, minimum size=1cm] (K2) at (16, -1.5) {};
    \node[circle, draw, minimum size=1cm] (K3) at (18, -1.5) {};

    \draw[->] (E) -- (F);
    \draw[->] (F) -- (G);
    \draw[->] (G) -- (H);
    \draw[->] (H) -- (I);
    \draw[->] (I) -- (J1);
    \draw[->] (J1) -- (K);
    \draw[->] (K) -- (K1);
    \draw[->] (K1) -- (K2);
    \draw[->] (K2) -- (K3);
    \draw[->, red, dashed, thick] (C) to[out=-45, in=45] (G);

    \node[font=\LARGE] (phi2) at (-1.5, -3) {\(\phi_2^2:\)};
    \node[circle, draw, minimum size=1cm] (L) at (0, -3) {};
    \node[circle, draw, minimum size=1cm] (M) at (2, -3) {};
    \node[circle, draw, minimum size=1cm] (N) at (4, -3) {};
    \node[circle, draw, minimum size=1cm] (O) at (6, -3) {};
    \node[circle, draw, minimum size=1cm] (P) at (8, -3) {};
    \node[circle, draw, minimum size=1cm, fill=black!70, text=white, font=\large] (Q) at (10, -3) {$\{\level{2}{2}\}$};
    \node[circle, draw, minimum size=1cm] (R) at (12, -3) {};
    \node[circle, draw, minimum size=1cm]  (S) at (14, -3) {};
    \node[circle, draw, minimum size=1cm]  (T) at (16, -3) {};
    \node[circle, draw, minimum size=1cm] (U) at (18, -3) {};

    \draw[->] (L) -- (M);
    \draw[->] (M) -- (N);
    \draw[->] (N) -- (O);
    \draw[->] (O) -- (P);
    \draw[->] (P) -- (Q);
    \draw[->] (Q) -- (R);
    \draw[->] (R) -- (S);
    \draw[->] (S) -- (T);
    \draw[->] (T) -- (U);
    \draw[->, red, dashed, thick] (F1) to[out=-45, in=45] (Q);

    \node[font=\LARGE] (phi2) at (-1.5, -4.5) {\(\level{3}{2}:\)};
    \node[circle, draw, minimum size=1cm] (L) at (0, -4.5) {};
    \node[circle, draw, minimum size=1cm] (M) at (2, -4.5) {};
    \node[circle, draw, minimum size=1cm]  (N1) at (4, -4.5) {};
    \node[circle, draw, minimum size=1cm] (O) at (6, -4.5) {};
    \node[circle, draw, minimum size=1cm] (P) at (8, -4.5) {};
    \node[circle, draw, minimum size=1cm]  (Q) at (10, -4.5) {};
    \node[circle, draw, minimum size=1cm] (R) at (12, -4.5) {};
    \node[circle, draw, minimum size=1cm] (R1) at (14, -4.5) {};
    \node[circle, draw, minimum size=1cm, fill=black!70, text=white, font=\large]  (R2) at (16, -4.5) {$\{\level{3}{2}\}$};
    \node[circle, draw, minimum size=1cm] (R3) at (18, -4.5) {};

    \draw[->] (L) -- (M);
    \draw[->] (M) -- (N1);
    \draw[->] (N1) -- (O);
    \draw[->] (O) -- (P);
    \draw[->] (P) -- (Q);
    \draw[->] (Q) -- (R);
    \draw[->] (R) -- (R1);
    \draw[->] (R1) -- (R2);
    \draw[->] (R2) -- (R3);
    \draw[->, red, dashed, thick] (G2) to[out=-45, in=45] (R2);

    \node[font=\LARGE] (phi2) at (-1.5, -6) {$\level{1}{1}:$};
    \node[circle, draw, minimum size=1cm] (L) at (0, -6) {};
    \node[circle, draw, minimum size=1cm] (M) at (2, -6) {};
    \node[circle, draw, minimum size=1cm] (N) at (4, -6) {};
    \node[circle, draw, minimum size=1cm] (O) at (6, -6) {};
    \node[circle, draw, minimum size=1cm] (P) at (8, -6) {};
    \node[circle, draw, minimum size=1cm] (Q) at (10, -6) {};
    \node[circle, draw, minimum size=1cm] (R) at (12, -6) {};
    \node[circle, draw, minimum size=1cm] (R1) at (14, -6) {};
    \node[circle, draw, minimum size=1cm, fill=black!70, text=white, font=\large]  (Z) at (16, -6) {$\{\level{1}{1}\}$};
    \node[circle, draw, minimum size=1cm] (R3) at (18, -6) {};

    \draw[->] (L) -- (M);
    \draw[->] (M) -- (N);
    \draw[->] (N) -- (O);
    \draw[->] (O) -- (P);
    \draw[->] (P) -- (Q);
    \draw[->] (Q) -- (R);
    \draw[->] (R) -- (R1);
    \draw[->] (R1) -- (Z);
    \draw[->] (Z) -- (R3);
    \draw[->, red, dashed, thick] (R2) to[out=-45, in=45] (Z);

\end{tikzpicture}
}
 \caption{\color{\modifycolor}{Each row illustrates either the observation or the output word of a specification. The filled nodes indicate scenarios where an atomic or composite proposition is satisfied. The observations ($\texttt{obsv}$) generate the input word for leaf specifications $\level{1}{2}\sim \level{3}{2}$, while the combination of output words for specifications $\level{1}{2}\sim \level{3}{2}$ serve as the input word for the non-leaf specification $\level{1}{1}$. The dashed arrows illustrate the correspondence between specific inputs that lead to the satisfaction of a given specification.}} 
    \label{fig:word}  
\end{figure}

\begin{cexmp}{exmp:hltl}{Semantics}~\label{exmp:word}
The output words for the specifications in~\eqref{eq:example_hltl} are illustrated in Fig.~\ref{fig:word}.\hfill $\square$
\end{cexmp}
}

\begin{rem}
The syntax and semantics for \hltl\ here are inspired by the domain of robots and may not be applicable to other domains. Our experience indicates that a top-down reasoning approach is effective in constructing \hltl\ specifications from task descriptions. This entails recognizing the inclusion relationships among tasks and organizing them into various levels of abstraction, ranging from general to specific. At each level, tasks that are closely related should be combined into single formulas, while those with weaker connections should be grouped into separate formulas. 
\end{rem}
{\color{\modifycolor}
\begin{rem}
    The hierarchical format can be applied to syntactically co-safe LTL (sc-LTL), a syntactic subset of LTL, to formulate hierarchical sc-LTL. Sc-LTL formulas can be satisfied by finite sequences followed by any infinite repetitions.
\end{rem}
}

\subsection{Expressiveness}
Temporal logics are interpreted over paths in finite Kripke structures. The Kripke structure is identical to the transition system as presented in Def.~\ref{def:ts} in this work. A Kripke structure, denoted by $\ccalK$, is deemed finite when it has a finite number of states. $\ccalK$ is considered a model for a temporal logic formula $\phi$ if every initial finite path within $\ccalK$ satisfies $\phi$.

\begin{defn}[Expressive power~\cite{bozzelli2018interval}]
Considering two logics $\ccalL_1$ and $\ccalL_2$, along with two formulas $\phi_1 \in \ccalL_1$ and $\phi_2 \in \ccalL_2$, the formula $\phi_1$ is {\it equivalent} to $\phi_2$ if, for any finite Kripke structure $\ccalK$, $\ccalK$ models $\phi_1$ if and only if it models $\phi_2$. The logic $\ccalL_2$ is {\it subsumed} by $\ccalL_1$, denoted as $\ccalL_1 \geq \ccalL_2$, when for every formula $\phi_2 \in \ccalL_2$, there is a corresponding formula $\phi_1 \in \ccalL_1$ that is equivalent to $\phi_2$. Additionally, $\ccalL_1$ is {\it (strictly) more expressive} than $\ccalL_2$ if $\ccalL_1 \geq \ccalL_2$ and $\ccalL_2 \not\geq \ccalL_1$.
\end{defn}

\begin{theorem}[Expressiveness]
   \hltl\ is more expressive than the standard \ltl.
\end{theorem}
\begin{proof}
    We first show that, for any standard \ltl, there exists an equivalent hierarchical counterpart. Therefore, the standard \ltl\ is subsumed by the \hltl. Given a standard \ltl\ \(\phi\), an equivalent two-level \hltl\ can be constructed in a straightforward manner, where the top level is \(\Diamond \phi\) and the bottom level is \(\phi\) itself, that is,
    \begin{align*}
        L_1:   \level{1}{1} =  \Diamond\level{1}{2}, \quad  L_2:   \level{1}{2} = \phi.    
    \end{align*}

Next, we introduce a counter-example to demonstrate the absence of the standard equivalent for certain \hltl\ specifications. Considering a Kripke structure with two atomic propositions $\{a, b\}$, the following \hltl\ specifications include two mutually exclusive leaf specifications:
\begin{align*}
     L_1: \quad &  \level{1}{1} =  \Diamond\level{1}{2} \wedge \, \Diamond \, \level{2}{2}\\
    L_2: \quad &  \level{1}{2} =   \neg b\,\mathcal{U}\, (a \wedge \neg b), \quad \level{2}{2} =   \neg a\,\mathcal{U}\, (b \wedge \neg a) .
 \end{align*}%
 While leaf specifications are mutually exclusive, the introduction of the state-specification sequence enables the association of each state with one specification, facilitating the separate treatment of leaf specifications. {\color{\modifycolor}For example, when these two specifications are associated with different Kripke structures,
 they can be realistically fulfilled.} However, when solely using atomic propositions \(a\) and \(b\), combining two leaf specifications into an \ltl\ specification is unfeasible, as it inevitably results in a contradiction, reducing the specification to \(\texttt{False}\).
\end{proof}



\section{Problem Formulation}
We first present an assumption on \hltl, and then introduce a variant of the product team model defined in work~\cite{schillinger2018simultaneous}, which aggregates the set of PA for a given specification. Following this, we develop the hierarchical team model that integrates a set of product team models.

{\color{\modifycolor}
\subsection{Assumptions on \hltl}
When considering the composite propositions in the context of  robot tasks, they correspond to sub-tasks. For example, one such task could be retrieving a beverage from the refrigerator, such as a bottle of apple juice or orange juice, where completing just one  sub-task suffices. Based on this observation, we assume that, for a task to be successfully completed, it is sufficient that each of its sub-tasks occurs at most once. This assumption applies only to non-leaf specifications.
 
\begin{asmp}[Bounded One-Time Satisfaction]\label{asmp:semantics}
    For any non-leaf specification \( \level{i}{k} \) at level $L_k$, there exists an accepting input word 
    $
    \sigma = \sigma_0, \sigma_1, \dots, \sigma_n
    $
    where \( \sigma_j \subseteq \texttt{Prop}(\level{i}{k}) \subseteq \Phi_{k+1} \) for $j \in [n]$, such that 
    $
    \sum_{j=0}^{n} \mathbb{I}(\level{l}{k+1} \in \sigma_j) \leq 1,
    $ for any \( l \in [|\Phi_{k+1}|]_+ \),
    where \( \mathbb{I} \) is the indicator function.
\end{asmp}
}

\subsection{Product Team Model}
Inspired by work~\cite{schillinger2018simultaneous}, we make the following assumption.
\begin{asmp}[Decomposition]\label{asmp:decomp}
     Every leaf specification $\phi \in \Phi_\text{leaf}$ can be decomposed into a set of tasks such that 
     \begin{itemize}
         \item {\bf Independence:} tasks can be executed independently;
         \item {\bf Completeness:} completion of all tasks, irrespective of the order, satisfies the leaf specification $\phi$.
     \end{itemize}
\end{asmp}

\begin{defn}[Decomposition Set~\cite{schillinger2018simultaneous}]
The decomposition set $\ccalD_\ccalA \subseteq \ccalQ_\ccalA$ of the NFA $\ccalA$ contains all states $q$ such that each state leads to a decomposition of tasks.
\end{defn}

\begin{exmp}[Decomposition Set]\label{exmp:decomp}
    Consider an \ltl\ specification \(\phi = \Diamond a \wedge \Diamond b\), with its corresponding NFA depicted in Fig.~\ref{fig:nba}. For any two-part satisfying word \(\sigma = \sigma_1 \sigma_2\), where \(\sigma_1\) leads to a run ending at state 2, the reverse sequence \(\sigma_2 \sigma_1\) also satisfies the specification by leading to a run through state 1. This indicates that the tasks before and after state 2 can be completed in any order. Therefore, state 2 belongs to the decomposition set. Similarly, state 1 also belongs to the decomposition set. By default, all initial and accepting states in an NFA are considered decomposition states. \hfill $\square$
\end{exmp}

\begin{defn}[Product Team Model \cite{schillinger2018simultaneous}]
   Given a leaf specification $\phi \in \Phi_{\text{leaf}}$, the product team model $\ccalP(\phi)$ consists of $N$ product models $\ccalP(r, \phi)$ and is given by the tuple $\ccalP(\phi) :=(\ccalQ_{\ccalP}, \ccalQ^0_\ccalP, \to_{\ccalP}, \ccalQ^F_{\ccalP})$: 
\begin{itemize}
\item $\ccalQ_{\ccalP} = \left\{(r, s, q) : r \in [N]_+, (s, q) \in  \ccalQ_\ccalP(r, \phi)\right\}$ is the set of state;
\item $\ccalQ^0_\ccalP = \left\{(r, s, q) \in  \ccalQ_\ccalP: r = 1, (s, q) \in \ccalQ_\ccalP^0 (1, \phi)\right\}$ is the set of initial states;
\item $\to_\ccalP \, \in \ccalQ_{\ccalP} \times \ccalQ_{\ccalP}$ is the transition relation. One type of transition, referred to as {\it in-spec transitions}, occurs inside a PA. That is, $ (r, s, q) \to_\ccalP (r, s', q')$ if $(s, q) \to_{\ccalP(r, \phi)} (s', q')$; the other type of transition relation, denoted by $\zeta_{\text{in}}$, occurs between two PAs and is defined below in Def.~\ref{defn:in-spec};
\item $\ccalQ^F_{\ccalP} = \{(r, s, q) \in \ccalQ_\ccalP: q \in \ccalQ^F_\ccalA(\phi)\}$ is the set of accepting/final states ; 
\end{itemize}
\end{defn}

\begin{defn}[In-Spec Switch Transition (modified from~\cite{schillinger2018simultaneous})]~\label{defn:in-spec}
   Given a specification $\phi$, the set of in-spec switch transitions in $\ccalP(\phi)$ is given by $\zeta_{\text{in}} \subset \ccalQ_\ccalP \times \ccalQ_\ccalP$. A transition $\left((r, s, q), (r', s', q') \right)$ belongs to $\zeta_{\text{in}}$ if and only if it: 
    \begin{itemize}
    \item points to the next robot, $r'= r + 1$; 
    \item preserves the NFA progress, $q= q'$; 
    \item represents a decomposition choice, $q \in \ccalD_\ccalA(\phi)$.
    \end{itemize}
\end{defn}

 \begin{figure}[!t]
    \centering
      \includegraphics[width=\linewidth]{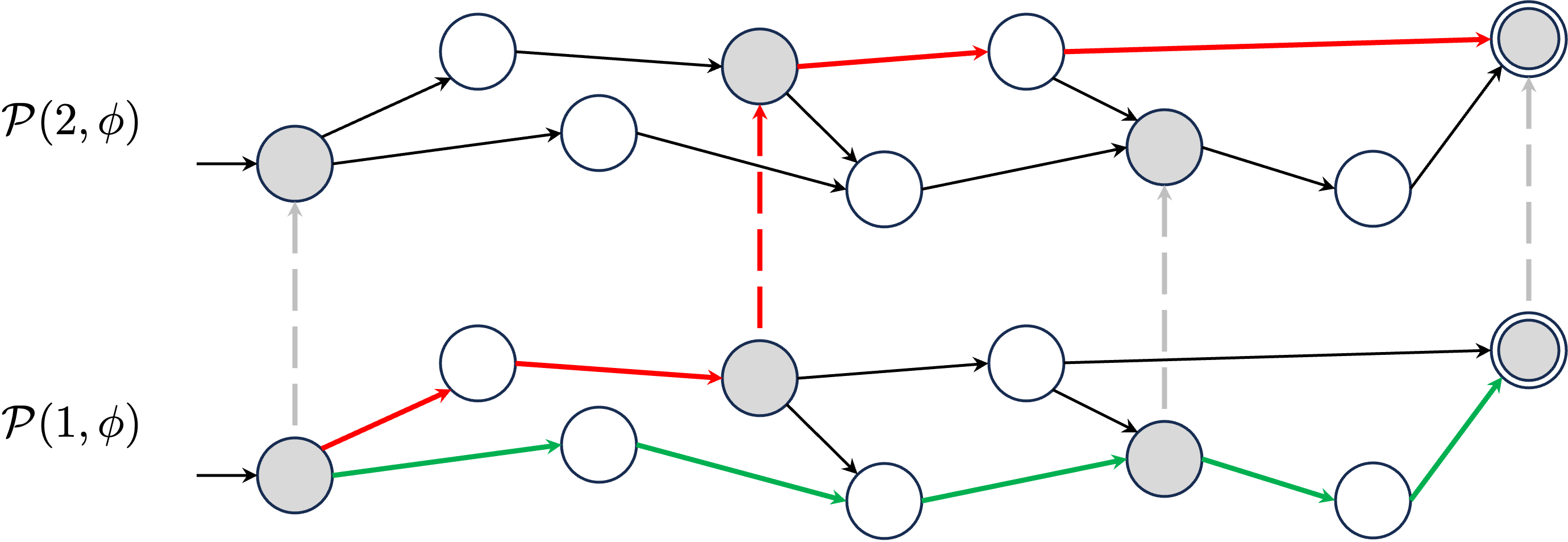}
    \caption{The product team model featuring two robots. PAs \(\ccalP(1, \phi)\) and \(\ccalP(2, \phi)\) may be different reflecting the robots' heterogeneity. Initial states are indicated with incoming arrows, and accepting states are depicted with double circles. The decomposition sets are highlighted in gray. In-spec switch transitions are represented by dashed arrows. A plan exclusive to robot 1 is traced by the green path, whereas the red path delineates a plan that involves the collaboration of both robots.}
    \label{fig:team}
 \end{figure}
 
The distinction between in-spec transitions and in-spec switch transitions lies in their impact on task progress: in-spec transitions advance the task through the same robot, while in-spec switch transitions involve changing to a different robot without altering the task's progress, as indicated by \(q= q'\). A visual representation of \(\ccalP(r, \phi)\) can be found in Fig.~\ref{fig:team}. The connections between product models are unidirectional, starting from the first robot and ending at the last robot. The specific arrangement of robots is irrelevant. Transitioning to the next product model occurs solely at in-spec switch transitions. Any path connecting an initial to an accepting state constitutes a feasible solution, comprising multiple path segments. Both the start and end state of each path segment have automaton states that belong to the decomposition set. Each segment results in a sequence of actions for an individual robot. These action sequences can be executed in parallel as they collectively decompose the task, as illustrated in Fig.~\ref{fig:team}.

\begin{rem}
Note that we construct team models exclusively for leaf specifications which directly involve atomic propositions. The switch transitions in~\cite{schillinger2018simultaneous} point to the initial state of the next robot, that is, \(s' = s_{r'}^0\), which is not the case for the in-spec switch transitions here. This difference arises because~\cite{schillinger2018simultaneous} addresses a single specification, where a robot typically starts a task from its initial state. In contrast, in scenarios involving multiple specifications, a robot can begin a new task from its last state at the end of the preceding task.
\end{rem}
 \begin{figure*}[!t]
    \centering
      \includegraphics[totalheight=0.3\textheight]{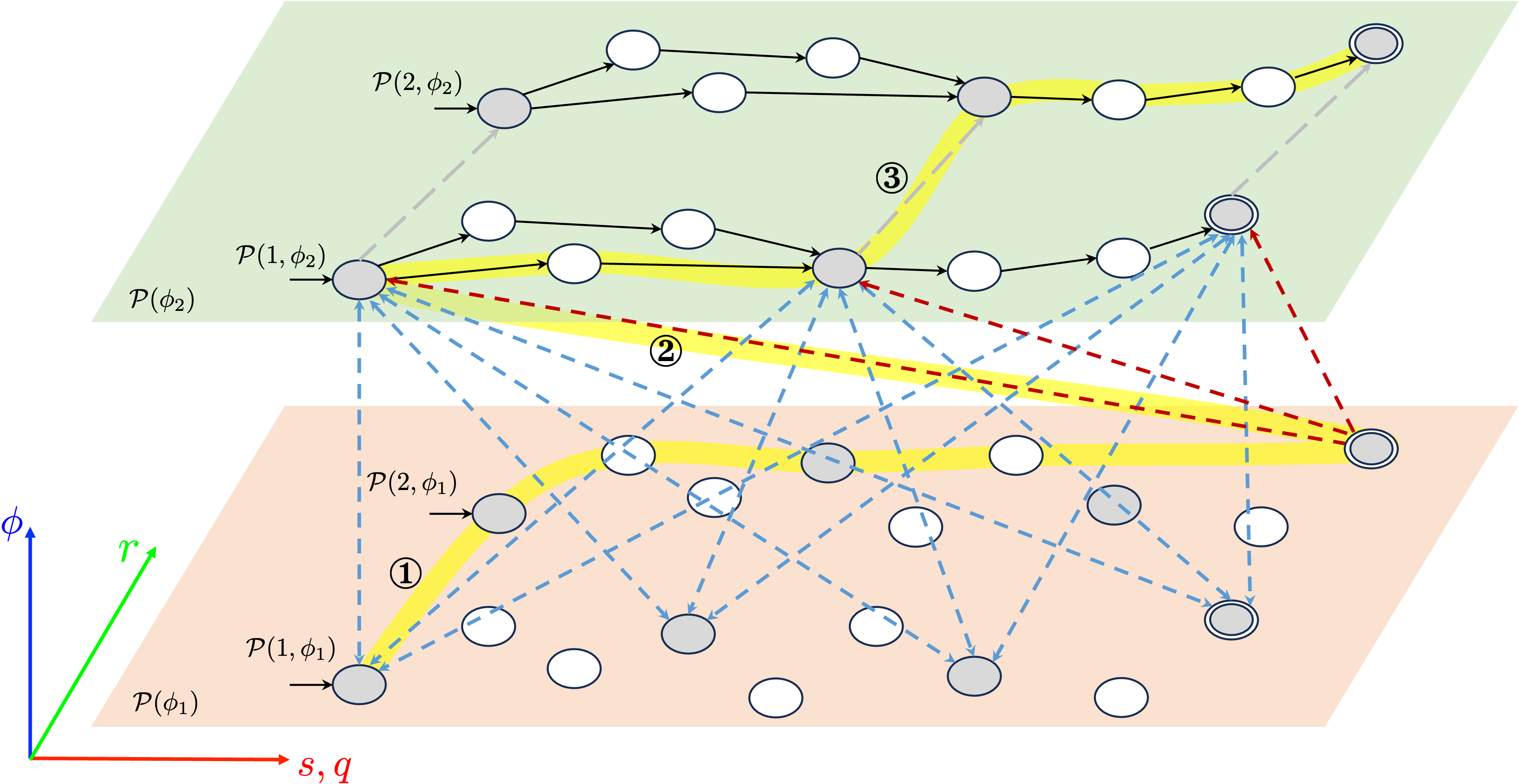}
    \caption{Hierarchical team models for two leaf specifications and two robots. The product team models for each leaf specification are depicted within areas shaded in light orange and green. To enhance clarity, some transitions within a single PA are omitted. Inter-spec type I and II switch transitions are illustrated using blue and red dashed arrows, respectively. It's important to note that type I switch transitions are bidirectional, while type II switch transitions are unidirectional. For the sake of readability, inter-spec type I switch transitions are specifically shown for \(\ccalP(1, \phi_1)\) and \(\ccalP(1, \phi_2)\). Inter-spec type II switch transitions are only shown from \(\ccalP(2, \phi_1)\) to \(\ccalP(1, \phi_2)\). The path marked in yellow highlights a specific plan. This plan originates from the initial state of \(\ccalP(1, \phi_1)\) and concludes at the accepting state of \(\ccalP(2, \phi_2)\). This path involves assigning robot 2 to fulfill the specification \(\phi_1\), which includes one in-spec switch transition~\ding{172}. Subsequently, robots 1 and 2 are tasked with satisfying specification \(\phi_2\), connected through a type II inter-spec switch transition~\ding{173}  and involving another in-spec switch transition~\ding{174}.}
    \label{fig:h-team}
    \vspace{-10pt}
 \end{figure*}

\subsection{Hierarchical Team Models}
To extend to \hltl, we construct a product team model per leaf specification. The goal here is to connect them to form a connectable space. We refer to the scenario of decomposition of tasks for one specifications as {\it in-spec independence}. In contrast, leaf specifications can exhibit dependencies. For example, one leaf specification \(\phi\) might need to be completed before another \(\phi'\). This scenario is referred to as {\it inter-spec dependence}. Another scenario is {\it inter-spec independence}, where two leaf specifications can be satisfied independently. We note that {\it in-spec dependence} does not exist, as precluded by Asm.~\ref{asmp:decomp}. To capture inter-spec dependency and independence, we introduce two additional types of switch transitions between different product team models.
\begin{defn}[Hierarchical Team Models]
   Given the \hltl\ specifications $\Phi$, the hierarchical team models $\ccalP$ consists of a set of product team models $\ccalP(\phi)$ with $\phi \in \Phi_{\text{leaf}}$ and is denoted by the tuple $\ccalP :=(\ccalQ_{\ccalP}, \ccalQ^0_\ccalP, \to_{\ccalP}, \ccalQ^F_{\ccalP})$: 
\begin{itemize}
\item $\ccalQ_{\ccalP} = \left\{(r, \phi, s, q) : \phi \in \Phi_\text{leaf}, (r, s, q) \in  \ccalQ_\ccalP(\phi)\right\}$ is the set of state;
\item $\ccalQ^0_\ccalP = \left\{(r, \phi, s, q) \in  \ccalQ_\ccalP: \phi \in \Phi_\text{leaf}, (r, s, q) \in  \ccalQ_\ccalP^0(\phi)\right\}$ is the set of initial states;
\item $\to_\ccalP \, \in \ccalQ_{\ccalP} \times \ccalQ_{\ccalP}$ is the transition relation. One type of transition occurs inside a product team model, that is, $ (r, \phi, s, q) \to_\ccalP (r', \phi, s', q')$ if $(r, s, q) \to_{\ccalP(\phi)} (r', s', q')$; another two types of transition relation, denoted by $\switch{inter}{1}$ and $\switch{inter}{2}$, occur between two product teams models and are defined below in Def.~\ref{defn:inter-spec1} and~\ref{defn:inter-spec2};
\item $\ccalQ^F_{\ccalP} = \{(r, \phi, s, q) \in \ccalQ_\ccalP: q \in \ccalQ^F_\ccalA(\phi)\}$ is the set of accepting/final states. 
\end{itemize}
\end{defn}

\begin{defn}[Inter-Spec Type I Switch Transition]~\label{defn:inter-spec1}
     A transition, denoted by $\switch{inter}{1} = \left( (r, \phi, s, q), ( r', \phi',  s', q') \right)$, is type I switch transition if and only if it: 
    \begin{itemize}
    \item connects the same robots, $r = r'$; 
    \item points to two different leaf specifications, $\phi \neq \phi'$;
    \item points to the same robot state, $s = s'$;
    \item connects decomposition states, $q \in \ccalD_\ccalA(\phi), q' \in \ccalD_\ccalA(\phi')$.
    \end{itemize}
\end{defn}

\begin{defn}[Inter-Spec Type II Switch Transition]~\label{defn:inter-spec2}
    A  transition, denoted by $\switch{inter}{2} = \left( (r, \phi, s, q) , ( r',  \phi',  s', q') \right)$, is type II switch transition if and only if it: 
    \begin{itemize}
    \item points to the first robot, $r' = 1$;
    \item points to two different leaf specifications, $\phi \neq \phi'$;
    \item connects an accepting state to a decomposition state, $q \in \ccalQ_\ccalA^F(\phi), q' \in \ccalD_\ccalA(\phi')$.
    \end{itemize}
\end{defn}

Inter-spec type I switch transitions enable switching to another PA with a different leaf specification, while retaining the same robot and pausing the current task. Conversely, inter-spec type II switch transitions create a connection from the accepting state of one PA to the decomposition state of another PA associated with the first robot, indicating the completion of one specification and the resumption of another. {\color{\modifycolor}It points toward the first robot, the starting robot of a product team model, because the in-spec switch transition is unidirectional, always pointing to the next robot.} These hierarchical team models can be graphically represented in a 3D space, with each product team model occupying a 2D plane; see Fig.~\ref{fig:h-team}. This structure has three axes: one representing the task progress for a specification (the \( (s, q) \)-axis), another indicating the task assignment among robots (the \( r \)-axis), and the last one tracking the  progress of different specifications (the \( \phi \)-axis).
\begin{rem}
The construction is not a product of all NFA and robot states. Instead, if each product team model is considered a distinct sub-space, the inter-spec switch transitions facilitate the connections between these sub-spaces. As a result, two consecutive product models are only loosely interconnected. The majority of states across different sub-spaces remain disconnected. Such a design significantly reduces the overall search space, enhancing the efficiency of the search.
\end{rem}

Given a state-specification sequence $\tau = \tau_0\tau_1\ldots\tau_h$ where $\tau_i = ((\state{s}{1}{i}, \state{\psi}{1}{i}), (\state{s}{2}{i}, \state{\psi}{2}{i}), \ldots, (\state{s}{N}{i}, \state{\psi}{N}{i}))$, the cost for robot $r$, such as energy consumption or completion time, is represented as $c_r =  \sum_{i=0}^{h-1}c_r(s^i_r, s^{i+1}_r)$, where $c_r(s^i_r, s^{i+1}_r)$ denotes the cost incurred transitioning between states for robot $r$.  The goal is to minimize the additive cost, expressed as
\begin{align}\label{eq:objective}
    J(\tau) = \sum_{r=1}^N  c_r 
\end{align}
Finally, the problem can be formulated as follow:
\begin{problem}
     Given the hierarchical team models $\ccalP$, derived from transition systems and the \hltl\ specifications $\Phi$ {\color{\modifycolor}that satisfy Asm.~\ref{asmp:semantics}}, find an optimal state-specification sequence $\tau^*$ that satisfies $\Phi$ and minimizes $J(\tau^*)$.
\end{problem}


\section{Planning under \hltl}\label{sec:search}
This section begins by introducing a search-based approach, which constructs hierarchical team models on the fly, ensuring both completeness and optimality; see Sec.~\ref{sec:thm}. Subsequently, we discuss several heuristics designed to expedite the search.
\subsection{On-the-fly Search}

The on-the-fly search algorithm is built upon the Dijkstra's algorithm, as detailed in Alg.~\ref{alg:search}. The search state is defined as \( v= (r, \phi, \bbs, \bbq) \), where \( \bbs \in \ccalS_1 \otimes \ldots \otimes \ccalS_N \) and \( \bbq = \ccalQ_\ccalA(\level{1}{K}) \otimes \ldots \otimes \ccalQ_\ccalA(\level{1}{1}) \). This state is comprised of three parts: (i) \( (r, \phi) \) indicates the specific PA currently being searched; (ii) \( \bbs \) represents the states of all robots; and  (iii) \( \bbq \) denotes the automaton states of all specifications (including leaf and non-leaf ones). Two states are considered {\it equivalent} if they share the same \( r \), \( \phi \), \( \bbs \), and \( \bbq \) values. The notation \( \bbs(r) \) and \( \bbq(\level{i}{k}) \) refer to the state of robot \( r \) and the automaton state of specification \( \level{i}{k} \) reached so far, respectively. We use \( \symtext{\ccalV}{explored} \) and \( \symtext{\ccalV}{seen} \) to denote the sets of states that have been explored and seen, and \( \symtext{\ccalH}{front} \) for the states awaiting exploration. States within \( \symtext{\ccalH}{front} \) are ordered by their costs. 
Lastly, \( \symtext{\Lambda}{explored} \) maps each explored state to the shortest path that led to it.

\begin{defn}[Initial State]\label{def:initial}
A state $v = (r, \phi, \bbs, \bbq)$ is an {\it initial} state if
\begin{itemize}
    \item $(r, \phi, \bbs(r), \bbq(\phi)) \in \ccalQ_\ccalP^0, \forall r \in [N]_+, \forall \phi \in \Phi_{\text{leaf}}$;
    \item $\bbq(\phi) \in \ccalQ_\ccalA^0(\phi), \forall \phi \in \Phi\setminus\Phi_\text{leaf}$.
\end{itemize}    
\end{defn}
That is, all robots begin at their initial states, and every specification starts from its initial automaton state [line~\ref{search:init}]. Whenever a state \( v \) is popped out from the set \( \symtext{\ccalH}{front} \) for exploration, it is added to the set of explored states if it hasn’t been already [lines~\ref{search:pop}-\ref{search:setexplore}]. The iteration process concludes when an accepting state of the topmost-level specification is reached. At this point, the function \texttt{ExtractPlan} is employed to parallelize the sequentially searched path, generating a state-specification sequence [line~\ref{search:accept}]. If the accepting state is not yet reached, the algorithm proceeds to determine the {\it successor} states of \( v \), along with their associated transition costs [line~\ref{search:get_succ}], through the function \texttt{GetSucc}, outlined in Alg.~\ref{alg:succ}. For each successor state that is either unseen or has a total cost lower than previously found, it is queued into \( \symtext{\ccalH}{front} \) for future exploration. Concurrently, the corresponding path leading to this state is recorded [lines~\ref{search:succ}-\ref{search:path}].

\begin{algorithm}[t]
\caption{Search-based STAP}
\LinesNumbered
\label{alg:search}
\KwIn {\hltl\ specifications $\Phi$, a set of robots $[N]_+$}
\KwOut {state-specification sequence}
$\ccalV_\text{seen} = \varnothing, \ccalV_\text{explored} = \varnothing, \ccalH_\text{front} = \varnothing, \Lambda_{\text{explored}} = \varnothing$ \label{search:varnothing}\; 
$\ccalV_{\text{init}} = \{v \,|\, (r, \phi, \bbs(r), \bbq(\phi)) \in \ccalQ_\ccalP^0, \forall r \in [N]_+, \forall \phi \in \Phi_{\text{leaf}} \wedge \bbq(\phi) \in \ccalQ_\ccalA^0(\phi), \forall \phi \in \Phi\setminus\Phi_{\text{leaf}} \}$ \label{search:init}\;
$\ccalH_\text{front}.\texttt{push}((\ccalV_{\text{init}}, 0)) $\;
$\symtext{\tau}{explored}(v) = v, \forall v \in \symtext{\ccalV}{init}$\;
Build specification hierarchy tree $\ccalG_h$\;
\While{$\ccalH_{\text{front}} \neq \varnothing$\label{search:while}}{
    $(v, c_v) = \ccalH_{\text{front}}.{\texttt{pop}}()$\label{search:pop}\;
    \If{$v \in \ccalV_\text{explored}$\label{search:explored}}{
        \textbf{continue}\;
    }
    $\ccalV_\text{explored}.\texttt{add}(v)$\label{search:setexplore}\;
    \If{$v.\bbq(\level{1}{1}) \in \ccalQ_\ccalA^F(\level{1}{1})$ \label{search:exit}}{
        \Return {$\texttt{ExtractPlan}(\Lambda_{\text{explored}}(v))$} \label{search:accept}\;
    }
    $\ccalV_{\text{succ}} = \texttt{GetSucc}(v, \ccalP, \ccalG_h)$ \label{search:get_succ}\;
    \For{$(u, c_{vu}) \in \ccalV_{\text{succ}}$ \label{search:succ}}{
        $c_u = c_v + c_{vu}$\;
        \If{$u \not \in \ccalV_\text{seen}$ or $c_u < \ccalV_\text{seen}(u) $ }{
             $\ccalV_\text{seen}(u) = c_u$\;
             $\ccalH_{\text{front}}.\texttt{push}((u, c_u))$\;
             $\Lambda_{\text{explored}}(u) = \Lambda_{\text{explored}}(v) + u$ \label{search:path}\;
        }
    }
}
\Return $\varnothing$\;
\end{algorithm}

\begin{algorithm}[t]
\caption{$\texttt{GetSucc}(v, \ccalP, \ccalG_h)$}
\LinesNumbered
\label{alg:succ}
\KwIn {state $v$, hierarchical team models $\ccalP$,  specification hierarchy tree $\ccalG_h$}
\KwOut {Set of successor states $\ccalV_{\text{succ}}$}
$\ccalV_{\text{succ}} = \varnothing$\;
$(r, \phi, s, q) = (v.r, v.\phi, v.\bbs(v.r), v.\bbq(v.\phi))$ \;

$\symtext{\Phi}{prt} = \texttt{Parents}(\ccalG_h, \phi) $ \label{succ:parents}\;
\If{$ q \not \in \ccalQ^F_\ccalA(\phi)$ and $ v.\bbq(\phi_{\text{prt}}) \not \in \ccalQ_\ccalA^F(\phi_{\text{prt}}), \forall \phi_{\text{prt}} \in \symtext{\Phi}{prt} $\label{succ:transition}}{
    \Comment*[r]{In-spec transitions in $\ccalP(r, \phi)$}
    \For{$((s, q), (s', q')) \in \to_{\ccalP(r, \phi)}$}{
        $u = v,  \ccalV_u = \varnothing$\;  
        $u.\bbs(r) = s', \; u.\bbq(\phi) = q'$ \label{succ:transition-rq}\;
        $\pi_c = \{\phi\}$ \textbf{if} {$q' \in \ccalQ^F_\ccalA(\phi)$} \textbf{else} {$\varnothing$} \label{succ:ap}\;
        $\texttt{UpdateNonLeafSpecs}(\pi_c, u, \symtext{\Phi}{prt}, \ccalV_u)$ \label{succ:nonleaf}\;
        $\ccalV_{\text{succ}}.\texttt{add}((\ccalV_u, c_r(s, s')))$\;
    }

    \Comment*[r]{In-spec switch transitions in $\ccalP(\phi)$}
    \If{$r \neq N$ and $((r, \phi, s, q), (r+1, \phi, v.\bbs(r+1), q)) \in \zeta_{\text{in}}$}{
        $u = v$\;  
        $u.r = r+1$ \label{succ:plus1}\;
        $\ccalV_{\text{succ}}.\texttt{add}((u, 0))$\;
    }
}
\Comment*[r]{Type I switch transitions in $\ccalP$}
\For{$\phi' \in \Phi_{\text{leaf}}\setminus\{\phi\}$ \label{succ:inter-spec1}}{
    \If{$((r, \phi, s, q), (r, \phi', s, v.\bbq(\phi')) \in \switch{inter}{1}$}{
        $u = v$\;
        $u.\phi = \phi'$ \label{succ:type1}\;
        $\ccalV_{\text{succ}}.\texttt{add}((u, 0))$\;
    }
}
\Comment*[r]{Type II switch transitions in $\ccalP$}
\For{$\phi' \in \Phi_\text{leaf}\setminus\{\phi\}$ \label{succ:inter-spec2}}{
    \If{$((r, \phi, s, q), (1, \phi', v.\bbs(1), v.\bbq(\phi')) \in \switch{inter}{2}$}{
        $u = v$\;
        $u.r = 1, \;u.\phi = \phi'$ \label{succ:type2}\;
        $\ccalV_{\text{succ}}.\texttt{add}((u, 0))$\;
    }
}
\Return $\symtext{\ccalV}{succ}$\;
\end{algorithm}

The function \texttt{GetSucc} determines the successor states of a given state \( v = (r, \phi, \bbs, \bbq) \) by considering in-spec transitions in \(\ccalP(r, \phi)\), in-spec switch transitions in \(\ccalP(v.\phi)\), and inter-spec type I and II switch transitions in \(\ccalP\). To identify in-spec transitions in \(\ccalP(r, \phi)\), the function initially checks if the specification \(\phi\) or any of its non-leaf parent specifications are satisfied. This is achieved using \texttt{Parents}(\(\ccalG_h\), \(\phi\)), which returns the path from the leaf specification \(\phi\) to the root specification in the specification hierarchy tree \(\ccalG_h\), excluding \(\phi\) itself [line~\ref{succ:parents}]. {\color{\modifycolor}If \(\phi\) or any of its parent specifications are satisfied, the search in \(\ccalP(r, \phi)\) and \(\ccalP(\phi)\) ceases, as further exploration is unnecessary. This is based on condition~\eqref{cond:union} in Definition~\ref{def:hltl} and~Assumption~\ref{asmp:semantics}, which state that each leaf specification appears only in one non-leaf specification and satisfying each specification at most once is adequate.} 
\begin{algorithm}[t]
\caption{$\texttt{UpdateNonLeafSpecs}(\pi_c, u, \symtext{\Phi}{prt}, \ccalV_u$)}
\LinesNumbered
\label{alg:nonleaf}
\KwIn {True composite propositions $\pi_c$, state $u$, parent specifications $\symtext{\Phi}{prt}$}
\KwOut {Updated set of states $\ccalV_u$}
$\phi_{\text{prt}} = \symtext{\Phi}{prt}.\texttt{pop}()$ \label{nonleaf:pop}\;
$\ccalQ_\ccalA^{\text{succ}} = \left\{  q' \,|\, (u.\bbq(\phi_{\text{prt}}), \pi_c, q') \in \to_\ccalA(\phi_{\text{prt}} ), \forall q' \in \ccalQ_\ccalA(\phi_{\text{prt}})  \right\}$\label{nonleaf:reach}\;
\For{$q' \in \ccalQ^{\text{succ}}_\ccalA$}{
    $u.\bbq(\phi_{\text{prt}}) = q'$ \label{nonleaf:update}\;
    \If{$\phi_{\text{prt}} = \level{1}{1}$}{
        $\ccalV_u.\texttt{add}(u)$ \label{nonleaf:topmost}\;
    }
    \Else{
        $\pi_c = \{\symtext{\phi}{prt}\}$ \textbf{if} {$q' \in \ccalQ^F_\ccalA(\phi_{\text{prt}})$} \textbf{else} {$\varnothing$}\;
        $\texttt{UpdateNonLeafSpecs}(\pi_c, u, \symtext{\Phi}{prt}, \ccalV_u)$\;
    }
}

\end{algorithm}

If none of the parent specifications are satisfied, the state of robot \( r \) in \( v.\bbs \) and the automaton state of leaf specification \(\phi\) in \( v.\bbq \) are updated based on in-spec transitions in \(\to_{\ccalP(r, \phi)}\) [line~\ref{succ:transition-rq}]. Subsequently, the automaton states of all parent specifications of \(\phi\) are also updated, as detailed in lines~\ref{succ:ap}-\ref{succ:nonleaf} and further elaborated in Alg.~\ref{alg:nonleaf}. For the three types of switch transitions, only the robot \( r \) (in the case of in-spec switch transitions [line~\ref{succ:plus1})], the leaf specification \(\phi\) (for inter-spec type I switch transitions [line~\ref{succ:type1}]), or both robot \( r \) and specification \(\phi\) (for inter-spec type II switch transitions [line \ref{succ:type2}]) are updated. This implies a shift in the search space to a different PA, without modifying \( v.\bbs \) and \( v.\bbq \). Since these transitions do not change robot states and automaton states, the function \texttt{UpdateNonLeafSpecs} is not invoked.

Given a state \( u \) and the set of parent specifications \( \symtext{\Phi}{prt} \) of the leaf specification \( u.\phi \) currently being explored, the function \texttt{UpdateNonLeafSpecs} utilizes a depth-first strategy to refresh the automaton states of non-leaf specifications within \( \symtext{\Phi}{prt} \). When considering a specific parent specification \( \symtext{\phi}{prt} \in \Phi_{\text{prt}} \), the function initially identifies all states that can be reached based on the acceptance of the preceding specification at the immediate lower level [lines~\ref{nonleaf:pop}-\ref{nonleaf:reach}]. The automaton state of \( \symtext{\phi}{prt} \) in \( u.\bbq \) is then updated using these reachable states [line~\ref{nonleaf:update}]. This process continues until the automaton state of the root specification is updated [line~\ref{nonleaf:topmost}].

\begin{algorithm}[t]
\caption{$\texttt{ExtractPlan}(\Lambda)$}
\LinesNumbered
\label{alg:exe}
\KwIn {Searched path $\Lambda$}
\KwOut{state-specification sequence $\tau$}
Remove switch states from $\Lambda$ and divide $\Lambda$ into $l$ path segments with each sharing the same specification\label{exe:divide}\;
\For{$k \in [l]_+$}{
    \For{$v \in \Lambda_k$}{
        $\tau_i = ((s_1^i, \psi_1^i), \ldots, (s_r^i, \psi_r^i), \ldots, (s_N^i, \psi_N^i))$ where $(s_r^i, \psi_r^i) = (v.\bbs(r), v.\phi)$ and $(s_{r'}^i, \psi_{r'}^i) = (v.\bbs(r'), \epsilon), \forall r' \not = v.r$ \label{exe:state-spec}\;
    }
}
\Return $\tau$\;
\end{algorithm}

  \begin{figure}[!t]
    \centering
      \includegraphics[width=\linewidth]{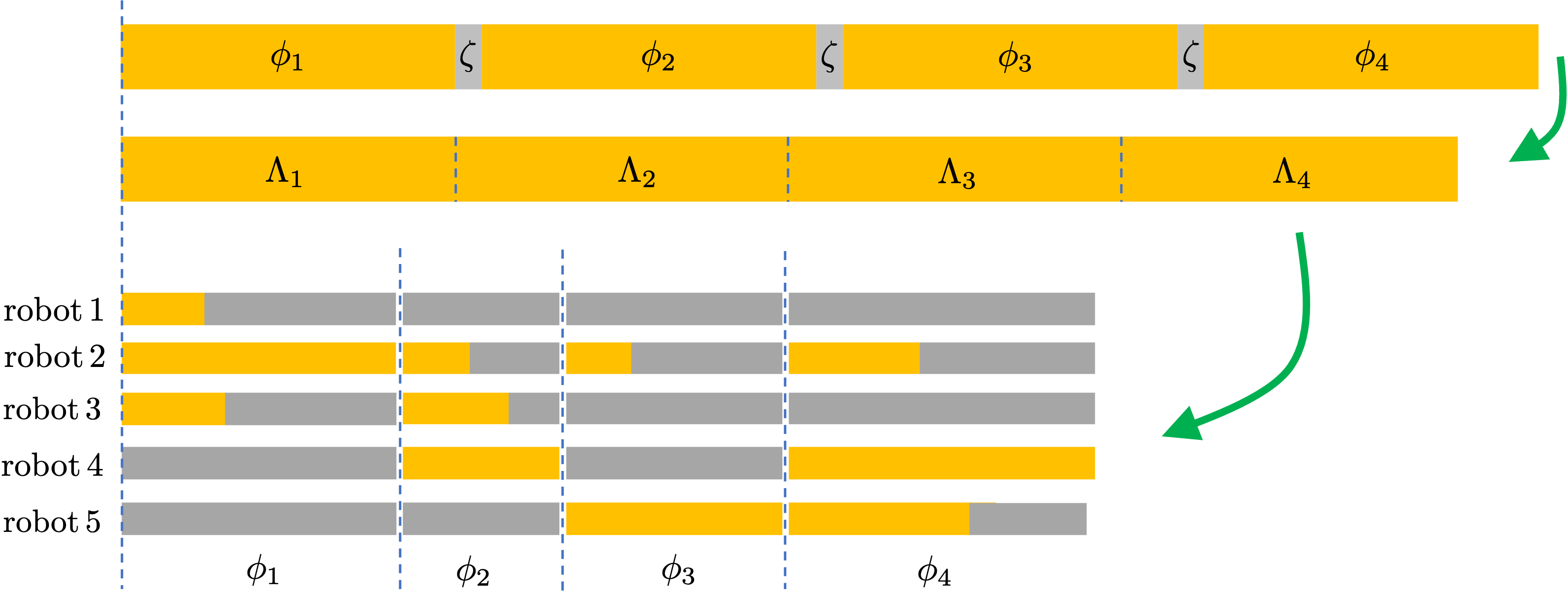}
    \caption{The process of plan extraction: The top bar represents the path, segmented by inter-spec switch transitions, with each segment corresponding to distinct specifications. The middle bar illustrates the path after removing switch states. At the bottom, the derived plan for a team of five robots is displayed. Here, orange segments indicate active periods for the robots, while gray segments represent their inactive phases.}
    \label{fig:plan_extraction}
    \vspace{-10pt}
 \end{figure}
 
The \texttt{ExtractPlan} function in Alg.~\ref{alg:exe} processes the path \(\Lambda = v_0 v_1\ldots v_n\), which leads to the state satisfying the root specification, and outputs a state-specification sequence \(\tau\). This process is visually represented in Fig.~\ref{fig:plan_extraction}. {\color{black}A state \( v \) within \(\Lambda\) is identified as a {\it switch state} if it is the end state of an inter-spec switch transition \(\symtext{\zeta}{inter}^1\cup \symtext{\zeta}{inter}^2\), which has the same robot state and automaton state as the corresponding start state.} The path \(\Lambda\) is segmented by these switch states, with each segment corresponding to states associated with the same specification, indicative of a search within a specific product team model. To create a concise path \(\Lambda'\), all switch states are removed, as they merely facilitate transitions between product team models and don't influence the plan's progress. This refined path is denoted as \(\Lambda' = \Lambda_1 \Lambda_2 \ldots \Lambda_l\), where \(\Lambda_k\) represents the \(k\)-th path segment, as shown in line~\ref{exe:divide}. The next step involves converting these sequential path segments into parallel executions. For each segment \(\Lambda_k\), we identify all {\it active} robots \(\ccalR_k\) participating in the specification \(\phi_k\), represented as \(\ccalR_k = \left\{v.r \,|\, v \in \Lambda_k \right\}\). Due to task decomposition~\cite{schillinger2018simultaneous}, each robot in \(\ccalR_k\) can execute its part of the plan independently. Robots not in $\ccalR_k$ are considered inactive for that segment. The plan horizon post-parallelization is determined by the longest path among all active robots. Note that an active robot might not maintain the active mode during the entire horizon if it completes its actions in a shorter time frame. For each robot state, we pair it with the relevant specification \(\phi_k\) if the robot is active, or with a null specification \(\epsilon\) if it's inactive [line~\ref{exe:state-spec}].

\begin{rem}
 The search state includes the joint robot states to keep track of the states of idle robots, allowing the search to resume from their previous states when necessary. In Fig.~\ref{fig:h-team}, after completing the search for robot 2 with respect to specification $\phi_1$, the search transitions to robot 1 while robot 2 remains idle. It then returns to robot 2 for specification $\phi_2$ via an in-spec switch transition~\ding{174}. At this stage, knowing the state of robot 2 is essential.
\end{rem}

\begin{rem}
Our method aligns with the decoupled approaches in the domain MAPF~\cite{van2005prioritized,wagner2011m}, where initially, individual agents independently plan their paths without considering potential collisions. Subsequently, the system assesses for collisions and modifies the paths to prevent conflicts. Resolving conflicts in temporal logic planning proves more difficult because, unlike MAPF, the actions of robots may be temporally correlated. When addressing collisions, it is crucial not only to consider robots that are physically close but also those that are temporally correlated. {\color{\modifycolor}Moreover, the modified path must be carefully adjusted to ensure that the temporal logic task remains satisfied.} Therefore, we leave the exploration of a  collision resolution strategy to future studies.
\end{rem}

\subsection{Search Acceleration with Heuristics}
We introduce three heuristics aimed at expediting the search.
\subsubsection{Temporal Order between Specifications}\label{sec:heuristic_order}
Consider the specification $\level{1}{1} = \Diamond (\level{1}{2} \wedge \Diamond \level{2}{2})$, where specification \(\level{2}{2}\) is required to be satisfied following~\(\level{1}{2}\), hence, a switch transition from the team model of \(\level{2}{2}\) to that of \(\level{1}{2}\) is unnecessary.  Given two leaf specifications $\phi_1, \phi_2 \in \Phi_\text{leaf}$, let $\phi_1 \bowtie \phi_2$, where $\bowtie\, \in \{\prec, \succ, \|\}$, denote the case where $\phi_1$ should be completed prior to ($\prec$), after ($\succ$) or independently ($\|$) of $\phi_2$. We base on~\cite{luo2022temporal} to determine these temporal relationships between leaf specifications. For each state $v$, we can determine the set of leaf specifications that have been fulfilled so far from component $v.\bbq$. When expanding a state $v$ in Alg.~\ref{alg:succ}, its inter-spec switch transitions only point to leaf specifications that (i) are not satisfied so far, and (ii) do not have any other unsatisfied leaf specifications except for $v.\phi$ that need to be completed beforehand.

\subsubsection{Essential Switch Transitions} \label{sec:heuristic_switch}
We introduce {\it essential switch transitions} to  refine the connections between PAs. 

\begin{defn}[Essential State]
An essential state for robot \( r \), denoted as \( (r, s) \), is defined either (a) $ s = s_r^0$ or (b) there exists a leaf specification \( \phi \) fulfilling the conditions:
\begin{itemize}
         \item $((\phi, r, s, q), (\phi, r, s, q')) \in \to_{\ccalP(r, \phi)}$;
        \item $q \neq q'$;
        \item $q'  \in \ccalD_\ccalA (\phi)$.
\end{itemize}
\end{defn}

This definition implies that an essential state is either (a) the initial state of robot \( r \) or (b) a state where robot \( r \) causes an in-spec transition to a decomposition state of a leaf specification {\it for the first time}. In the latter case, the intuition is that essential states are pivotal for achieving substantial progress in the task.
\begin{defn}[Essential Switch Transition]
{\color{black}A switch transition $\zeta = \left( (\phi, r, s, q) , (\phi', r', s', q') \right) \in \zeta_{\text{in}} \cup \switch{inter}{1} \cup \switch{inter}{2} $ is an essential switch transition if both $(r, s)$ and $(r', s')$ are essential states.}
\end{defn}

The leaf specifications \( \phi \) and \( \phi' \) might not be the exact ones that render the states \( (r, s) \) and \( (r', s') \) essential. In other words, \( \phi \) and \( \phi' \) could represent any leaf specifications, which is explained in Ex.~\ref{ex:essential} below. Consequently, any switch transitions that are not deemed essential will be excluded.


\begin{exmp}\label{ex:essential}
Consider a sequence of three transitions: 
\((\phi, r, s, q) \overset{\text{\ding{172}}}{\to} (\phi, r, s, q') \overset{\text{\ding{173}}}{\to} (\phi', r, s, q'') \overset{\text{\ding{174}}}{\to} (\phi', r+1, s', q'')\). 
Here, transition \ding{172} is an in-spec transition, transition \ding{173} is an inter-spec type I switch transition, and transition \ding{174} is an in-spec switch transition. Let's assume transition \ding{172} renders \((r, s)\) an essential state. If we strictly limit \((r, s)\) to be associated only with specification \(\phi\), then transition \ding{174} would not be classified as essential. Consequently, it would be excluded, potentially impacting the feasibility of the overall plan. \hfill $\square$
\end{exmp}
\subsubsection{Guidance by Automaton States}\label{sec:heuristic_automaton}
Unlike the first two heuristics, which focus on reducing the graph size, the last heuristic aims to speed up the search by prioritizing the expansion of more promising states — those that are closer to achieving satisfaction. For a state \( u = (r, \phi, \bbs, \bbq) \), we define its task-level cost as summation over all leaf specifications:
\begin{align}
\begin{aligned}
    c_q(\bbq) & = \sum_{\phi' \in \Phi_\text{leaf}} c_q(\bbq(\phi'))  \\
     =  \sum_{\phi' \in \Phi_{\text{leaf}}} & \max \left\{ \texttt{len} (q^0_\ccalA \to \bbq(\phi'))\, | \, \forall\, q^0_\ccalA \in \ccalQ_\ccalA^0 (\phi')  \right\}.
\end{aligned}
\end{align}
The cost \( c_q(\bbq(\phi')) \) is determined by the length of the simplest path from any initial state \( q^0_\ccalA \) to the automaton state \( \bbq(\phi') \) within the NFA \( \ccalA(\phi)\). A simple path is defined as one that does not include any repeated states, and the function \(\texttt{len}\) measures the path's length in terms of the number of transitions. Essentially, \( c_q(\bbq(\phi)) \) captures the maximal progress made in the task \( \phi \) since its beginning. {\color{black}A similar heuristic has been employed in~\cite{kantaros2020stylus,luo2021abstraction}.} The cost \( c_{vu} \) is now calculated as
\begin{align}\label{eq:new_cost}
    c_{vu} =  c_r\left(v.\bbs(r), u.\bbs(r)\right) + w \left[c_q(u.\bbq) - c_q(v.\bbq)\right],
\end{align}
where $c_r$ is the cost between robot states, $c_q(u.\bbq) - c_q(v.\bbq)$ is the cost between automaton states and $w$ is the weight.

\begin{rem}
The first two heuristics essentially alter the structure of the hierarchical team models. As a result, the completeness and optimality of the search approach, when these heuristics are applied, are no longer assured. However, employing exclusively the third heuristic still maintains the completeness of the search, as it ensures the full exploration of the search space eventually. 
Extensive simulations presented in Sec.~\ref{sec:sim} provide empirical evidence that our approach, incorporating all three heuristics, can generate solutions for complex tasks. 
The rationale behind this empirical completeness is twofold. Firstly,~\cite{luo2022temporal} prove that the inference on temporal order is  correct for a broad range of standard LTL specifications. By eliminating switch transitions that conflict with precedence orders, transitions that are in accordance with temporal orders remain unaffected. Secondly, the concept of essential switch transitions ensures that the search process only transitions to a different product model when there has been task progress in the current product model, which is reasonable as it is unnecessary to switch to a different task if the ongoing task has not been completed.
\end{rem}



        
\section{Theoretic Analysis}\label{sec:thm}
We conduct a theoretical analysis of the complexity, soundness, completeness, and optimality of our approach.

\begin{prop}
For every search state \( v = (r, \phi, \bbs, \bbq) \), among all the automaton states in \(\bbq\) related to leaf specifications, at most one automaton state does not belong to the decomposition set.
\end{prop}
\begin{proof}
Based on Def.~\ref{def:initial}, all automaton states within an initial search state are initial automaton states. By default, these are recognized as decomposition states. When a PA \(\ccalP(r, \phi)\) is chosen for exploration, only the corresponding automaton state \(\bbq(\phi)\) is updated, leaving the automaton states from other leaf specifications unaffected. The search can exit \(\ccalP(r, \phi)\) only through three types of switch transitions, each, by definition, associated with decomposition states. Upon switching to a different PA \(\ccalP(r', \phi')\), the automaton state of \(\phi\), which is a decomposition state, remains unchanged. Therefore, by inductive reasoning, it follows that at any given point, at most one automaton state from the leaf specifications does not belong to the decomposition set.
\end{proof}

Note that while the set of automaton states \(\bbq\) encompasses states from non-leaf specifications, these are not considered into our analysis of planning complexity, as they entirely depend on the leaf specifications.

\begin{theorem}[Complexity] 
The total number of transitions in the hierarchical team models \(\ccalP\) is given as follows:
\begin{align}\label{eq:complexity}
 |\to_\ccalP| = \sum_{\phi \in \Phi_\text{leaf}} \sum_{r\in [N]_+} \left|\to_{\ccalP(r, \phi)}\right| + \underbrace{|\symtext{\zeta}{in}| + |\symtext{\zeta}{inter}^1| + |\symtext{\zeta}{inter}^2|}_{|\zeta|}, 
\end{align}
which is the total number of in-spec transitions and switch transitions. According to~\cite{schillinger2018simultaneous}, $|\to_{\ccalP(r, \phi)}|$ is bounded by $|\to_{\ccalP(r, \phi)}| \leq \left(|\ccalQ_\ccalA(\phi)| \cdot |\ccalS_r| \right)\left(|\ccalQ_\ccalA(\phi)| \cdot |\ccalS_r| \right).$ Furthermore, $|\symtext{\zeta}{in}|$, $|\symtext{\zeta}{inter}^1|$ and $|\symtext{\zeta}{inter}^2|$ are, respectively, upper bounded by
\begin{align}\label{eq:complexity}
    |\symtext{\zeta}{in}| & \leq \sum_{\phi \in \Phi_\text{leaf}} \sum_{r\in [N-1]_+} \left(|\ccalD_\ccalA(\phi)| \cdot |\ccalS_r|\right)  \left(|\ccalD_\ccalA(\phi)|\cdot  |\ccalS_{r+1}|\right), \nonumber \\
    |\symtext{\zeta}{inter}^1| & \leq \sum_{\phi, \phi' \in \Phi_\text{leaf}} \sum_{r\in [N]_+} \left(|\ccalD_\ccalA(\phi)|\cdot |\ccalS_r| \right)  \left(|\ccalD_\ccalA(\phi')| \cdot |\ccalS_r|\right), \\
    |\symtext{\zeta}{inter}^2| &\leq \sum_{\phi, \phi' \in \Phi_\text{leaf}} \sum_{r\in [N]_+} \left(|\ccalQ^F_\ccalA(\phi)| \cdot |\ccalS_r|\right)  \left(|\ccalD_\ccalA(\phi')| \cdot  |\ccalS_1|\right). \nonumber
\end{align}
\end{theorem}
\begin{proof}
For a given PA \(\ccalP(r, \phi)\), the number of states is upper bounded by \(|\ccalQ_\ccalA(\phi)| \cdot |\ccalS_r|\). Its squared term gives the upper bound for the number of in-spec transitions \(|\to_{\ccalP(r, \phi)}|\), since any pair of states within the PA can potentially be connected. Regarding in-spec switch transitions between two consecutive PAs, \(\ccalP(r, \phi)\) and \(\ccalP(r+1, \phi)\), each decomposition state can be associated with any robot state, which sets the upper bound on the number of possible starting states to \(|\ccalD_\ccalA(\phi)| \times |\ccalS_r|\). Similarly, the maximum number of end states is determined by \(|\ccalD_\ccalA(\phi)| \times |\ccalS_{r+1}|\), subsequently defining the upper bound for \(|\symtext{\zeta}{in}|\). Following the same logic, the upper bound for inter-spec switch transitions can be deduced as outlined in Eq.~\eqref{eq:complexity}.
\end{proof}


\begin{cor}[Complexity of Search with Heuristics] 
    The upper bound on the number of in-spec transitions $|\symtext{\zeta}{in}|$ is the same as that in Eq.~\eqref{eq:complexity}, but the upper bound on the number of switch transitions $|\symtext{\zeta}{inter}|$ is obtained by replacing $\ccalS_r$ in Eq.~\eqref{eq:complexity} with $\Tilde{\ccalS}_r$, where $\Tilde{\ccalS}_r$ is the set of essential states of robot $r$.  
\end{cor}

\begin{theorem}[Soundness]
    The state-specification sequence given by Alg.~\ref{alg:search} satisfies \hltl\ specifications $\Phi$.
\end{theorem}
\begin{proof}
{\color{\modifycolor}The soundness of Alg.~\ref{alg:search} follows directly from its correct-by-construction design. The correctness of the state-specification sequence aligns with the correctness of the computed path $\Lambda = v_0, v_1, \ldots, v_n$. Each state $v \in \Lambda$ maintains its specification-related information through its component $v.\bbq$, which tracks the automaton states for all specifications. Consequently, for each specification $\phi \in \Phi$, we can derive an output word $\sigma(\phi) = \sigma_0, \sigma_1, \ldots, \sigma_n$, where $\sigma_i = \{\phi\}$ if $v_i.\bbq(\phi) \in \ccalQ_\ccalA^F(\phi)$ and $\sigma_i = \varnothing$ otherwise. Output words at lower levels propagate upward, generating those at the next higher level according to the bottom-up updates of non-leaf specifications in Alg.~\ref{alg:nonleaf}. Ultimately, the output word for the root specification $\level{1}{1}$ includes $\level{1}{1}$ itself, which meets the termination condition of Alg.~\ref{alg:search} at line~\ref{search:exit}. As a result, the approach conforms to the semantic definition in Def.~\ref{def:semantics}.}
\end{proof}

\begin{asmp}[Segmented Plan]\label{asmp:exist}
Assuming that there exists a feasible state-specification sequence \(\tau = \tau_0 \tau_1 \ldots \tau_h\) that can be sequentially divided into $\ell$ segments with \(\ell \leq |\Phi_\text{leaf}|\), represented as \(\tau = \langle\tau\rangle_1 \langle\tau\rangle_2 \ldots \langle\tau\rangle_{\ell}\) where no two segments share the same specification. Each segment \(\langle\tau\rangle_l\) fulfills a single leaf specification. That is, for every aggregated state-specification pair {\color{black}\(\tau_i = ((\state{s}{1}{i}, \state{\psi}{1}{i}), (\state{s}{2}{i}, \state{\psi}{2}{i}), \ldots, (\state{s}{N}{i}, \state{\psi}{N}{i}))\)} in the segment \(\langle\tau\rangle_l = \tau_{m} \tau_{m+1} \ldots \tau_{n}\), each {\color{black}\(\psi^i_r\)} is either the same leaf specification \(\phi\) that \(\langle\tau\rangle_l\) satisfies, or it is the null specification \(\epsilon\).
\end{asmp}

Note that the number of segments $\ell$ might be less than the number of leaf specifications $|\Phi_\text{leaf}|$, since it is possible that not all leaf specifications are required to be satisfied, such as leaf specifications that are connected by disjunction. The following properties only apply to Alg.~\ref{alg:search} without heuristics.
\begin{theorem}[Completeness] 
    Assuming that a state-specification sequence $\tau$ satisfying Asm.~\ref{asmp:exist} exists, Alg.~\ref{alg:search} is complete in that it is guaranteed to return a feasible plan that fulfills the \hltl\ specifications $\Phi$.
\end{theorem}
\begin{proof}
We proved the completeness of Alg.~\ref{alg:search} by first disregarding inter-spec type I switch transitions. The inclusion of type I transitions merely enlarges the search space, without impacting the algorithm's completeness. In this case, only inter-spec type II switch transitions exist between team models of any two leaf specifications \(\phi\) and \(\phi'\). Type II transitions connect accepting states of every PA within the product team model of \(\phi\) (or \(\phi'\)) to decomposition states of the first PA of the team model of \(\phi'\) (or \(\phi\)). As a result, the search does not advance to product team models of other leaf specifications until it discovers a path segment that satisfies the current one, e.g., $\phi$. Furthermore, once an accepting state of a PA for \(\phi\) is reached, the search only moves to initial states, not other decomposition states, of the first PA in another team model, e.g., \(\phi'\), as the search on this first PA has not started yet.

Following this logic, Alg.~\ref{alg:search} initiates from an initial state of a certain specification \(\phi\). Upon reaching one of its accepting states, the search transitions to the initial state of another specification \(\phi'\). This process is repeated until the termination condition is satisfied. Given that we employ a Dijkstra-based approach and the search within a product team model is complete as per~\cite{schillinger2018simultaneous}, Alg.~\ref{alg:search} is complete even without considering inter-spec type I switch transitions, concluding the proof.
\end{proof}
\begin{theorem}[Optimality]
Assuming the existence of a state-specification sequence \(\tau^*\) that meets Asm.~\ref{asmp:exist} with least cost, Alg.~\ref{alg:search} is optimal in that the cost of the plan produced by Alg.~\ref{alg:search} is equal to or less than the cost of \(\tau^*\).
\end{theorem}
\begin{proof}
To prove the optimality, we first set aside inter-spec type I switch transitions. Following the same logic as in the proof of completeness, and employing a Dijkstra-based search algorithm, we can assert that searching within a product team model is optimal, as established in~\cite{schillinger2018simultaneous}. The addition of inter-spec type I switch transitions expands the search space, allowing for the potential discovery of a feasible plan that might not adhere to Asm.~\ref{asmp:exist} yet could yield a lower cost.
\end{proof}

\section{Simulation Experiments}\label{sec:sim}
We use Python 3.10.12 on a computer with 3.5 GHz Apple M2 Pro and 16G RAM. Our code can be available via the \href{https://github.com/XushengLuo92/Hierarchical-LTL-STAP}{https://github.com/XushengLuo92/Hierarchical-LTL-STAP}. 
The video is accessible via the link~\href{https://www.youtube.com/watch?v=N3f8pUHDPF4\&t=4s}{https://www.youtube.com/watch?v=N3f8pUHDPF4\&t=4s}.

\subsection{User Study on Understanding \hltl}\label{sec:user_study}
We conduct a user study to assess if \hltl\ makes task specifications more straightforward.\footnote{The survey did not gather any personal information from the participants.} We use a robotic arm to sort blocks into designated areas, as depicted in Fig.~\ref{fig:user-study}, where blocks at source locations are color-filled, while their corresponding target locations are of the same shape but unfilled. Here, $s$ and $t$ represent the source and target locations for the blocks, with the subscripts $\text{cc}, \text{dm}, \text{st}$, and $\text{ht}$ indicating the shapes circle, diamond, star, and heart, respectively.

\begin{figure}[!t]
    \centering
    \includegraphics[width=\linewidth]{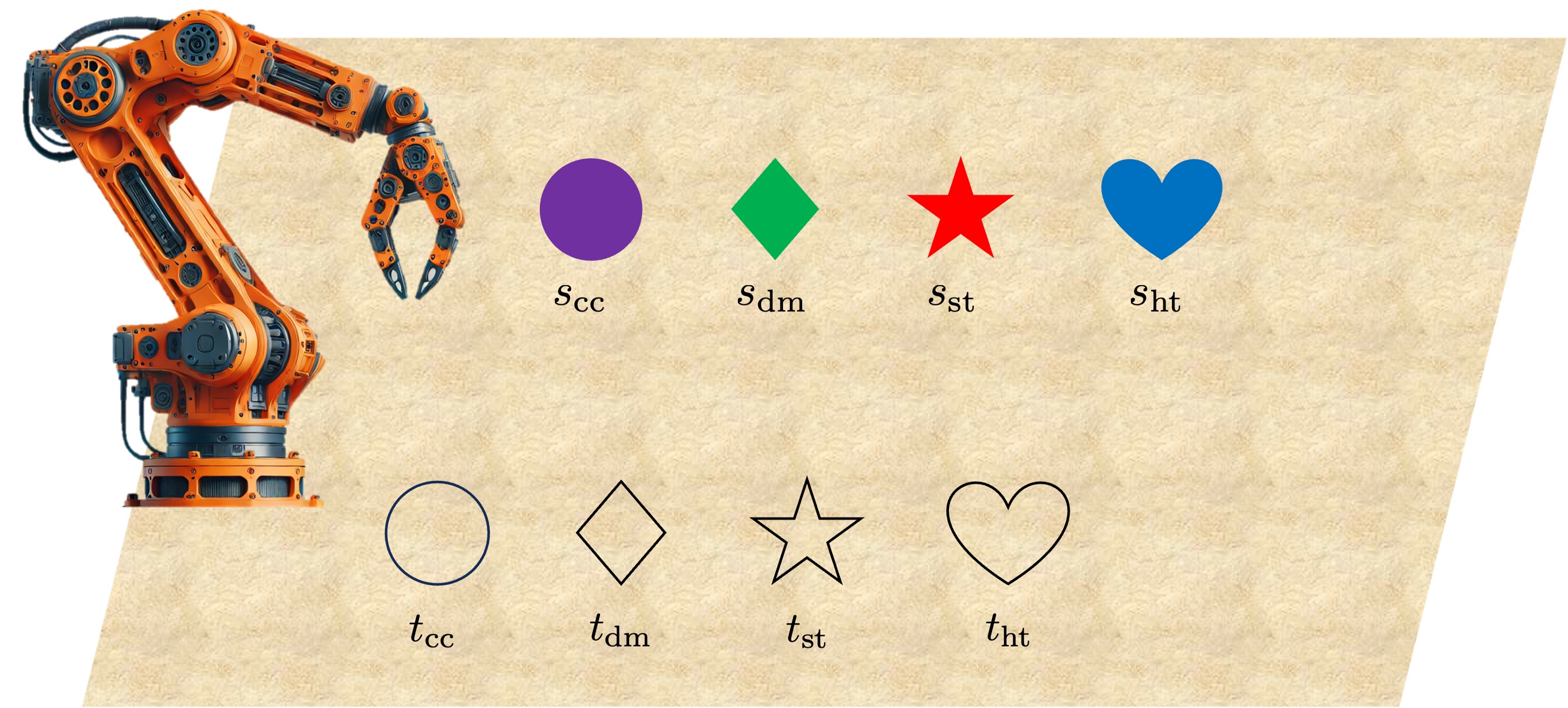}
    \caption{Pick-and-place platform for user study.}
    \label{fig:user-study}
    \vspace{-10pt}
\end{figure}

\subsubsection{Survey} We carry out the survey on Prolific (\href{www.prolific.com}{www.prolific.com}), recruiting 25 anonymous participants. The screening criteria specify that participants must have a STEM background and be at least undergraduate students, ensuring familiarity with mathematical logic. We begin the survey with an introduction to linear temporal logic, using the sorting task as a demonstration. Following this, participants answer two types of questions: subjective and objective. The subjective section consists of 7 questions where participants are presented with natural language instructions for 7 tasks. For each task, we provide a standard \ltl\ formula and hierarchical formulas. The participants are asked to choose which formula is easier to understand and clearer in representing the task described, of which task 5 is, ``First place the circle, and then place diamond and star in any order, and lastly place the heart.'' The standard \ltl\ is 
            \begin{align*}
            \begin{aligned}
               & \Diamond (s_{\text{cc}} \wedge \Diamond t_{\text{cc}}) \wedge \Diamond (s_{\text{dm}} \wedge \Diamond t_{\text{dm}})  \wedge \Diamond (s_{\text{st}} \wedge \Diamond t_{\text{st}}) \wedge \Diamond (s_{\text{ht}} \wedge \Diamond t_{\text{ht}}) \\ 
               &\wedge \neg s_{\text{dm}} \mathcal{U} t_{\text{cc}} 
               \wedge \neg s_{\text{dm}}  \mathcal{U} t_{\text{cc}}
               \wedge \neg s_{\text{st}}  \mathcal{U} t_{\text{cc}} 
               \wedge \neg s_{\text{ht}}  \mathcal{U} t_{\text{dm}}
               \wedge \neg s_{\text{ht}}  \mathcal{U} t_{\text{st}}
               \end{aligned}
            \end{align*}
        and the \hltl\ specifications are
 \begin{align}\label{eq:survey_hltl}
 \begin{aligned}
       L_1: \quad &  \level{1}{1}= \Diamond  (\level{1}{2} \wedge \Diamond  (\level{2}{2} \wedge  \Diamond  \level{3}{2} ))    \\
       L_2: \quad &  \level{1}{2} = \Diamond (s_{\text{cc}} \wedge \Diamond t_{\text{cc}})\\  
                & \level{2}{2} = \Diamond (s_{\text{dm}} \wedge \Diamond t_{\text{dm}}) \wedge  \Diamond (s_{\text{st}} \wedge \Diamond t_{\text{st}}) \\  
               & \level{3}{2} = \Diamond (s_{\text{ht}} \wedge \Diamond t_{\text{ht}}) 
       \end{aligned}
    \end{align}
{\color{\modifycolor}Note that standard \ltl\ and \hltl\ formulas are not equivalent, as achieving such equivalence is nontrivial. However, both formalisms are capable of capturing task descriptions to a large extent.} For the objective section, we prepare 4 pairs of standard and hierarchical formulas, each pair corresponding to the same task.
This results in a total of 8 questions, with each one featuring either a standard \ltl\ formula or \hltl\ formulas.
We randomly order the formulas so that adjacent questions do not necessarily involve the same task. Participants are required to select from four natural language instructions the one that correspond to the given formula.  
\begin{figure}
    \centering
    \includegraphics[width=\linewidth]{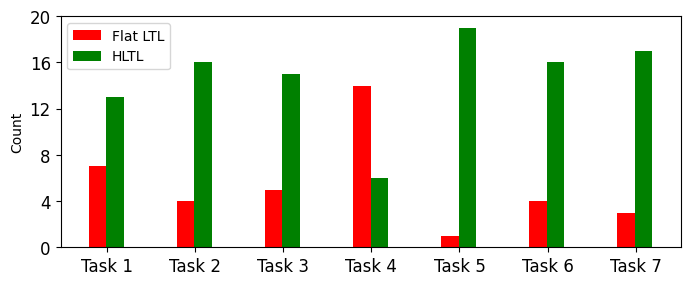}
    \caption{The count of preferences per task.}
    \label{fig:subjective}
    \vspace{-10pt}
\end{figure}
\subsubsection{Results analysis}Responses with low accuracy on objective questions were excluded, indicating that respondents lack a solid understanding of temporal logics. This results in 20 valid surveys. For the subjective questions, we compute the percentage of participants who prefer standard versus hierarchical formulas across 140 instances (20 participants, 7 questions each). For the objective questions, we calculate the accuracy for each type of \ltl\ across 80 instances per \ltl\ type (20 participants, 4 questions per \ltl\ type). 
We track the time taken by each participant to respond to each question. 

For the subjective questions, 27.1\% of all responses show a preference for standard \ltl, while 72.9\% favor \hltl. A detailed count of preferences is illustrated in Fig.~\ref{fig:subjective}, where respondents showed a clear preference over \hltl\ except for task 4. The specification is that ``Place the circle, diamond, star and heart. Diamond should be after circle and heart should be after star.'' The temporal constraints are simpler compared to those in task 5 in~\eqref{eq:survey_hltl}, making standard \ltl\ adequate for this scenario.  Regarding the objective questions, the accuracy rates are 65.0\% for standard \ltl\ and 86.3\% for \hltl. Additionally, the average response time for questions involving standard \ltl\ formulas is 78.3 seconds, compared to 47.1 seconds for \hltl\ formulas. These findings further validate the user-friendliness of \hltl, as evidenced by both higher accuracy and reduced response times. Details are presented in Fig.~\ref{fig:objective}, where respondents achieved higher accuracy rates in all tasks specified under \hltl.

\begin{figure}
    \centering
    \includegraphics[width=\linewidth]{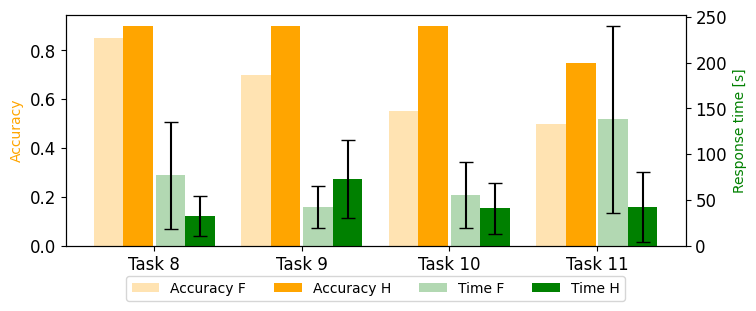}
    \caption{The accuracy and response time per task.}
    \label{fig:objective}
\end{figure}

\subsection{Scenarios}

We use the office environment described in~\cite{schillinger2018simultaneous} for service tasks, as depicted in Fig.~\ref{fig:bosch-building}. This 30$\times7$ grid map features 14 desk areas, 6 meeting rooms, and several other functional rooms. Additionally, we distribute $m$ robots at various locations. 
The first three scenarios are derived from the work~\cite{schillinger2018simultaneous}. We outline the hierarchical specifications for these scenarios below, while the corresponding standard versions of these specifications are referred to~\cite{schillinger2018simultaneous}. Note that the hierarchical specifications were derived directly from task descriptions, rather than from standard \ltl\ specifications.
\subsubsection{Scenario 1} {\color{\modifycolor}Transport the paper bin from desk $d_5$ to area $g$ for emptying, avoiding the public area while it is full. Return an empty bin from $g$ to desk $d_5$. The atomic propositions are: $\mathsf{default}$: the robot is not carrying any object, $\mathsf{carrybin}$: the robot is carrying a full paper bin, $\mathsf{dispose}$: the robot is disposing of garbage, $\mathsf{emptybin}$: the robot is carrying an empty paper bin, and $\mathsf{public}$: the robot is located in a public area.}
\begin{sizeddisplay}{\small}
{\allowdisplaybreaks\begin{align*}
 L_1: \quad &  \level{1}{1} =  \Diamond \level{1}{2} \wedge \Diamond\level{2}{2}\\
 L_2: \quad &  \level{1}{2} = \Diamond (d_5 \wedge \mathsf{default} \wedge \bigcirc ((\mathsf{carrybin}\; \mathcal{U}\; \mathsf{dispose}) \wedge \Diamond \mathsf{default}))  \\ 
 & \quad \quad \quad \quad \wedge \square (\mathsf{carrybin \Rightarrow \neg \mathsf{public}})\\
                & \level{2}{2} = \Diamond (g \wedge \bigcirc (g \wedge \mathsf{emptybin}) \wedge \Diamond (d_5 \wedge \bigcirc (d_5 \wedge \mathsf{default}))) 
\end{align*}}
\end{sizeddisplay}
\begin{figure}[!t]
    \centering
    \includegraphics[width=\linewidth]{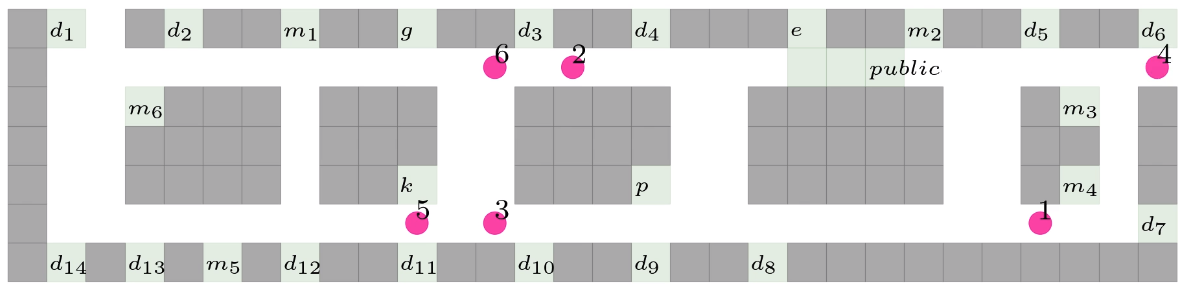}
    \caption{The office building in a grid-based layout, where areas \(d_1\) to \(d_{14}\) represent desks, \(m_1\) to \(m_6\) are meeting rooms, \(e\) stands for the elevator, \(g\) for the garbage room, \(p\) for the printer room, and \(k\) for the coffee kitchen. Areas marked as ``public'' indicate public spaces. Obstacles are illustrated in gray. The locations of robots are shown as numbered red dots.}
    \label{fig:bosch-building}
    \vspace{-10pt}
\end{figure}
\vspace{-1em}
\subsubsection{Scenario 2} Distribute documents to desks \(d_{10}\), \(d_7\), and \(d_5\), and avoid public areas while carrying the document. {\color{\modifycolor}Let $\mathsf{carry}$ denote the action of the robot carrying the document.}
\vspace{-1em}
\begin{sizeddisplay}{\small}
{ \allowdisplaybreaks
\begin{align*}
L_1: \quad &  \level{1}{1} =  \Diamond\level{1}{2} \wedge \Diamond\level{2}{2} \wedge \Diamond\level{3}{2}\\
 L_2: \quad &  \level{1}{2} = \Diamond (p \wedge \mathsf{carry}\; \mathcal{U}\; (d_{10} \wedge \bigcirc \neg \mathsf{carry}))  \wedge \mathsf{notpublic}  \\
 &  \level{2}{2} = \Diamond (p \wedge \mathsf{carry}\; \mathcal{U}\; (d_{7} \wedge \bigcirc \neg \mathsf{carry}))  \wedge \mathsf{notpublic}  \\
 &  \level{3}{2} = \Diamond (p \wedge \mathsf{carry}\; \mathcal{U}\; (d_{5} \wedge \bigcirc \neg \mathsf{carry}))  \wedge \mathsf{notpublic}  \\
 \quad & \mathsf{notpublic} :=   \square (\mathsf{carry} \Rightarrow \neg \mathsf{public}) 
\end{align*}}
\vspace{-10pt}
\end{sizeddisplay}

\begin{table*}[!t]
\renewcommand{\arraystretch}{1.2}
\centering\footnotesize
\setlength{\tabcolsep}{4.5pt}
\begin{tabular}{ccccccccccc}
\bhline
scenario &  $t$ &  $t_{\text{heur}}$ &  $t_{\text{heur}}^1$ &  $t_{\text{heur}}^2$ & $t_{\text{heur}}^3$ & $c$ &  $c_{\text{heur}}$ &  $c_{\text{heur}}^1$ &  $c_{\text{heur}}^2$ &  $c_{\text{heur}}^3$\\
\bhline
1 & 242.0$\pm$3.1 & \cellcolor{gray!20}{\bf 2.5$\pm$0.8 }& 242.0$\pm$3.1 & 11.1$\pm$0.2 & 41.7$\pm$33.5 & \cellcolor{gray!20}{\bf 63.7$\pm$6.9} & 77.4$\pm$1.2 & 63.7$\pm$6.9 & 66.8$\pm$8.7 & 80.3$\pm$8.2 \\
1$'$ & 232.9$\pm$18.0 &\cellcolor{gray!20}{\bf 2.3$\pm$0.4} & 93.7$\pm$14.6 & 20.8$\pm$6.7 & 27.9$\pm$20.4 & \cellcolor{gray!20}{\bf 81.4$\pm$8.8}& 82.8$\pm$6.7 & 82.1$\pm$6.9 & 82.7$\pm$9.2 & 82.4$\pm$9.7\\
2 & 1036.7$\pm$13.6  & \cellcolor{gray!20}{\bf 4.1$\pm$0.9} & 1036.7$\pm$13.6 & 46.0$\pm$1.8 & 21.1$\pm$13.8 & \cellcolor{gray!20}{\bf 91.5$\pm$6.7} & 98.5$\pm$5.6 &91.5$\pm$6.7 & 94.7$\pm$ 8.6 & 95.4$\pm$7.2 \\
2$'$ & 859.6$\pm$15.4&\cellcolor{gray!20}{\bf 2.1$\pm$0.2} & 283.2$\pm$2.4 & 39.8$\pm$0.8 & 289.7$\pm$211.1 & \cellcolor{gray!20}{\bf 93.3$\pm$3.2} & 94.4$\pm$2.2 & 96.3$\pm$5.9 & 95.0$\pm$5.9 & 93.6$\pm$5.4\\
3 & timeout &\cellcolor{gray!20}{\bf 10.2$\pm$3.4} & timeout & 594.6$\pm$23.0 &69.6$\pm$40.1 & --- & 120.2$\pm$3.4 & --- & 106.5$\pm$4.6 & 116.1$\pm$13.6\\
3$'$ & timeout & \cellcolor{gray!20}{\bf 19.1$\pm$10.4} & 1371.0$\pm$28.6 & 572.5$\pm$20.9 & timeout & --- & 120.0$\pm$5.7 & 119.4$\pm$7.3 & 117.6$\pm$6.4 & ---\\
\bhline
\end{tabular}
\caption{Comparison of results between methods using heuristics and those not employing heuristics. Scenarios 1$'$, 2$'$, and 3$'$ are modifications of the first three scenarios by arranging non-leaf specifications sequentially. The symbol $t$ represents the runtime without using any heuristics. The terms $t_{\text{heur}}$, $t_{\text{heur}}^1$, $t_{\text{heur}}^2$, and $t_{\text{heur}}^3$ correspond to the runtimes when applying all heuristics, the heuristic of temporal order, the heuristic of essential switch transitions, and the heuristic of automaton state, respectively. This notation is similarly used for the cost, denoted as $c$.}
\label{tab:heur}
\vspace{-10pt}
\end{table*}
\subsubsection{Scenario 3}  Take a photo in meeting rooms \(m_1\), \(m_4\), and \(m_6\). The camera should be turned off for privacy reasons when not in meeting rooms. Deliver a document from desk \(d_5\) to \(d_3\), ensuring it does not pass through any public areas, as the document is internal and confidential. Guide a person waiting at desk \(d_{11}\) to meeting room \(m_6\). {\color{\modifycolor}Let $\mathsf{guide}$, $\mathsf{photo}$, and $\mathsf{camera}$ denote the actions of the robot guiding a person, capturing a photo, and activating the camera, respectively.}
\vspace{-1em}
\begin{sizeddisplay}{\small}
{ 
\allowdisplaybreaks
\begin{align*}
 L_1: \quad &  \level{1}{1} =  \Diamond\level{1}{2} \wedge \Diamond \level{2}{2} \wedge  \Diamond \level{3}{2}\\
 L_2: \quad &  \level{1}{2} = \Diamond \level{1}{3}\wedge \Diamond \level{2}{3}\wedge \Diamond \level{3}{3}\\
  &  \level{2}{2} = \Diamond (d_5 \wedge \mathsf{carry} \; \mathcal{U} \; (d_3 \wedge \bigcirc \neg \mathsf{carry})) \wedge \mathsf{notpublic}\\
  &  \level{3}{2} = \Diamond (d_{11}\wedge \mathsf{guide} \; \mathcal{U}\; (m_6 \wedge \bigcirc \neg \mathsf{guide})) \\
 L_3: \quad &  \level{1}{3} = \Diamond(m_1 \wedge \mathsf{photo}) \wedge \square (\neg \mathsf{meeting} \Rightarrow \neg \mathsf{camera} ) \\
 &  \level{2}{3} = \Diamond(m_4 \wedge \mathsf{photo})\wedge \square (\neg \mathsf{meeting} \Rightarrow \neg \mathsf{camera} ) \\
 &  \level{3}{3} = \Diamond(m_6 \wedge \mathsf{photo}) \wedge \square (\neg \mathsf{meeting} \Rightarrow \neg \mathsf{camera} ) \\
 & \mathsf{meeting} := \; m_1 \vee  m_2 \vee  m_3 \vee  m_4 \vee  m_5 \vee  m_6 
\end{align*}}
\vspace{-10pt}
\end{sizeddisplay}
 \subsubsection{Combinations of scenarios 1, 2 and 3}
We examine combinations of any two of these tasks as well as the combination of all three. {\color{\modifycolor}Due to space constraints, we only detail the scenario involving the final occurrence of all three tasks.} 
\begin{sizeddisplay}{\small} 
{\allowdisplaybreaks
\begin{align}\label{eq:task123}
 L_1: \quad &  \level{1}{1} =  \Diamond \level{1}{2} \wedge \Diamond\level{2}{2}  \wedge \Diamond\level{3}{2} \nonumber \\
 L_2: \quad &  {\color{black}\level{1}{2} =  \Diamond \level{1}{3} \wedge \Diamond\level{2}{3}} \quad \text{\makebox[5.4cm][r]{(scenario 1)}}  \nonumber \\
            &  {\color{black}\level{2}{2} =  \Diamond\level{3}{3} \wedge \Diamond\level{4}{3} \wedge \Diamond\level{5}{3}} \quad \text{\makebox[4.5cm][r]{(scenario 2)}}  \nonumber\\
            &  {\color{black}\level{3}{2} =  \Diamond\level{6}{3} \wedge \Diamond\level{7}{3} \wedge \Diamond\level{8}{3}} \quad \text{\makebox[4.5cm][r]{(scenario 3)}}  \nonumber \\
 L_3: \quad &  {\color{black}\level{1}{3} = \Diamond (d_5 \wedge \mathsf{default}} \wedge \bigcirc ((\mathsf{carrybin}\; \mathcal{U}\; \mathsf{dispose}) \wedge \Diamond \mathsf{default}))  \nonumber\\ 
 & \quad \quad \quad \quad {\color{black}\wedge \square (\mathsf{carrybin \Rightarrow \neg \mathsf{public}})} \nonumber \\
              & \level{2}{3}  = \Diamond (g \wedge \bigcirc (g \wedge \mathsf{emptybin}) \wedge \Diamond (d_5 \wedge \bigcirc (d_5 \wedge \mathsf{default}))) \nonumber \\ 
 & {\color{black}\level{3}{3} = \Diamond (p \wedge \mathsf{carry}\; \mathcal{U}\; (d_{10} \wedge \bigcirc \neg \mathsf{carry}))} \wedge \mathsf{notpublic}  \\
 &  {\color{black}\level{4}{3} = \Diamond (p \wedge \mathsf{carry}\; \mathcal{U}\; (d_{7} \wedge \bigcirc \neg \mathsf{carry}))} \wedge \mathsf{notpublic}  \nonumber\\
 &  {\color{black}\level{5}{3} = \Diamond (p \wedge \mathsf{carry}\; \mathcal{U}\; (d_{5} \wedge \bigcirc \neg \mathsf{carry}))} \wedge \mathsf{notpublic}  \nonumber\\
                & {\color{black}\level{6}{3} = \Diamond  \level{1}{4} \wedge \Diamond  \level{2}{4} \wedge \Diamond  \level{3}{4}} \nonumber\\
                  & {\color{black} \level{7}{3} = \Diamond (d_5 \wedge \mathsf{carry} \; \mathcal{U} \; (d_3 \wedge \bigcirc \neg \mathsf{carry}))} \wedge \mathsf{notpublic}  \nonumber\\
  &  {\color{black}\level{8}{3} = \Diamond (d_{11}\wedge \mathsf{guide} \; \mathcal{U}\; (m_6 \wedge \bigcirc \neg \mathsf{guide}))} \nonumber\\
 L_4: \quad &  {\color{black}\level{1}{4} = \Diamond(m_1 \wedge \mathsf{photo}) \wedge \square (\neg \mathsf{meeting} \Rightarrow \neg \mathsf{camera} )} \nonumber\\
 &  {\color{black}\level{2}{4} = \Diamond(m_4 \wedge \mathsf{photo})\wedge \square (\neg \mathsf{meeting} \Rightarrow \neg \mathsf{camera} )} \nonumber\\
 &  {\color{black}\level{3}{4} = \Diamond(m_6 \wedge \mathsf{photo}) \wedge \square (\neg \mathsf{meeting} \Rightarrow \neg \mathsf{camera} )} \nonumber
 \end{align}
}%
 \end{sizeddisplay}
{\color{\modifycolor} If the objective is to accomplish either scenario, this can be indicated by replacing the formula at level $L_1$ with \( \level{1}{1} = \Diamond (\level{1}{2} \vee \level{2}{2}  \vee \level{3}{2}) \). Furthermore, the expression \( \level{1}{1} = \Diamond (\level{1}{2} \vee \level{2}{2})  \wedge \Diamond \level{3}{2} \) specifies that either task 1 or task 2 must be completed, in addition to task 3.}
\begin{table*}[!t]
 \renewcommand{\arraystretch}{1.2}
\centering \footnotesize
\begin{tabular}{lcccccccc}
\bhline
scenario &  $l_{\text{std}}$ &  $l_{\text{hier}}$ &  $\ccalA_{\text{std}}$ &  $\ccalA_{\text{hier}}$ &  $t_{\text{std}}$ &  $t_{\text{hier}}$ & $c_{\text{std}}$ &  $c_{\text{hier}}$\\
\bhline
1 & \cellcolor{gray!20}{\bf 18} & 19 & (17, 39) & \cellcolor{gray!20}{\bf (12, 13)} & 14.1$\pm$3.7 & \cellcolor{gray!20}{\bf 4.9$\pm$2.3} & 71.0$\pm$6.3 & \cellcolor{gray!20}{\bf 69.0$\pm$5.7}\\
2 & \cellcolor{gray!20}{\bf 24} & 35 & (56, 326) & \cellcolor{gray!20}{\bf (20, 31)} & 39.6$\pm$2.5 & \cellcolor{gray!20}{\bf 7.0$\pm$3.1} & 90.4$\pm$4.0 &\cellcolor{gray!20}{\bf 88.6$\pm$7.0} \\ 
3 & \cellcolor{gray!20}{\bf 35} & 52 & (180, 1749) & \cellcolor{gray!20}{\bf (30, 49)}  &  timeout   & \cellcolor{gray!20}{\bf 14.5$\pm$4.1} &  --- &\cellcolor{gray!20}{\bf 97.1$\pm$5.9} \\ 
1 $\wedge$ 2 &\cellcolor{gray!20}{\bf 43} & 57 & (868, 12654) & \cellcolor{gray!20}{\bf (36, 49)} & timeout & \cellcolor{gray!20}{\bf 16.4$\pm$7.1} & --- & \cellcolor{gray!20}  {\bf 148.8$\pm$8.4} \\ 
1 $\wedge$ 3 & \cellcolor{gray!20}{\bf 54} & 74 & (2555, 69858) & \cellcolor{gray!20}{\bf (46, 67)}  & timeout & \cellcolor{gray!20}{\bf 47.9$\pm$22.3}& --- &  \cellcolor{gray!20}{\bf 166.4$\pm$8.0} \\ 
2 $\wedge$ 3 &\cellcolor{gray!20}{\bf 60} & 90 & (6056, 325745) & \cellcolor{gray!20}{\bf (54, 85)}   & timeout &\cellcolor{gray!20}{\bf 42.5$\pm$16.9}& --- &  \cellcolor{gray!20}{\bf 175.4$\pm$6.9} \\ 
1 $\wedge$ 2 $\wedge$ 3 & \cellcolor{gray!20}{\bf 79} & 111 & --- & \cellcolor{gray!20}  {\bf (70, 112)}   & timeout & \cellcolor{gray!20}{\bf 89.6$\pm$28.1}& --- & \cellcolor{gray!20}{\bf 246.6$\pm$8.3}  \\ 
(1 $\vee$ 2) $\wedge$ 3 & \cellcolor{gray!20}{\bf 79} & 110 & --- & \cellcolor{gray!20}  {\bf \color{\modifycolor}{(66, 98)}}   & timeout & \cellcolor{gray!20}{\bf  71.5$\pm$17.6}& --- & \cellcolor{gray!20}{\bf \color{\modifycolor} 213.1$\pm$23.4 } \\
1 $\vee$ 2 $\vee$ 3 & \cellcolor{gray!20}{\bf 79} & 109 & --- & \cellcolor{gray!20}  {\bf \color{\modifycolor}(64, 94)}   & timeout & \cellcolor{gray!20}{\bf {\color{\modifycolor}63.4$\pm$20.9}}& --- & \cellcolor{gray!20}{\bf {\color{\modifycolor}163.9$\pm$40.5}} \\
\bhline
\end{tabular}
\caption{The comparative analysis focuses on two different types of \ltl\ specifications. We denote the lengths of the standard and \hltl\ specifications as \(l_{\text{std}}\) and \(l_{\text{hier}}\), respectively. The sizes of the corresponding NFAs are represented by \(\ccalA_{\text{std}}\) and \(\ccalA_{\text{hier}}\), which detail the number of nodes and edges, with the node count listed first. In terms of solutions, the runtimes for the standard and \hltl\ specifications are indicated by \(t_{\text{std}}\) and \(t_{\text{hier}}\), respectively. Additionally, the plan horizons, or the lengths of the solutions, for the standard and hierarchical specifications are denoted by \(c_{\text{std}}\) and \(c_{\text{hier}}\), respectively.
}
\label{tab:result}
\end{table*}

\subsection{Ablation Study on Effects of Heuristics}
We set $m=2$. The results of the effects of different heuristics in scenarios 1 to 3 are detailed in Tab.~\ref{tab:heur}, which presents the average runtimes and cost over 20 trials per scenario. During each trial, the locations of the robots are randomly sampled within areas without labels. Our method was evaluated under three conditions: without any heuristics, with a single heuristic (either the temporal order as in Sec.~\ref{sec:heuristic_order}, critical switch transitions as in Sec.~\ref{sec:heuristic_switch}, or automaton state as in Sec.~\ref{sec:heuristic_automaton}), and with the combination of all heuristics. It was observed that in scenarios 1 to 3 where non-leaf specifications have no precedence temporal order, the impact of solely using the temporal order heuristic is equivalent to not using any heuristic at all. To further investigate this, three additional scenarios were introduced where non-leaf specifications were arranged sequentially. For example, in scenario 1, $\level{1}{1} = \Diamond \level{1}{2} \wedge \Diamond \level{2}{2}$ was modified to $\level{1}{1} = \Diamond (\level{1}{2} \wedge \Diamond \level{2}{2})$. Each navigational and manipulative action carried out incurs a cost of 1, but transitions between switches do not incur any costs. A maximum time limit of one hour was set for these tests, same as other simulations.

As expected, the method without heuristics yields solutions with minimum overall costs. However, employing all three heuristics simultaneously results in a substantial acceleration of the search process, roughly by two orders of magnitude. This significant improvement in runtimes outweighs the minor increase in costs associated with the use of heuristics. In six different scenarios, the heuristic based on essential switch transitions outperforms the other two heuristics in four scenarios. Its effectiveness is attributed to its ability to significantly reduce the number of switch transitions among team models, thus keeping the search confined within individual team models at the most time. Conversely, the temporal order heuristic is the least effective in five of the six scenarios. This inefficiency stems from the absence of precedent constraints in the non-leaf specifications of scenarios 1, 2, and 3. The cost differences between these heuristics are insignificant. In conclusion, each heuristic demonstrates improvement over the baseline method, and all of them should be applied in order to maximize the performance gain.

\subsection{Comparison with Existing Works}
We use $m=6$ robots and compare with the approach in~\cite{schillinger2018simultaneous}, both using all heuristics. Note that~\cite{schillinger2018simultaneous}'s method is tailored to standard \ltl\ specifications. Due to the absence of open-source code, we implemented their method. 
Robot locations in each scenario are randomly assigned within the free space. The performance, in terms of average runtimes and costs over 20 runs, is detailed in Tab.~\ref{tab:result} and includes the length of formulas and sizes of {\color{\modifycolor}automata}. The length of a formula is the total number of logical and temporal operators. 
Upon reviewing the results,~\cite{schillinger2018simultaneous}'s method failed to generate solutions for the last seven tasks within the one-hour limit. For tasks 1 and 2, our method produced solutions more quickly and with comparable costs. The failure of ~\cite{schillinger2018simultaneous}'s approach is attributed to the excessively large {\color{\modifycolor}automata} it generates, sometimes with hundreds of thousands of edges, e.g., 325745 edges for scenario $2 \wedge 3$, making the computation of the decomposition set time-consuming as it requires iterating over all possible runs. For the most complex scenario, such as $1 \wedge 2 \wedge 3$, generating an automaton within one hour is impossible. In contrast, our method was able to find a solution in around 90 seconds. Moreover, considering the last three tasks, which require completing all, two, or only one task, both the runtime and cost decrease as the number of required tasks is reduced.

\begin{table}[!t]
 \renewcommand{\arraystretch}{1.2}
\centering \footnotesize
\begin{tabular}{lclc}
\bhline
flattened specifications  &  $\ccalA_{\text{hier}}$ &  $t_{\text{hier}}$  &  $c_{\text{hier}}$\\
\bhline
None & (70, 112) &  89.6$\pm$28.1 & 246.6$\pm$8.3 \\
$\level{6}{3}$ & (64, 109) &  98.4$\pm$50.0 & 245.4$\pm$10.3 \\
$\level{6}{3}, \level{1}{2}$  & (69, 135) &  173.5$\pm$118.1 &  247.9$\pm$8.7\\
$\level{6}{3}, \level{1}{2}, \level{2}{2}$  & (105, 430) & 309.4$\pm$487.4 (7) &  237.6$\pm$12.7 \\
$\level{6}{3}, \level{1}{2}, \level{2}{2}, \level{3}{2}$  & (261, 2133) & timeout & --- \\
\bhline
\end{tabular}
\caption{Results illustrating the impact of hierarchy granularity on performance. The first column lists the set of non-leaf specifications that have been flattened. The number within the parentheses in the third column indicates the number of runs that resulted in a timeout.}
\label{tab:granularity}
\end{table}
\begin{figure*}[!t]
    \centering
     \subfigure[Runtimes w.r.t number of robots]{
      \label{fig:runtimes}
      \includegraphics[width=0.48\linewidth]{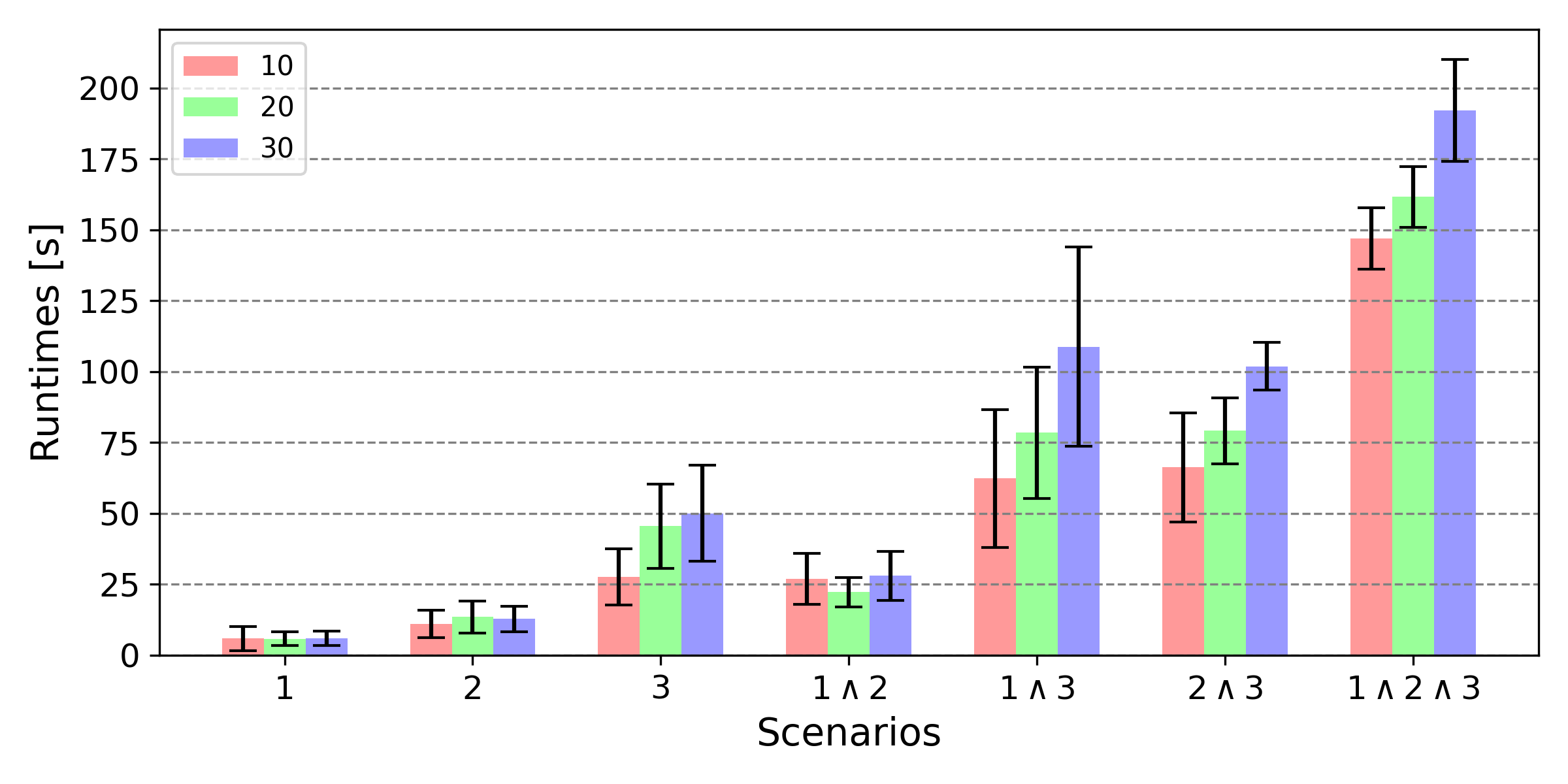}}
     \subfigure[Cost w.r.t number of robots]{
      \label{fig:cost}
      \includegraphics[width=0.48\linewidth]{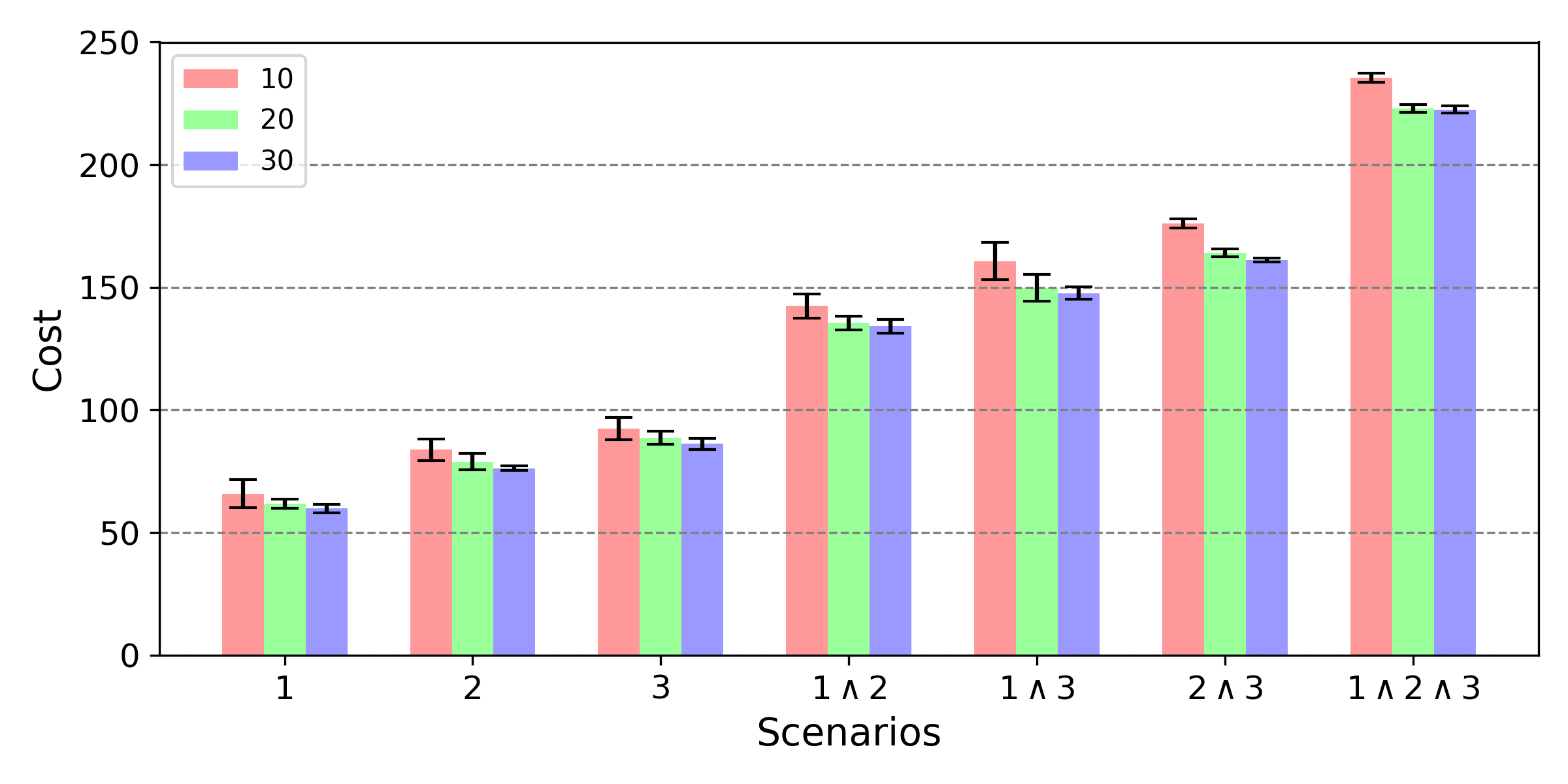}}
    \caption{Scalability results of runtimes and cost w.r.t number of robots.}
    \label{fig:scalability}
    \vspace{-10pt}
 \end{figure*}

\subsection{Effects of Granularity of Hierarchy}
For the same task, the performance of the planning algorithm is influenced by the granularity of hierarchical specifications. To investigate this, we evaluate the most complex task, which combines scenarios 1, 2, and 3 as defined in~\eqref{eq:task123}. The specifications in~\eqref{eq:task123} represent the finest granularity. We then progressively increase the coarseness by flattening more non-leaf specifications through the merging of their respective leaf specifications, thus gradually transforming the hierarchical structure towards a more flattened form. Eventually, this process results in a single standard specification. For example, flattening specification $\level{6}{3}$ results in $\level{6}{3} = \Diamond(m_1 \wedge \mathsf{photo}) \wedge \Diamond(m_4 \wedge \mathsf{photo})  \wedge \Diamond(m_6 \wedge \mathsf{photo}) \wedge \square (\neg \mathsf{meeting} \Rightarrow \neg \mathsf{camera})$. The outcomes from 20 runs are presented in Tab.~\ref{tab:granularity}. {Generally, as more non-leaf specifications are flattened, the {\color{\modifycolor}automata} increase in size, resulting in longer runtimes, while the costs remain relatively stable. Consequently, finer granularity in the hierarchy is associated with enhanced planning performance.}

\subsection{Scalability}\label{sec:scalability}
For assessing scalability, we varied the number of robots from 10 to 30. The statistical findings are presented in Fig.~\ref{fig:scalability}. Across different scenarios, a consistent pattern emerged: as the number of robots increased, the runtimes tended to increase, while the overall costs decreased. This trend is intuitive since a greater number of robots introduces a wider range of potential solutions. In the most complex scenario, our method was capable of identifying solutions in approximately 200 seconds, even with 30 robots. For tasks 1, 2, and 3, increasing the number of robots to 20 or 30 does not result in a substantial cost reduction compared to using 10 robots, indicating that 10 robots are sufficient for tasks of this complexity. However, when tasks are combined, the cost reduction becomes more significant with 20 or 30 robots. Additionally, 20 robots are capable of handling tasks of this complexity effectively, with the marginal benefit of increasing to 30 robots being minor.
\begin{figure*}[!t]
    \centering
      \includegraphics[width=0.95\linewidth]{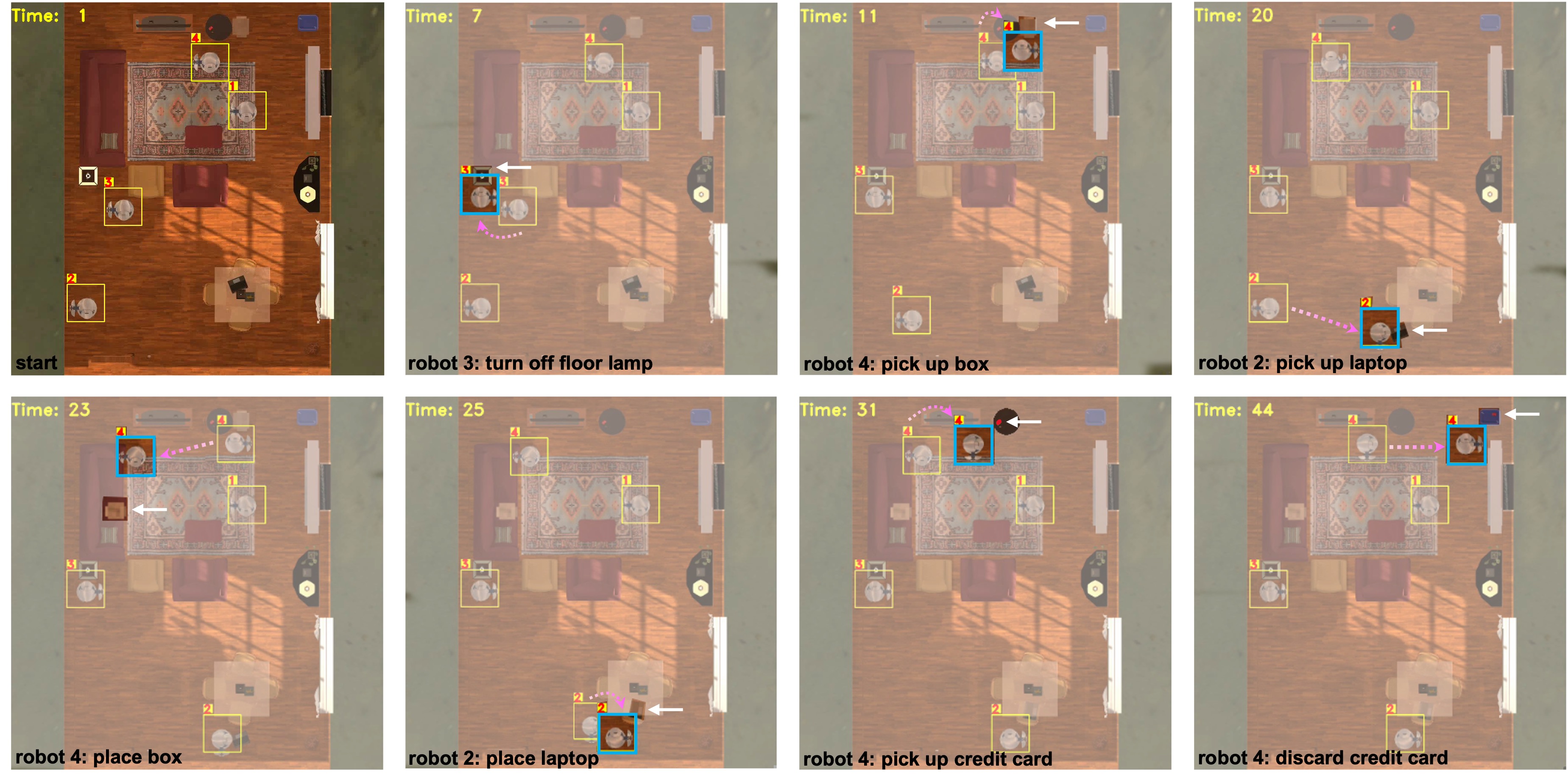}
    \caption{The keyframes illustrate the execution of the derivative task described in Eq.~\eqref{eq:ai2thor} within the AI2-THOR framework, indicating the crucial phases in the completion of task. Each of the four robots is inside yellow boxes with distinct numbers. {\color{\modifycolor}The upper-left corner of the keyframes indicates the time instants, while the bottom-left corner displays the working robot and its corresponding action. The working robots are enclosed in blue boxes, with dashed lines in a magenta gradient color connecting the position where they performed most recent manipulation action to their current location. White arrows highlight the objects with which the robots are interacting.} Robots 2 and 4 are simultaneously manipulating the box and the laptop.}
    \label{fig:ai2thor}
 \end{figure*}

\subsection{Evaluation on simulator AI2-THOR}
We conduct our experiments using the AI2-THOR simulator and the ALFRED dataset. AI2-THOR~\cite{kolve2017ai2} provides a realistic 3D simulated environment for robots to perform interactions within various household settings such as kitchens and living rooms. This setting allows for the assessment of algorithms in tasks that include navigation and object manipulation. The ALFRED dataset~\cite{shridhar2020alfred}, used within AI2-THOR, focuses on natural language instructions for simple, sequentially ordered tasks. Examples of such tasks in a living room setting include (i) ``Turn off the floor lamp,'' (ii) ``Move the box to a sofa,'' (iii) ``Move the laptop from the table to an armchair,'' and (iv) ``Dispose of the credit card in the garbage can.'' We refer to these tasks as \textit{base} tasks. Then we construct \textit{derivative} tasks by combining several base tasks within the same floor plan. One derivative task from base tasks (i)-(iv) is that ``First turn off the floor lamp. After that, move the box to a sofa and the laptop to an armchair. Finally, dispose of the credit card in the garbage can.'' The \hltl\ formulas are:
\begin{sizeddisplay}{\small}
{\allowdisplaybreaks
\begin{align}\label{eq:ai2thor}
       L_1: \quad &  \level{1}{1}= \Diamond  (\level{1}{2} \wedge \Diamond (\level{2}{2} \wedge \Diamond \level{3}{2} ))  \nonumber \\
       L_2: \quad         & \level{1}{2} = \Diamond (\mathsf{floorlamp} \wedge \bigcirc \mathsf{turnoff}) \nonumber \\  
        & \level{2}{2} =  \Diamond \level{1}{3} \wedge \Diamond  \level{2}{3}   \\
         & \level{3}{2} = \Diamond (\mathsf{creditcard} \wedge \bigcirc (\mathsf{pickup}  \wedge \Diamond (\mathsf{garbagecan} \wedge \bigcirc \mathsf{place}) ))  \nonumber \\
        L_3: \quad         &   \level{1}{3}  = \Diamond (\mathsf{box} \wedge \bigcirc (\mathsf{pickup}  \wedge \Diamond ( (\mathsf{sofa1} 
      \vee \mathsf{sofa2}) \wedge \bigcirc \mathsf{place}) )) \nonumber\\  
                & \level{2}{3} = \Diamond (\mathsf{laptop} \wedge \bigcirc (\mathsf{pickup} \wedge \Diamond ((\mathsf{armchair1} \nonumber \\
                & \quad \quad \vee \mathsf{armchair2} \vee \mathsf{armchair3} \vee \mathsf{armchair4})  \wedge \bigcirc \mathsf{place} ) )).   \nonumber
\end{align}}%
\end{sizeddisplay}
We represent navigation destinations using nouns and the manipulative actions to be performed at these destinations with verbs. In $\level{2}{3}$, given that there are four armchair available and the instructions do not specify a particular one, placing the laptop on any of these armchairs meets the task requirements.

{\color{\modifycolor} }
We vary the number of base tasks included in each derivative task, selecting from the set \{1, 2, 3, 4\}. A total of 35 derivative tasks are constructed, each comprising an equal number of randomly chosen base tasks from the dataset. For each derivative, we assign 1, 2, or 4 robots, each placed at random starting locations within the floor plan. This results in $4 \times 35 \times 3 = 420$ testing scenarios. The floor map is transformed into a grid world where each grid unit measures 0.25 meters. The grid world's dimensions vary from 10 to 40 units per side due to various floor plans. Recall that in objective~\eqref{eq:objective}, $c_r$ denotes the cost for robot $r$. In this part, we modified the goal~\eqref{eq:objective} as
\begin{align}\label{eq:minmax_objective}
    J(\tau) = \omega \max_{r\in [N]_+}  c_r + (1 - \omega)  \sum_{r=1}^N  c_r
\end{align}
where $\omega \in [0, 1]$. Here, setting $\omega = 0$ simplifies~\eqref{eq:minmax_objective} to objective~\eqref{eq:objective}, while $\omega = 1$ seeks to reduce the maximum cost among robots, with the secondary term regularizing the costs of the remaining robots. In the evaluation, we set $\omega = 0.9$ to prioritize minimizing the completion time (makespan). The results, detailed in Table~\ref{tab:ai2thor}, indicate an increase in runtime for solutions as the number of base tasks or robots increases. Additionally, with more robots engaged, the completion time is reduced due to the ability to distribute tasks more evenly among multiple robots. The execution of derivative task~\eqref{eq:ai2thor}  is illustrated in Fig.~\ref{fig:ai2thor}.

\renewcommand{\arraystretch}{1.2} 
\begin{table}[!t]
    \centering \footnotesize
    \setlength{\tabcolsep}{3pt}  
    \begin{tabular}{ccccccccccc}
    \bhline
    {\# robots}   & \# tasks    & $t$ & $c$ & \# tasks    & $t$ & $c$ \\
    \bhline
     
   1 &  \multirow{3}{*}{1} & \cellcolor{gray!20}{\bf 2.6$\pm$0.9} & 11.7$\pm$8.7 &  \multirow{3}{*}{3}  & \cellcolor{gray!20}{\bf 10.8$\pm$7.7} & 38.2$\pm$25.4 \\
   2 &    & 3.0$\pm$1.2 & 11.7$\pm$9.8 & &  12.1$\pm$9.2 & 27.9$\pm$21.6  \\
   4&    & 4.1$\pm$1.3 &\cellcolor{gray!20}{\bf 9.0$\pm$7.6 }  & & 14.6$\pm$9.1 & \cellcolor{gray!20}{\bf 22.0$\pm$14.2}  \\
     \hline
   1 &  \multirow{3}{*}{2}    & \cellcolor{gray!20}{\bf 4.3$\pm$2.9} & 23.4$\pm$21.0  &  \multirow{3}{*}{4}  & \cellcolor{gray!20}{\bf 13.0$\pm$13.9} & 50.1$\pm$29.6  \\
   2 &    & 4.8$\pm$3.2 & 17.1$\pm$15.4 & & 13.3$\pm$14.9 & 30.7$\pm$16.9  \\
   4 &   & 6.5$\pm$4.7 & \cellcolor{gray!20}{\bf 14.8$\pm$12.7}& & 14.9$\pm$11.4 & \cellcolor{gray!20}{\bf 23.9$\pm$14.0}   \\
      \bhline
    \end{tabular}
    \caption{Results across different numbers of base tasks and robots. Here, $t$ represents the runtimes required to return solutions, and $c$ denotes the completion time. }
    \label{tab:ai2thor}
    \vspace{-10pt}
\end{table}

\section{Conclusion and Discussion}

We addressed the problem of STAP for multiple robots subject to \hltl\ specifications. First, we presented the syntax and semantics for \hltl. Next, we developed a search-based planning algorithm, and introduced several heuristics leveraging the task structure to expedite it. We  provided theoretical analysis on the completeness and optimality.

Our simulations, focused on service tasks involving navigation and manipulation, included an ablation study on heuristics. This study demonstrated that each heuristic variably accelerates the search, with their combination yielding a significant speedup. Comparative studies with existing works illustrate that our approach can handle complex tasks beyond the reach of current methods. Scalability tests reveal that our method can manage up to 30 robots within 200 seconds. We attribute the computational efficiency to two factors. First, the adoption of hierarchical specifications results in smaller automatons, as indicated in Tab.~\ref{tab:result}. The results from the user study also highlighted the user-friendliness of \hltl. The second factor is  the planning algorithm, particularly its heuristics. As shown in Tab.~\ref{tab:heur}, the planning algorithm performs faster with any heuristic than without. Furthermore, employing a combination of all three heuristics results in shorter runtimes compared to using any single heuristic.

While the \hltl\ introduced has proven effective, there are several areas that require investigation. {\color{\modifycolor}An open challenge is developing an algorithmic approach to convert standard \ltl\ formulas into \hltl\ formulas and exploring how to introduce the hierarchical form to full LTL. The second aspect relates to the planning algorithm.} Currently, the planning algorithm assumes that robots operate independently, without considering scenarios where coordination, such as jointly carrying heavy items, is necessary. Therefore, it is crucial to include such situations, as in work~\cite{wei2025hierarchical}. Since \hltl\ can encompass multiple conflicting specifications, a future direction could be to define and maximize degree of satisfaction  when  not all specifications can be  met. 
Finally, integrating natural languages, as in work~\cite{luo2023obtaining}, to embed task hierarchies from instructions into structured formal languages, presents an intriguing prospect.



\begin{sizeddisplay}{\small}
\bibliographystyle{unsrtnat}
\bibliography{stap}

\begin{thebibliography}{72}
\providecommand{\natexlab}[1]{#1}
\providecommand{\url}[1]{\texttt{#1}}
\expandafter\ifx\csname urlstyle\endcsname\relax
  \providecommand{\doi}[1]{doi: #1}\else
  \providecommand{\doi}{doi: \begingroup \urlstyle{rm}\Url}\fi

\bibitem[Korsah et~al.(2013)Korsah, Stentz, and Dias]{korsah2013comprehensive}
G~Ayorkor Korsah, Anthony Stentz, and M~Bernardine Dias.
\newblock A comprehensive taxonomy for multi-robot task allocation.
\newblock \emph{The International Journal of Robotics Research}, 32\penalty0
  (12):\penalty0 1495--1512, 2013.

\bibitem[LaValle(2006)]{lavalle2006planning}
Steven~M LaValle.
\newblock \emph{Planning algorithms}.
\newblock Cambridge university press, 2006.

\bibitem[Woodcock et~al.(2009)Woodcock, Larsen, Bicarregui, and
  Fitzgerald]{woodcock2009formal}
Jim Woodcock, Peter~Gorm Larsen, Juan Bicarregui, and John Fitzgerald.
\newblock Formal methods: Practice and experience.
\newblock \emph{ACM computing surveys (CSUR)}, 41\penalty0 (4):\penalty0 1--36,
  2009.

\bibitem[Pnueli(1977)]{pnueli1977temporal}
Amir Pnueli.
\newblock The temporal logic of programs.
\newblock In \emph{18th Annual Symposium on Foundations of Computer Science
  (SFCS 1977)}, pages 46--57. ieee, 1977.

\bibitem[Kurtz and Lin(2023)]{kurtz2023temporal}
Vince Kurtz and Hai Lin.
\newblock Temporal logic motion planning with convex optimization via graphs of
  convex sets.
\newblock \emph{IEEE Transactions on Robotics}, 2023.

\bibitem[Tenenbaum et~al.(2011)Tenenbaum, Kemp, Griffiths, and
  Goodman]{tenenbaum2011grow}
Joshua~B Tenenbaum, Charles Kemp, Thomas~L Griffiths, and Noah~D Goodman.
\newblock How to grow a mind: Statistics, structure, and abstraction.
\newblock \emph{Science}, 331\penalty0 (6022):\penalty0 1279--1285, 2011.

\bibitem[Kemp et~al.(2007)Kemp, Perfors, and Tenenbaum]{kemp2007learning}
Charles Kemp, Andrew Perfors, and Joshua~B Tenenbaum.
\newblock Learning overhypotheses with hierarchical bayesian models.
\newblock \emph{Developmental Science}, 10\penalty0 (3):\penalty0 307--321,
  2007.

\bibitem[De~Giacomo and Vardi(2013)]{de2013linear}
Giuseppe De~Giacomo and Moshe~Y Vardi.
\newblock Linear temporal logic and linear dynamic logic on finite traces.
\newblock In \emph{Twenty-Third International Joint Conference on Artificial
  Intelligence}, 2013.

\bibitem[Schillinger et~al.(2018{\natexlab{a}})Schillinger, B{\"u}rger, and
  Dimarogonas]{schillinger2018simultaneous}
Philipp Schillinger, Mathias B{\"u}rger, and Dimos~V Dimarogonas.
\newblock Simultaneous task allocation and planning for temporal logic goals in
  heterogeneous multi-robot systems.
\newblock \emph{The International Journal of Robotics Research}, 37\penalty0
  (7):\penalty0 818--838, 2018{\natexlab{a}}.

\bibitem[Guo and Dimarogonas(2015)]{guo2015multi}
Meng Guo and Dimos~V Dimarogonas.
\newblock Multi-agent plan reconfiguration under local {LTL} specifications.
\newblock \emph{The International Journal of Robotics Research}, 34\penalty0
  (2):\penalty0 218--235, 2015.

\bibitem[Tumova and Dimarogonas(2016)]{tumova2016multi}
Jana Tumova and Dimos~V Dimarogonas.
\newblock Multi-agent planning under local {LTL} specifications and event-based
  synchronization.
\newblock \emph{Automatica}, 70:\penalty0 239--248, 2016.

\bibitem[Yu and Dimarogonas(2021)]{yu2021distributed}
Pian Yu and Dimos~V Dimarogonas.
\newblock Distributed motion coordination for multirobot systems under ltl
  specifications.
\newblock \emph{IEEE Transactions on Robotics}, 38\penalty0 (2):\penalty0
  1047--1062, 2021.

\bibitem[Loizou and Kyriakopoulos(December 2004)]{loizou2004automatic}
Savvas~G Loizou and Kostas~J Kyriakopoulos.
\newblock Automatic synthesis of multi-agent motion tasks based on {LTL}
  specifications.
\newblock In \emph{43rd IEEE Conference on Decision and Control (CDC)},
  volume~1, pages 153--158, The Bahamas, December 2004.

\bibitem[Smith et~al.(2011)Smith, T{\r u}mov{\'a}, Belta, and
  Rus]{smith2011optimal}
Stephen~L Smith, Jana T{\r u}mov{\'a}, Calin Belta, and Daniela Rus.
\newblock Optimal path planning for surveillance with temporal-logic
  constraints.
\newblock \emph{The International Journal of Robotics Research}, 30\penalty0
  (14):\penalty0 1695--1708, 2011.

\bibitem[Saha et~al.(2014)Saha, Ramaithitima, Kumar, Pappas, and
  Seshia]{saha2014automated}
Indranil Saha, Rattanachai Ramaithitima, Vijay Kumar, George~J Pappas, and
  Sanjit~A Seshia.
\newblock Automated composition of motion primitives for multi-robot systems
  from safe {LTL} specifications.
\newblock In \emph{2014 IEEE/RSJ International Conference on Intelligent Robots
  and Systems}, pages 1525--1532. IEEE, 2014.

\bibitem[Kantaros and Zavlanos(2017)]{kantaros2017sampling}
Yiannis Kantaros and Michael~M Zavlanos.
\newblock Sampling-based control synthesis for multi-robot systems under global
  temporal specifications.
\newblock In \emph{2017 ACM/IEEE 8th International Conference on Cyber-Physical
  Systems (ICCPS)}, pages 3--14. IEEE, 2017.

\bibitem[Kantaros and Zavlanos(2018{\natexlab{a}})]{kantaros2018distributedOpt}
Yiannis Kantaros and Michael~M Zavlanos.
\newblock Distributed optimal control synthesis for multi-robot systems under
  global temporal tasks.
\newblock In \emph{Proceedings of the 9th ACM/IEEE International Conference on
  Cyber-Physical Systems}, pages 162--173. IEEE Press, 2018{\natexlab{a}}.

\bibitem[Kantaros and Zavlanos(2018{\natexlab{b}})]{kantaros2018sampling}
Yiannis Kantaros and Michael~M Zavlanos.
\newblock Sampling-based optimal control synthesis for multirobot systems under
  global temporal tasks.
\newblock \emph{IEEE Transactions on Automatic Control}, 64\penalty0
  (5):\penalty0 1916--1931, 2018{\natexlab{b}}.

\bibitem[Kantaros and Zavlanos(2020)]{kantaros2020stylus}
Yiannis Kantaros and Michael~M Zavlanos.
\newblock Stylus*: A temporal logic optimal control synthesis algorithm for
  large-scale multi-robot systems.
\newblock \emph{The International Journal of Robotics Research}, 39\penalty0
  (7):\penalty0 812--836, 2020.

\bibitem[Kantaros et~al.(2022)Kantaros, Kalluraya, Jin, and
  Pappas]{kantaros2022perception}
Yiannis Kantaros, Samarth Kalluraya, Qi~Jin, and George~J Pappas.
\newblock Perception-based temporal logic planning in uncertain semantic maps.
\newblock \emph{IEEE Transactions on Robotics}, 38\penalty0 (4):\penalty0
  2536--2556, 2022.

\bibitem[Luo and Zavlanos(2019)]{luo2019transfer}
Xusheng Luo and Michael~M Zavlanos.
\newblock Transfer planning for temporal logic tasks.
\newblock In \emph{2019 IEEE 58th Conference on Decision and Control (CDC)},
  pages 5306--5311. IEEE, 2019.

\bibitem[Luo et~al.(2021)Luo, Kantaros, and Zavlanos]{luo2021abstraction}
Xusheng Luo, Yiannis Kantaros, and Michael~M Zavlanos.
\newblock An abstraction-free method for multirobot temporal logic optimal
  control synthesis.
\newblock \emph{IEEE Transactions on Robotics}, 37\penalty0 (5):\penalty0
  1487--1507, 2021.

\bibitem[Kloetzer et~al.(2011)Kloetzer, Ding, and Belta]{kloetzer2011multi}
Marius Kloetzer, Xu~Chu Ding, and Calin Belta.
\newblock Multi-robot deployment from {LTL} specifications with reduced
  communication.
\newblock In \emph{2011 50th IEEE Conference on Decision and Control and
  European Control Conference}, pages 4867--4872. IEEE, 2011.

\bibitem[Shoukry et~al.(2017)Shoukry, Nuzzo, Balkan, Saha,
  Sangiovanni-Vincentelli, Seshia, Pappas, and Tabuada]{shoukry2017linear}
Yasser Shoukry, Pierluigi Nuzzo, Ayca Balkan, Indranil Saha, Alberto~L
  Sangiovanni-Vincentelli, Sanjit~A Seshia, George~J Pappas, and Paulo Tabuada.
\newblock Linear temporal logic motion planning for teams of underactuated
  robots using satisfiability modulo convex programming.
\newblock In \emph{2017 IEEE 56th Annual Conference on Decision and Control
  (CDC)}, pages 1132--1137. IEEE, 2017.

\bibitem[Moarref and Kress-Gazit(2017)]{moarref2017decentralized}
Salar Moarref and Hadas Kress-Gazit.
\newblock Decentralized control of robotic swarms from high-level temporal
  logic specifications.
\newblock In \emph{2017 International Symposium on Multi-robot and Multi-agent
  Systems (MRS)}, pages 17--23. IEEE, 2017.

\bibitem[Lacerda and Lima(2019)]{lacerda2019petri}
Bruno Lacerda and Pedro~U Lima.
\newblock Petri net based multi-robot task coordination from temporal logic
  specifications.
\newblock \emph{Robotics and Autonomous Systems}, 122:\penalty0 103289, 2019.

\bibitem[Sahin et~al.(2019)Sahin, Nilsson, and Ozay]{sahin2019multirobot}
Yunus~Emre Sahin, Petter Nilsson, and Necmiye Ozay.
\newblock Multirobot coordination with counting temporal logics.
\newblock \emph{IEEE Transactions on Robotics}, 2019.

\bibitem[Leahy et~al.(2021)Leahy, Serlin, Vasile, Schoer, Jones, Tron, and
  Belta]{leahy2021scalable}
Kevin Leahy, Zachary Serlin, Cristian-Ioan Vasile, Andrew Schoer, Austin~M
  Jones, Roberto Tron, and Calin Belta.
\newblock Scalable and robust algorithms for task-based coordination from
  high-level specifications (scratches).
\newblock \emph{IEEE Transactions on Robotics}, 38\penalty0 (4):\penalty0
  2516--2535, 2021.

\bibitem[Luo and Zavlanos(2022)]{luo2022temporal}
Xusheng Luo and Michael~M Zavlanos.
\newblock Temporal logic task allocation in heterogeneous multirobot systems.
\newblock \emph{IEEE Transactions on Robotics}, 38\penalty0 (6):\penalty0
  3602--3621, 2022.

\bibitem[Liu et~al.(2022)Liu, Guo, and Li]{liu2022time}
Zesen Liu, Meng Guo, and Zhongkui Li.
\newblock Time minimization and online synchronization for multi-agent systems
  under collaborative temporal tasks.
\newblock \emph{arXiv preprint arXiv:2208.07756}, 2022.

\bibitem[Li et~al.(2023)Li, Chen, Wang, and Kan]{li2023fast}
Lin Li, Ziyang Chen, Hao Wang, and Zhen Kan.
\newblock Fast task allocation of heterogeneous robots with temporal logic and
  inter-task constraints.
\newblock \emph{IEEE Robotics and Automation Letters}, 2023.

\bibitem[Djeumou et~al.(2020)Djeumou, Xu, and Topcu]{djeumou2020probabilistic}
Franck Djeumou, Zhe Xu, and Ufuk Topcu.
\newblock Probabilistic swarm guidance subject to graph temporal logic
  specifications.
\newblock In \emph{Robotics: Science and Systems}, 2020.

\bibitem[Yan et~al.(2019)Yan, Xu, and Julius]{yan2019swarm}
Ruixuan Yan, Zhe Xu, and Agung Julius.
\newblock Swarm signal temporal logic inference for swarm behavior analysis.
\newblock \emph{IEEE Robotics and Automation Letters}, 4\penalty0 (3):\penalty0
  3021--3028, 2019.

\bibitem[Camacho et~al.(2017)Camacho, Triantafillou, Muise, Baier, and
  McIlraith]{camacho2017non}
Alberto Camacho, Eleni Triantafillou, Christian~J Muise, Jorge~A Baier, and
  Sheila~A McIlraith.
\newblock Non-deterministic planning with temporally extended goals: Ltl over
  finite and infinite traces.
\newblock In \emph{AAAI}, pages 3716--3724, 2017.

\bibitem[Camacho et~al.(2019)Camacho, Icarte, Klassen, Valenzano, and
  McIlraith]{camacho2019ltl}
Alberto Camacho, R~Toro Icarte, Toryn~Q Klassen, Richard Valenzano, and
  Sheila~A McIlraith.
\newblock {LTL} and beyond: Formal languages for reward function specification
  in reinforcement learning.
\newblock In \emph{Proceedings of the 28th International Joint Conference on
  Artificial Intelligence (IJCAI)}, pages 6065--6073, 2019.

\bibitem[Schillinger et~al.(2019)Schillinger, B{\"u}rger, and
  Dimarogonas]{schillinger2019hierarchical}
Philipp Schillinger, Mathias B{\"u}rger, and Dimos~V Dimarogonas.
\newblock Hierarchical {LTL}-task mdps for multi-agent coordination through
  auctioning and learning.
\newblock \emph{The International Journal of Robotics Research}, 2019.

\bibitem[Liu et~al.(2024)Liu, Guo, and Li]{liu2024time}
Zesen Liu, Meng Guo, and Zhongkui Li.
\newblock Time minimization and online synchronization for multi-agent systems
  under collaborative temporal logic tasks.
\newblock \emph{Automatica}, 159:\penalty0 111377, 2024.

\bibitem[Biere et~al.(2006)Biere, Heljanko, Junttila, Latvala, and
  Schuppan]{biere2006linear}
Armin Biere, Keijo Heljanko, Tommi Junttila, Timo Latvala, and Viktor Schuppan.
\newblock Linear encodings of bounded {LTL} model checking.
\newblock \emph{Logical Methods in Computer Science}, 2\penalty0
  (5:5):\penalty0 1--64, 2006.

\bibitem[Leahy et~al.(2022)Leahy, Jones, and Vasile]{leahy2022fast}
Kevin Leahy, Austin Jones, and Cristian-Ioan Vasile.
\newblock Fast decomposition of temporal logic specifications for heterogeneous
  teams.
\newblock \emph{IEEE Robotics and Automation Letters}, 7\penalty0 (2):\penalty0
  2297--2304, 2022.

\bibitem[Schillinger et~al.(2018{\natexlab{b}})Schillinger, B{\"u}rger, and
  Dimarogonas]{schillinger2018decomposition}
Philipp Schillinger, Mathias B{\"u}rger, and Dimos~V Dimarogonas.
\newblock Decomposition of finite {LTL} specifications for efficient
  multi-agent planning.
\newblock In \emph{Distributed Autonomous Robotic Systems}, pages 253--267.
  Springer, 2018{\natexlab{b}}.

\bibitem[Faruq et~al.(2018)Faruq, Parker, Laccrda, and
  Hawes]{faruq2018simultaneous}
Fatma Faruq, David Parker, Bruno Laccrda, and Nick Hawes.
\newblock Simultaneous task allocation and planning under uncertainty.
\newblock In \emph{2018 IEEE/RSJ International Conference on Intelligent Robots
  and Systems (IROS)}, pages 3559--3564. IEEE, 2018.

\bibitem[Robinson et~al.(2021)Robinson, Su, and Zhang]{robinson2021multiagent}
Thomas Robinson, Guoxin Su, and Minjie Zhang.
\newblock Multiagent task allocation and planning with multi-objective
  requirements.
\newblock In \emph{Proceedings of the 20th International Conference on
  Autonomous Agents and MultiAgent Systems}, pages 1628--1630, 2021.

\bibitem[Luo et~al.(2024)Luo, Xu, Liu, and Liu]{luo2024decomposition}
Xusheng Luo, Shaojun Xu, Ruixuan Liu, and Changliu Liu.
\newblock Decomposition-based hierarchical task allocation and planning for
  multi-robots under hierarchical temporal logic specifications.
\newblock \emph{IEEE Robotics and Automation Letters}, 2024.

\bibitem[Gerkey and Matari{\'c}(2004)]{gerkey2004formal}
Brian~P Gerkey and Maja~J Matari{\'c}.
\newblock A formal analysis and taxonomy of task allocation in multi-robot
  systems.
\newblock \emph{The International journal of robotics research}, 23\penalty0
  (9):\penalty0 939--954, 2004.

\bibitem[Quinton et~al.(2023)Quinton, Grand, and Lesire]{quinton2023market}
F{\'e}lix Quinton, Christophe Grand, and Charles Lesire.
\newblock Market approaches to the multi-robot task allocation problem: a
  survey.
\newblock \emph{Journal of Intelligent \& Robotic Systems}, 107\penalty0
  (2):\penalty0 29, 2023.

\bibitem[Chakraa et~al.(2023)Chakraa, Gu{\'e}rin, Leclercq, and
  Lefebvre]{chakraa2023optimization}
Hamza Chakraa, Fran{\c{c}}ois Gu{\'e}rin, Edouard Leclercq, and Dimitri
  Lefebvre.
\newblock Optimization techniques for multi-robot task allocation problems:
  Review on the state-of-the-art.
\newblock \emph{Robotics and Autonomous Systems}, page 104492, 2023.

\bibitem[Zavlanos et~al.(2008)Zavlanos, Spesivtsev, and
  Pappas]{zavlanos2008distributed}
Michael~M Zavlanos, Leonid Spesivtsev, and George~J Pappas.
\newblock A distributed auction algorithm for the assignment problem.
\newblock In \emph{2008 47th IEEE Conference on Decision and Control}, pages
  1212--1217. IEEE, 2008.

\bibitem[Choi et~al.(2009)Choi, Brunet, and How]{choi2009consensus}
Han-Lim Choi, Luc Brunet, and Jonathan~P How.
\newblock Consensus-based decentralized auctions for robust task allocation.
\newblock \emph{IEEE transactions on robotics}, 25\penalty0 (4):\penalty0
  912--926, 2009.

\bibitem[Koes et~al.(2005)Koes, Nourbakhsh, and Sycara]{koes2005heterogeneous}
Mary Koes, Illah Nourbakhsh, and Katia Sycara.
\newblock Heterogeneous multirobot coordination with spatial and temporal
  constraints.
\newblock In \emph{AAAI}, volume~5, pages 1292--1297, 2005.

\bibitem[Patel et~al.(2020)Patel, Rudnick-Cohen, Azarm, Otte, Xu, and
  Herrmann]{patel2020decentralized}
Ruchir Patel, Eliot Rudnick-Cohen, Shapour Azarm, Michael Otte, Huan Xu, and
  Jeffrey~W Herrmann.
\newblock Decentralized task allocation in multi-agent systems using a
  decentralized genetic algorithm.
\newblock In \emph{2020 IEEE International Conference on Robotics and
  Automation (ICRA)}, pages 3770--3776. IEEE, 2020.

\bibitem[Wei et~al.(2020)Wei, Ji, and Cai]{wei2020particle}
Changyun Wei, Ze~Ji, and Boliang Cai.
\newblock Particle swarm optimization for cooperative multi-robot task
  allocation: a multi-objective approach.
\newblock \emph{IEEE Robotics and Automation Letters}, 5\penalty0 (2):\penalty0
  2530--2537, 2020.

\bibitem[Okumura and D{\'e}fago(2023)]{okumura2023solving}
Keisuke Okumura and Xavier D{\'e}fago.
\newblock Solving simultaneous target assignment and path planning efficiently
  with time-independent execution.
\newblock \emph{Artificial Intelligence}, 321:\penalty0 103946, 2023.

\bibitem[Edison and Shima(2011)]{edison2011integrated}
Eugene Edison and Tal Shima.
\newblock Integrated task assignment and path optimization for cooperating
  uninhabited aerial vehicles using genetic algorithms.
\newblock \emph{Computers \& Operations Research}, 38\penalty0 (1):\penalty0
  340--356, 2011.

\bibitem[Chen et~al.(2021)Chen, Alonso-Mora, Bai, Harabor, and
  Stuckey]{chen2021integrated}
Zhe Chen, Javier Alonso-Mora, Xiaoshan Bai, Daniel~D Harabor, and Peter~J
  Stuckey.
\newblock Integrated task assignment and path planning for capacitated
  multi-agent pickup and delivery.
\newblock \emph{IEEE Robotics and Automation Letters}, 6\penalty0 (3):\penalty0
  5816--5823, 2021.

\bibitem[Aggarwal et~al.(2022)Aggarwal, Ho, and Nakadai]{aggarwal2022extended}
Aayush Aggarwal, Florence Ho, and Shinji Nakadai.
\newblock Extended time dependent vehicle routing problem for joint task
  allocation and path planning in shared space.
\newblock In \emph{2022 IEEE/RSJ International Conference on Intelligent Robots
  and Systems (IROS)}, pages 12037--12044. IEEE, 2022.

\bibitem[Ma and Koenig(2016)]{ma2016optimal}
Hang Ma and Sven Koenig.
\newblock Optimal target assignment and path finding for teams of agents.
\newblock In \emph{Proceedings of the 2016 International Conference on
  Autonomous Agents \& Multiagent Systems}, pages 1144--1152, 2016.

\bibitem[Wilkins(2014)]{wilkins2014practical}
David~E Wilkins.
\newblock \emph{Practical planning: extending the classical AI planning
  paradigm}.
\newblock Elsevier, 2014.

\bibitem[Georgievski and Aiello(2015)]{georgievski2015htn}
Ilche Georgievski and Marco Aiello.
\newblock Htn planning: Overview, comparison, and beyond.
\newblock \emph{Artificial Intelligence}, 222:\penalty0 124--156, 2015.

\bibitem[Weser et~al.(2010)Weser, Off, and Zhang]{weser2010htn}
Martin Weser, Dominik Off, and Jianwei Zhang.
\newblock Htn robot planning in partially observable dynamic environments.
\newblock In \emph{2010 IEEE International Conference on Robotics and
  Automation}, pages 1505--1510. IEEE, 2010.

\bibitem[De~Mello and Sanderson(1990)]{de1990and}
LS~Homem De~Mello and Arthur~C Sanderson.
\newblock And/or graph representation of assembly plans.
\newblock \emph{IEEE Transactions on robotics and automation}, 6\penalty0
  (2):\penalty0 188--199, 1990.

\bibitem[Cheng et~al.(2021)Cheng, Sun, and Tomizuka]{cheng2021human}
Yujiao Cheng, Liting Sun, and Masayoshi Tomizuka.
\newblock Human-aware robot task planning based on a hierarchical task model.
\newblock \emph{IEEE Robotics and Automation Letters}, 6\penalty0 (2):\penalty0
  1136--1143, 2021.

\bibitem[Baier and McIlraith(2009)]{baier2009htn}
SSJA Baier and Sheila~A McIlraith.
\newblock Htn planning with preferences.
\newblock In \emph{21st Int. Joint Conf. on Artificial Intelligence}, pages
  1790--1797, 2009.

\bibitem[Lin and Bercher(2022)]{lin2022expressive}
Songtuan Lin and Pascal Bercher.
\newblock On the expressive power of planning formalisms in conjunction with
  ltl.
\newblock In \emph{Proceedings of the International Conference on Automated
  Planning and Scheduling}, volume~32, pages 231--240, 2022.

\bibitem[Baier and Katoen(2008)]{baier2008principles}
Christel Baier and Joost-Pieter Katoen.
\newblock \emph{Principles of model checking}.
\newblock MIT press Cambridge, 2008.

\bibitem[Bacchus and Kabanza(2000)]{bacchus2000using}
Fahiem Bacchus and Froduald Kabanza.
\newblock Using temporal logics to express search control knowledge for
  planning.
\newblock \emph{Artificial intelligence}, 116\penalty0 (1-2):\penalty0
  123--191, 2000.

\bibitem[Bozzelli et~al.(2018)Bozzelli, Molinari, Montanari, Peron, and
  Sala]{bozzelli2018interval}
Laura Bozzelli, Alberto Molinari, Angelo Montanari, Adriano Peron, and Pietro
  Sala.
\newblock Interval vs. point temporal logic model checking: An expressiveness
  comparison.
\newblock \emph{ACM Transactions on Computational Logic (TOCL)}, 20\penalty0
  (1):\penalty0 1--31, 2018.

\bibitem[Van Den~Berg and Overmars(2005)]{van2005prioritized}
Jur~P Van Den~Berg and Mark~H Overmars.
\newblock Prioritized motion planning for multiple robots.
\newblock In \emph{2005 IEEE/RSJ International Conference on Intelligent Robots
  and Systems}, pages 430--435. IEEE, 2005.

\bibitem[Wagner and Choset(2011)]{wagner2011m}
Glenn Wagner and Howie Choset.
\newblock M*: A complete multirobot path planning algorithm with performance
  bounds.
\newblock In \emph{2011 IEEE/RSJ international conference on intelligent robots
  and systems}, pages 3260--3267. IEEE, 2011.

\bibitem[Kolve et~al.(2017)Kolve, Mottaghi, Han, VanderBilt, Weihs, Herrasti,
  Deitke, Ehsani, Gordon, Zhu, et~al.]{kolve2017ai2}
Eric Kolve, Roozbeh Mottaghi, Winson Han, Eli VanderBilt, Luca Weihs, Alvaro
  Herrasti, Matt Deitke, Kiana Ehsani, Daniel Gordon, Yuke Zhu, et~al.
\newblock Ai2-thor: An interactive 3d environment for visual ai.
\newblock \emph{arXiv preprint arXiv:1712.05474}, 2017.

\bibitem[Shridhar et~al.(2020)Shridhar, Thomason, Gordon, Bisk, Han, Mottaghi,
  Zettlemoyer, and Fox]{shridhar2020alfred}
Mohit Shridhar, Jesse Thomason, Daniel Gordon, Yonatan Bisk, Winson Han,
  Roozbeh Mottaghi, Luke Zettlemoyer, and Dieter Fox.
\newblock Alfred: A benchmark for interpreting grounded instructions for
  everyday tasks.
\newblock In \emph{Proceedings of the IEEE/CVF conference on computer vision
  and pattern recognition}, pages 10740--10749, 2020.

\bibitem[Wei et~al.(2025)Wei, Luo, and Liu]{wei2025hierarchical}
Zhongqi Wei, Xusheng Luo, and Changliu Liu.
\newblock Hierarchical temporal logic task and motion planning for multi-robot
  systems.
\newblock In \emph{Robotics: Science and Systems}, 2025.

\bibitem[Luo et~al.(2023)Luo, Xu, and Liu]{luo2023obtaining}
Xusheng Luo, Shaojun Xu, and Changliu Liu.
\newblock Obtaining hierarchy from human instructions: an llms-based approach.
\newblock In \emph{CoRL 2023 Workshop on Learning Effective Abstractions for
  Planning (LEAP)}, 2023.

\end{thebibliography}
\end{sizeddisplay}

 \end{document}